\documentclass[10pt,journal]{IEEEtran}
\usepackage{multicol}
\usepackage{balance}

\usepackage{color}
\usepackage[colorlinks, linkcolor=color1, anchorcolor=blue, citecolor=color1]{hyperref}

\usepackage{comment}
\usepackage{verbatim} 

\usepackage{amssymb}
\usepackage{amsmath}
\usepackage{amsthm}
\usepackage[lined,boxed,commentsnumbered, ruled]{algorithm2e}
\usepackage{mathrsfs}
\usepackage{algorithmic}
\usepackage{bm}
\usepackage{hyperref}
\usepackage{pgfplots}
\usepackage{tikz}
\usetikzlibrary{arrows}
\usepackage{graphicx,booktabs,multirow}
\usepackage{subcaption}

\definecolor{colorhkust}{RGB}{20,43,140}
\definecolor{colortsinghua}{RGB}{116,52,129}
\definecolor{color1}{RGB}{128,0,0}

\newtheorem{lemma}{Lemma}
\newtheorem{theorem}{Theorem}
\newtheorem{corollary}{Corollary}
\newtheorem{proposition}{Proposition}

\newtheorem{remark}{Remark}
\newtheorem{assumption}{Assumption}

\usepackage{mathtools}

\newcommand{\eqdef}{\coloneqq}

\def\<#1,#2>{\left \langle #1,#2 \right \rangle}
\newcommand{\equa}[1]{\begin{equation}\begin{aligned}
			#1
		\end{aligned}
	\end{equation}
}
		
\usepackage{cite}

\allowdisplaybreaks

\begin{document}

\title{Federated Learning over Hierarchical Wireless \\ Networks: Training Latency Minimization \\ via Submodel Partitioning}
\author{
Wenzhi Fang, Dong-Jun Han, and Christopher G. Brinton
\thanks{
Wenzhi Fang and Christopher G. Brinton are with the Elmore Family School of Electrical and Computer Engineering, Purdue University, West Lafayette, IN, 47906 USA email:\{fang375, cgb\}@purdue.edu

D.-J. Han is with the Department of Computer Science and Engineering, Yonsei University, Republic of Korea. email: djh@yonsei.ac.kr.

This work was supported in part by the National Science Foundation (NSF) under grant CNS-2146171, the Office of Naval Research (ONR) under grant N00014-21-1-2472, and the Defense Advanced Research Projects Agency (DARPA) under grant D22AP00168.

An abridged version of this paper appeared in the 2024 IEEE International Conference on Communications (ICC) \cite{fang2023submodel}.
}
}




\maketitle

\begin{abstract}
Hierarchical federated learning (HFL) has demonstrated promising scalability advantages over the traditional ``star-topology'' architecture-based federated learning (FL).
However, HFL still imposes significant computation, communication, and storage burdens on the edge, especially when training a large-scale model over resource-constrained wireless devices. In this paper, we propose \textit{hierarchical independent submodel training} (\textnormal{\texttt{HIST}}), a new FL methodology that aims to address these issues in hierarchical cloud-edge-client networks. 
The key idea behind \textnormal{\texttt{HIST}} is to divide the global model into disjoint partitions (or submodels) per round  so that each group of clients (i.e., cells) is responsible for training only one partition of the model. 
We characterize the convergence behavior of \textnormal{\texttt{HIST}} under mild assumptions, showing the impacts of several key attributes (e.g., submodel sizes, number of cells, edge and global aggregation frequencies) on the rate and stationarity gap.
Building upon the theoretical results, we propose a submodel partitioning strategy to minimize the training latency depending on network resource availability and a target learning performance guarantee.
We then demonstrate how \textnormal{\texttt{HIST}} can be augmented with over-the-air computation (AirComp) to further enhance the efficiency of the model aggregation over the edge cells. 
Through numerical evaluations, we verify that \textnormal{\texttt{HIST}} is able to save training time and communication costs by wide margins while achieving comparable accuracy as conventional HFL. Moreover, our experiments demonstrate that AirComp-assisted \textnormal{\texttt{HIST}} provides further improvements in training latency. 

\end{abstract}



\section{Introduction}

Massive amounts of training data collected by geographically distributed users have contributed to the huge success of modern machine learning (ML) applications. However, due to privacy and communication resource constraints, users may be unwilling to share their data with the service provider for centralized model training.
 To avoid this issue, federated learning (FL) \cite{mcmahan17a}, as a promising distributed learning approach, has been widely investigated recently.  
Different from centralized training, in the conventional version of FL, clients are responsible for training the model while the central server is only in charge of aggregating the client models and synchronizing them with the resulting global model.
This enables clients to collaboratively learn a global model without any sharing of raw data.
Owing to this benefit,  FL has  
 attracted significant attention in recent years \cite{dinh2020federated,wang2021network,yuan2023decentralized,jianyu}. 

In the traditional cloud-based FL \cite{wang2019adaptive}, all clients in the system directly communicate with a central cloud server to exchange model information, resulting in communication scalability issues as the size of the network grows.
To address this, hierarchical federated learning (HFL) has been proposed as an alternative \cite{liu2020client,lin2021semi,hosseinalipour2022multi,wang2022demystifying,fang2024hierarchical}, taking advantage of the fact that in many network systems, clusters of clients are served by intermediate edge servers (e.g., mobile devices partitioned into cells, with the base station containing an edge server).
The introduction of edge servers in HFL reduces communication and scheduling complexity, as the cloud server now only needs to communicate with the edge servers. 
However, as the size of the model to be trained increases, the HFL training process still suffers from scalability issues. These issues manifest in several dimensions: (i) computation/storage costs at individual clients, (ii) communication burden between clients and edge servers, and (iii) communication load between edge servers and the cloud server. These are fundamental bottlenecks for the practical deployment of HFL, especially in the growing set of cases where resource-constrained devices (e.g., mobile phones) are aiming to collaboratively train a large-scale neural network (e.g., state-of-the-art image classifiers may have millions of parameters \cite{he2016deep}). 

Motivated by these challenges, in this paper, we investigate a training methodology for HFL, termed \textit{hierarchical independent submodel training} (\textnormal{\texttt{HIST}}), that is communication-, computation-, and storage efficient. One of our key innovations is to integrate independent submodel training (IST) into HFL. The core idea is to partition the global model into disjoint submodels in each training round and distribute them across different cells, so that devices in distinct cells are responsible for training different partitions of the full model. Such a submodel partitioning has the potential to reduce computation and storage loads at clients, and also alleviate communication burden on both the links between clients and edge servers and between edge servers and the cloud server. 
In particular, it makes the per-iteration communication complexity at the cloud server remain consistent regardless of the number
of edge servers.
Doing so, however, requires a careful analysis of how submodel partitioning impacts the learning performance and training efficiency in HFL, which we address in this paper.

In addition to the training methodology, the client-edge wireless communication mechanism also has significant implications on the resource efficiency of HFL. Limited radio resources within wireless cells create communication bottlenecks when serving large transmissions from several users.
For this reason, some researchers, e.g., \cite{lim2021decentralized,wen2022joint} have investigated optimizing radio resource allocation within cells, blended with techniques like partial client selection, to enhance the efficiency of HFL training.
However, these works focus on orthogonal multiple access (OMA) transmissions in each cell, in which the communication latency incurred by each edge server in each round of training will increase with the number of clients in its coverage \cite{wang2022interference}.

This motivates our second key idea, which is to consider over-the-air computation (AirComp) in HFL together with IST. 
With AirComp \cite{zhu2021over,fang2021over,wang2022over}, clients in the same cell can transmit their models simultaneously to the edge server, resulting in significant training time savings during the model aggregation in each cell. When designed properly, the communication complexity at each edge server will not scale with the number of clients. This complements submodel partitioning, which allows the communication complexity at the cloud server to become independent of the number of cells. 
Achieving this, however, requires closely studying the distortion error introduced by AirComp, how that impacts the \texttt{HIST} training process, and how it can be controlled through physical-layer design. We will develop such an understanding in this paper and employ it within our optimization methodology.


\vspace{-0.1in}
\subsection{Main Contributions}
  \begin{itemize}
\item We propose hierarchical independent submodel training (\texttt{HIST}) to enhance the scalability of FL over cloud-edge-client network topologies (Section \ref{sec_algo}). \texttt{HIST} aims to reduce computation, communication, and storage costs incurred during HFL training via submodel partitioning across the edge. {\color{black}We demonstrate {\tt HIST}'s applicability on fully connected and convolutional neural networks.}
\item We analytically characterize the convergence characteristics of \texttt{HIST} (Section \ref{sec_theory}) for non-convex loss functions, under milder assumptions than those adopted in the existing IST literature. Based on the convergence result, we analyze the performance-efficiency trade-off induced by submodel partitioning, and provide guidelines on setting the key system parameters (e.g., aggregation frequencies, step size, partitioning strategy) of the proposed \texttt{HIST}.
\item We analyze the impact of submodel size on the convergence bound and the training latency of \texttt{HIST}. Using these relationships, we 
propose a training latency minimization strategy (Section \ref{sec_partion_opt}) which optimizes the submodel partitioning sizes without significantly compromising the learning performance (i.e., quantified through the convergence bound), while considering the network resource availability (e.g., computation powers, data rate). 
\item To further enhance training scalability, we propose an AirComp-assisted \texttt{HIST} for the client-edge wireless network within each cell (Section \ref{sec_aircomp}). 
With AirComp in place, we show that each edge server is able to obtain an unbiased estimator of client local model averages within its coverage. We characterize the variance of this estimate on the convergence behavior of the proposed AirComp-assisted \texttt{HIST}. We leverage these relationships to augment our submodel partitioning optimization, as well as to minimize model distortion in receive beamforming.
\end{itemize}
{\color{black}We conduct experiments (Section \ref{sec_simulation}) using both fully connected neural networks and convolutional neural networks to validate the effectiveness of the proposed algorithm in hierarchical network settings.}
Results show that \texttt{HIST} achieves significant resource savings for the same target accuracy compared with standard hierarchical FL in a variety of network configurations. Moreover, numerical experiments confirm that the optimized submodel partitioning strategy further reduces the training latency without model performance degradation. We also show that, under a wide signal-to-noise ratio (SNR) region, the AirComp-assisted \texttt{HIST} algorithm attains almost the same testing accuracy while significantly reducing the training latency compared with OMA-based aggregation.  To the best of our knowledge, this work is one of the earliest attempts to successfully integrate HFL, IST, and AirComp and to analyze its performance over wireless networks.

This work is an extension of our conference paper \cite{fang2023submodel}.  Compared with \cite{fang2023submodel}, this paper offers the following contributions: (i) We derive a new convergence bound of our \texttt{HIST} algorithm that is explicitly a function of the mask sizes, and analyze its impact. (ii) Based on the newly derived bound, we develop our submodel partitioning optimization algorithm to minimize the training latency in each global round subject to a learning performance constraint. (iii) We develop the AirComp-assisted version of \texttt{HIST} to further enhance the efficiency of the model aggregation, and develop the associated beamforming and subnet partitioning optimizations. The combination of AirComp with IST provides communication scalability at the cloud server and edge server levels.



\vspace{-0.1in}
\subsection{Related Works}
\label{ssec:related}

\textbf{Hierarchical federated learning:} 
FL was first studied in a star-topology architecture where all clients are connected to a central cloud server \cite{mcmahan17a}.
The authors of \cite{liu2020client,hosseinalipour2022multi} extended FL to a hierarchical network that consists of a cloud server, edge servers, and clients,  to reduce the communication complexity of the cloud server as well as the aggregation latency.  Built upon this foundation, researchers have developed variants of hierarchical FedAvg to further improve training efficiency. Specifically, the authors of \cite{liu2022hierarchical}, \cite{abad2020hierarchical}, and \cite{malinovsky2022variance} incorporated model quantization, sparsification, and compression into the standard hierarchical FedAvg to reduce the per-round transmission load. 
Furthermore, researchers in \cite{lim2021decentralized} focused on improving the communication efficiency of hierarchical FedAvg by optimizing bandwidth allocation during the model aggregation stage within each cell. Meanwhile, the authors of  \cite{wen2022joint} formulated a joint client selection and resource allocation strategy to further enhance the efficiency when communication resources are limited. 
These works are orthogonal to ours, as we focus on a fundamentally different dimension of improving scalability in HFL, via submodel partitioning. These other techniques could be applied complementary to our approach.


\textbf{Independent submodel training:}
The exploration of submodel training commenced with the pioneering work \cite{yuan2022distributed}, where the authors introduced the concept of IST for fully connected neural networks and provided theoretical analysis under centralized settings. Subsequently, submodel training was extended to graph neural networks \cite{wolfe2023gist} and ResNets \cite{dun2022resist}. Due to its effectiveness in addressing communication, computation, and storage issues, the concept of IST was extended to distributed scenarios in \cite{diao2020heterofl}, where the authors empirically show the effectiveness of submodel training in FL. 
 Several more recent studies have characterized the convergence behavior of distributed submodel training \cite{zhou2022convergence,mohtashami2022masked,shulgin2023better}. 
{\color{black}
However, the aforementioned works have employed some restrictive assumptions in their analysis. Specifically, \cite{zhou2022convergence} assumes bounded gradients of the model, while \cite{mohtashami2022masked} imposes a constraint on model partitioning which may be difficult to satisfy in practice (Appendix D of \cite{shulgin2023better} gives an example of a strongly convex, quadratic model which violates the constraint in \cite{mohtashami2022masked}). Furthermore, the convergence result in \cite{mohtashami2022masked} imposes a step size adjustment which involves computing the gradient on the full model, which submodel partitioning and training aims to avoid. Additionally, some works, such as \cite{shulgin2023better}, focus specifically on quadratic models as opposed to more general (non-convex) ML models.
}

{\color{black}More importantly, existing works focus on cloud-based FL with a single server, and thus do not provide insights into the hierarchical case. To the best of our knowledge, \texttt{HIST} is the earliest work to integrate IST with HFL and provide theoretical analysis with experimental verification. 
As we will see, the multi-layer, multi-timescale nature of \texttt{HIST} introduces unique analytical challenges compared to star topology FL settings.
}


\textbf{AirComp-assisted FL:}
 AirComp has been investigated as a promising solution for improving the aggregation efficiency of FL over wireless networks. The authors of \cite{yang2020federated,zhu2020one,cao2021optimized} proposed to utilize AirComp to accelerate the convergence of FedAvg in single-cell networks. Subsequently, the authors in \cite{wang2022interference} extended it to multi-cell FL and focused on the problem of inter-cell interference mitigation. 
  \cite{aygun2022over} notably explored the incorporation of AirComp into the HFL architecture for edge model aggregations, demonstrating its advantages in terms of reducing the communication complexity of edge servers.
 More recently, the authors in \cite{zhou2023over} employed AirComp to support hierarchical personalized FL, and characterized the impact of AirComp on convergence, which they find introduces a non-diminishing term from aggregation distortion.
Different from these works, we explore the impact of AirComp on an HFL training process that employs submodel partitioning. One of our key contributions in this paper is to optimize submodel partitioning sizes with and without AirComp, and to demonstrate the substantial reductions in latency that can be obtained by jointly designing AirComp and IST over hierarchical wireless networks.

\section{Hierarchical Independent Submodel Training} \label{sec_algo}

In this section, we detail HFL's problem formulation, followed by the proposed \texttt{HIST} algorithm tailored to HFL. 

\vspace{-0.05in}
\subsection{System Model and Formulation}
We consider an HFL system that consists of a single cloud server, $N$ edge servers indexed by $j = 1,...,N$, and $\sum_{j=1}^N n_j$ clients, where $n_j$ is the number of clients located in the $j$-th cell. We let $\mathcal{C}_j$ denote the set of clients in cell $j$ and index them $i \in \mathcal{C}_j$. Edge server $j$ is responsible for coordinating the training process of the $n_j$ clients in cell $j$. The cloud server is in charge of global model aggregations over $N$ geographically distributed edge servers. 

The system aims to train an ML model parameterized by a $d$-dimensional vector $\bm x \in \mathbb{R}^d$. Given the loss function $l(\bm x,\xi)$ which measures the loss on sample $\xi$ for model $\bm x$, the training objective of HFL can be formulated as follows:
\equa{
\min _{\bm{x}} f(\bm{x}) &\eqdef \frac{1}{N} \sum_{j=1}^N f_j(\bm{x}), & \text{(Global loss)}\\
f_j (\bm x)&\eqdef \frac{1}{n_j} \sum_{i \in \mathcal{C}_j} F_i(\bm x),& \text{(Cell loss)} \\
F_i(\bm{x}) &\eqdef \mathbb{E}_{\xi_i \sim \mathcal{D}_i} \left[ l(\bm x, \xi_i)  \right], & \text{(Client loss)} \\
}
where $f : \mathbb{R}^d \rightarrow \mathbb{R}$, $f_j : \mathbb{R}^d \rightarrow \mathbb{R}$, and $F_i: \mathbb{R}^d \rightarrow \mathbb{R}$ represent the global loss, the loss across the $j$-th cell, and the loss of client $i$, respectively.  $ \mathcal{D}_i$ denotes the local data distribution of client $i$. 
In this work, we mainly consider the non-i.i.d. (non-independent and identically distributed) scenario where data distributions are heterogeneous across different clients.

In conventional HFL, all clients in the system train local versions of the full model. To support such training, each client needs to be equipped with enough computation, storage, and communication resources. However, 
 as the size of models continues to grow -- the trend of deep learning --  
it is prohibitive for resource-constrained clients to handle full model training. This motivates us to develop a submodel partitioning strategy for HFL, which we will introduce in the rest of this section.

\vspace{-0.05in}
{\color{black}
\subsection{Preliminaries of Model Partitioning}\label{appendix:model_partition}
\begin{figure}[t!]
    \centering
    \begin{subfigure}[b]{0.23\textwidth}
        \centering
        \includegraphics[width=\textwidth]{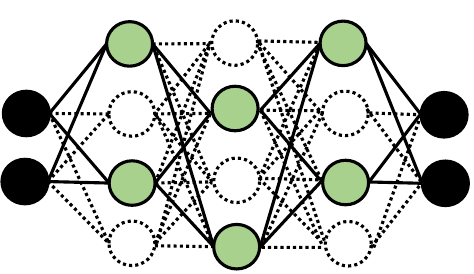}
        \caption{Proposed by \cite{yuan2022distributed}, partition ratio of $\frac{1}{2^2}$.}
        \label{fig:neural_partition1}
    \end{subfigure}
    \hfill
    \begin{subfigure}[b]{0.24\textwidth}
        \centering
        \includegraphics[width=\textwidth]{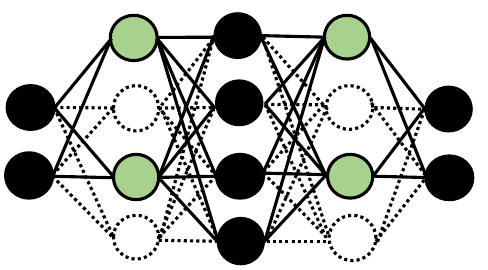}
        \caption{Our partition scheme, partition ratio of $\frac{1}{2}$.}
        \label{fig:neural_partition2}
    \end{subfigure}
    \caption{Example comparison of partition strategies for a fully connected neural network with multiple hidden layers.}
    \label{fig:neural_partitions}
    \vspace{-0.2in}
\end{figure}


Model partitioning aims to improve training efficiency by dividing a full neural network model into smaller submodels that are trained in parallel. Existing works studying such strategies include \cite{yuan2022distributed} and \cite{zhou2022convergence}, which partition the hidden neurons and distribute them across clients. Fig.~\ref{fig:neural_partition1} summarizes the method proposed in \cite{yuan2022distributed} for fully connected layers, where the neurons at every hidden layer are divided into $N$ parts for $N$ different submodels, one per client (shown here for $N = 2$, green and white), while leaving the input and output neurons independent of the partitioning. This results in a partition ratio of $1/N^2$, indicating the ratio of the submodel's size to the full model's size. With this strategy, the total number of parameters across $N$ submodels is expected to be lower than that of the original full model, i.e., some parameters are expected to be missing. For example, in Fig.~\ref{fig:neural_partition1}, all link weights between green neurons of one layer and white neurons of an adjacent layer are severed, since the neurons are assigned to different clients during partitioning.

We instead consider a strategy where we in effect partition by \textit{link} instead of neurons. To accomplish this, we can partition the neurons in one hidden layer (e.g., the $l$-th hidden layer) into $N$ parts, and leave the subsequent layer (e.g., the $(l+1)$-th hidden layer) intact, with all its neurons shared across the $N$ parts.
This alternating approach is depicted in Fig.~\ref{fig:neural_partition2} (again for $N = 2$). Each client then receives one of the $N$ parts, thereby preserving all the link weights (parameters) between two consecutive layers with none repeated (e.g., green-black links for client 1, white-black links for client 2). This results in a partition ratio of $1/N$.}

\vspace{-0.05in}
\subsection{Algorithm Description}

\begin{figure}[t]
\centering
\includegraphics[width=8.5cm]{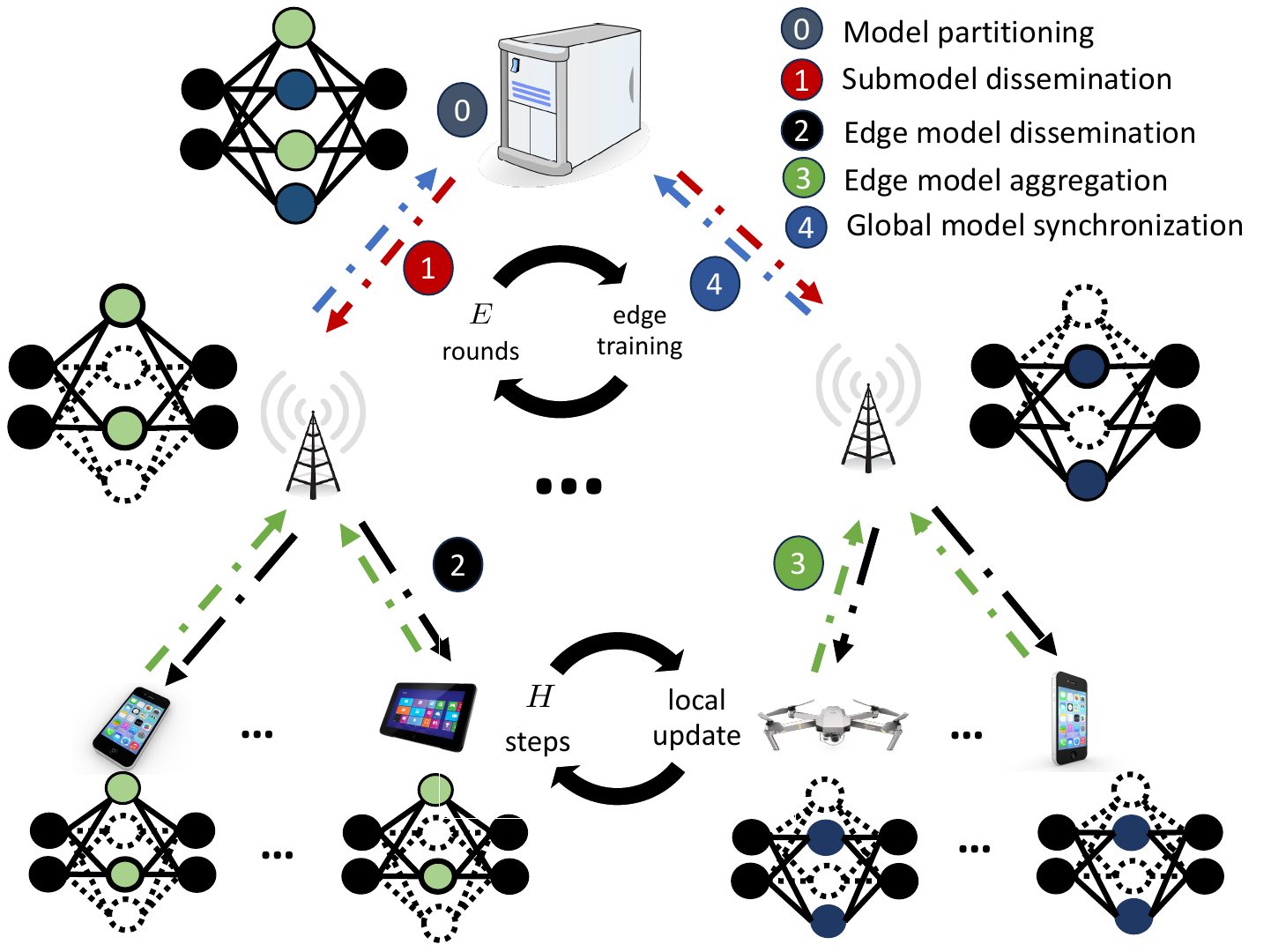}
\caption{Overview of the proposed \texttt{HIST} algorithm. Each cell is responsible for training only a specific partition of the full model in each global round, with the specific submodel partitioning changing over each round.}\label{fig:HIST}
\vspace{-0.2in}
\end{figure}

An overview of our hierarchical federated submodel training (\texttt{HIST}) algorithm is presented in Fig. \ref{fig:HIST} and Algorithm \ref{alg:HFIST}. Similar to hierarchical FedAvg, the cloud server periodically aggregates edge models from edge servers, while each edge server periodically aggregates local models from the {\color{black}active} clients within the corresponding cell. {\color{black}Active clients refer to those selected to participate in the current training round.}
The key difference in \texttt{HIST} is that clients will only store, update, and exchange a portion of the model in each training iteration.

Specifically, at the start of the $t$-th global round, with the current global model denoted as $\bar{\bm x}^t$, the cloud server initiates the training process by partitioning $\bar{\bm x}^{t}$ as follows: 
\begin{equation}\label{}
\{\bm p_j^t \odot \bar{\bm x}^{t} \mid j = 1,2,\ldots, N \}, 
\end{equation}
where $\odot$ denotes the Hadamard Product operation {\color{black} and $\bm p_j^t$ is the mask for edge server $j$. This mask will in general change over global rounds, and thus is indexed by $t$.\footnote{
It is worth noting that for any submodel, denoted $\bm x_j$, there exists a mask $\bm p$ satisfying $\bm x_j = \bm p \odot {\bm x}$, where $\bm x$ denotes the full model.}
} These masks are binary vectors and satisfy
\begin{equation}\label{mask_IST}
\bm p_j^t \odot \bm p_{j^{\prime}}^t = \bm 0, \forall j^{\prime} \neq j,~ \text{and}~ \sum_{j=1}^N \bm p_j^t = \bm 1,
\end{equation}
which can be satisfied through our strategy in Sec.~\ref{appendix:model_partition}.

These submodels are then distributed to the corresponding edge servers. Specifically, edge server $j$ receives $\bm p_j^t \odot \bar{\bm x}^{t}$ from the cloud server and initializes its model for global round $t$ as 
\begin{equation}\label{how_to_partition}
\bar{\bm x}^{t,0}_j = \bm p_j^t \odot \bar{\bm x}^{t}, \forall j \in \{1,2,\ldots, N\}.
\end{equation}
Subsequently, edge server $j$ disseminates $\bar{\bm x}^{t,0}_j$ to the clients in its cell for local model initialization, i.e., 
\[{\bm x}^{t,0}_{i,0} = \bar{\bm x}^{t,0}_j, \forall i\in \mathcal{C}_j.\]
In our notation, the $0$'s in the subscript and superscript refer to client-level and cell-level model initializations, respectively.
Once the clients receive this model from the edge server, they commence local training on their own data.

In \texttt{HIST}, each global round consists of $E$ steps of edge aggregation at each edge server and one global aggregation at the cloud server. We use $(t,e)$ to denote the $e$-th edge round at global round $t$. Each edge round in turn includes $H$ steps of local updates at each client and one edge aggregation.
 The essential steps conducted by clients, edge servers, and the cloud server in our algorithm are outlined as follows. 

 \textbf{Clients:}
 {\color{black}Let $\mathcal{C}_j^{t,e}$ denote the set of clients participating in the $(t,e)$-th round of training within the $j$-th cell, with cardinality $|\mathcal{C}_j^{t,e}| = n_j^{\prime}$. This set is assumed to be a uniform sampling from the full set of clients $\mathcal{C}_j$.
 Each client $i \in \mathcal{C}_j^{t,e}$} computes stochastic gradients with respect to its corresponding submodel, and updates the local model for $H$ steps via the following iteration:
 \begin{equation}\label{sgd}
 {\bm x}^{t,e}_{i,h+1} = {\bm x}^{t,e}_{i,h} - \gamma {\bm p}^t_j \odot \nabla l({\bm x}^{t,e}_{i,h}, \xi^{t,e}_{i,h}) , h = 0,1,\ldots, H\!-\!1,
 \end{equation} 
 where $\gamma$ represents the learning rate 
 and $t$, $e$, and $h$ denote the number of global rounds, edge rounds, and local iterations, respectively.
 Note that mask ${\bm p}^t_j$ is kept invariant during one global round.
 Subsequently, each client uploads the updated submodel to the corresponding edge server.
 
	 \textbf{Edge Servers:} After every $H$ steps of local submodel updates at the clients, each edge server $j$ aggregates the local models of {\color{black}active clients} within its coverage as 
\begin{equation}\label{fedavg}
 \bar{\bm x}^{t,e+1}_j = \frac{1}{n_j^{\prime}} \sum_{i\in {\color{black} \mathcal{C}_j^{t,e}}} {\bm x}^{t,e}_{i,H}, \forall j \in \{1,2,\ldots, N\}. 
 \end{equation}
  If less than $E$ rounds of edge training have passed, the servers disseminate \eqref{fedavg} to their clients to initialize $\bm x_{i,0}^{t,e+1}$ for the next iteration of local updates. Otherwise, each edge server uploads the aggregated model to the cloud to update the global model.
        
         \textbf{Cloud Server:} Once the cloud server receives the latest models from edge servers,
         it updates the global model as:
\begin{equation}\label{synchronization}
        \bar{\bm x}^{t\!+\!1} = \sum_{j=1}^N \bar{\bm x}^{t,E}_j.
        \end{equation}
         Subsequently, the cloud server repartitions the global model $\bar{\bm x}^{t\!+\!1}$ based on a newly generated set of masks as $\bar{\bm x}_j^{t\!+\!1,0} = \bm p^{t\!+\!1}_j \odot  \bar{\bm x}^{t\!+\!1}, \forall j \in \{1,2,\ldots, N\}$.
       Finally, $\bar{\bm x}_j^{t\!+\!1,0}$ is sent to cell $j$ to initiate the next round of training.

With the proposed algorithm, clients and edge servers are not required to store or manipulate the full size of the global model. This enables \texttt{HIST} to alleviate communication, computation, and storage burdens for clients and edge servers compared to conventional HFL.  In particular, assuming each cell receives an equal-sized submodel partition, the resource requirements decrease by a factor of $N$. In Section \ref{sec_theory}, our convergence analysis will reveal the impact of this mask selection and other \texttt{HIST} parameters on model training performance.

{\color{black}
\begin{remark}
While our partitioning strategy in Section \ref{appendix:model_partition} is presented for fully connected layers, it can be extended to convolutional layers as well.
Concretely, the parameters of a convolutional layer can be represented as a 4D tensor of dimension $d^{s} \times d^{s} \times d^{\text{in}} \times d^{\text{ker}}$, where $d^{s}$ denotes the spatial size, $d^{\text{in}}$ is the number of channels of the input to this layer, and $d^{\text{ker}}$ corresponds to the number of kernels, which also determines the number of channels in the output feature map of this layer. The output feature map serves as the input to the next convolutional layer. In other words, $d^{\text{in}}$ for a given layer depends on the number of kernels in the last layer. Therefore, the $l$-th convolutional layer's tensor dimension can be expressed as $d^{s}_{l} \times d^{s}_{l} \times d_{l\!-\!1}^{\text{ker}} \times d_l^{\text{ker}}$.
This can then be treated as a special matrix of size $d_{l\!-\!1}^{\text{ker}} \times d_l^{\text{ker}}$, where each element is a $d^{s}_l \times d^{s}_l$ matrix. These matrices are analogous to connections between neurons in a fully connected layer. Thus, for layer $l$, we can randomly assign a subset of convolutional kernels to each submodel. This partitioning in the $l$-th layer will specify a partitioning of the subsequent $(l+1)$-th layer as well, because the output channels of the $l$-th layer correspond to the input channels of the $(l+1)$-th layer. Hence, this becomes an alternating partitioning strategy, analogous the fully connected layer approach described in Sec.~\ref{appendix:model_partition}. Specifically, we can partition the convolutional kernels in the $l$-th layer by selecting a subset of kernels for each group, assigning all kernels in the $(l+1)$-th layer to each group, and then partitioning again in the $(l+2)$-th layer.
\end{remark}
}

\vspace{-0.1in}
\section{Convergence Analysis of \texttt{HIST}}\label{sec_theory}
This section provides convergence analysis for the proposed \texttt{HIST} algorithm. 
We note that the convergence proof of conventional hierarchical FedAvg cannot be directly extended to our case, due to the effect of the masks in submodel partitioning.
Specifically, the mask $\bm p_j^{t} $ compresses not only the gradient but also the model, with an effect of the form $\bm p_j^{t} \odot \nabla F_i(\bm p_j^{t} \odot \bm x)$, while many existing works only investigate compressing the gradient $\nabla F_i(\bm x)$. 
Theoretical analysis on model compression \cite{khaled2019gradient} in FL is quite limited. Even in the single-cell scenario, existing analyses of IST \cite{mohtashami2022masked,zhou2022convergence} rely on some stronger assumptions like bounded gradients (e.g., see Assumption 3 in \cite{zhou2022convergence}) and particular forms of mask partitions \cite{mohtashami2022masked} which we aim to overcome, as discussed in Sec.~\ref{ssec:related}.

{\color{black} The hierarchical architecture of {\tt HIST} further complicates the analysis, due to the multiple layers and multi-timescale communications.
Specifically, the analysis of IST over FL's star topology cannot be easily extended to \texttt{HIST} where the cloud server does not directly communicate with the clients; instead, communication occurs at two different timescales—between the cloud server and edge servers, and between edge servers and clients, with nested aggregation periods. 
This leads to what is in effect a hierarchical drift of submodels in the system, where (i) client submodels drift away from their group/cell aggregations, and (ii) cell submodels drift away from their global versions (i.e., across different partitionings), with respect to diversity in client data distributions.
}

\begin{algorithm}[t]
	\caption{Hierarchical Independent Submodel Training (\texttt{HIST}) Algorithm}
	\label{alg:HFIST}
	\begin{algorithmic}[1]
		\STATE Initialization: masks $\{\bm p^0_1, \bm p^0_2, \ldots, \bm p^0_N\}$, initial models $\bar{\bm x}^{0}$, ${\bm x}_{i,0}^{0,0} = \bar{\bm x}_j^{0,0} = \bm p^0_j \odot  \bar{\bm x}^0, \forall i \in \mathcal{C}_j, \forall j$, learning rate $\gamma$
		\FOR{$t=0,1,\dotsc,T-1$}
		\FOR{$e=0,1,\ldots,E-1$}
            \STATE {\color{black} Uniformly sample a subset $\mathcal{C}_j^{t,e}$ of $\mathcal{C}_j$ for each $j$}
		\FOR{$h=0,1,\ldots,H-1$}
		\STATE {\color{black}Clients in set $\mathcal{C}_j^{t,e}, \forall j$ update local models by \eqref{sgd} in parallel}
		\ENDFOR
		\STATE Edge servers update edge models $\bar{\bm x}_{j}^{t,e\!+\!1}$ by \eqref{fedavg} \\ in parallel
		\STATE Edge servers broadcast the updated edge models to clients: ${\bm x}^{t,e\!+\!1}_{i,0} \leftarrow \bar{\bm x}^{t,e\!+\!1}_j, ~ \forall i \in \mathcal{C}_j$  
		\ENDFOR
  \STATE Cloud server updates the global model $\bar{\bm x}^{t\!+\!1}$ by \eqref{synchronization}
		\STATE Generate new masks $\{\bm p^{t\!+\!1}_1, \!\bm p^{t\!+\!1}_2, \ldots, \!\bm p^{t\!+\!1}_N\!\}$
		\STATE Partition the global model according to \eqref{how_to_partition} and send the obtained submodels $\bar{\bm x}^{t\!+\!1}_j$ to clients in cell $j$ to initiate the next round of training, $\bm x_{i,0}^{t\!+\!1,0} \! \leftarrow \bar{\bm x}^{t\!+\!1}_j$, $\forall i \in \mathcal{C}_j, \forall j$
		\ENDFOR 
	\end{algorithmic}
\end{algorithm}
\setlength{\textfloatsep}{10pt}

\vspace{-0.05in}
\subsection{Assumptions}
Our theoretical analysis focuses on general smooth functions under a non-i.i.d data setting. The detailed assumptions are listed as follows.
\begin{assumption}\label{assump_lowerbound}
    The global loss function $f(\bm x)$ has a lower bound $f_*$, i.e., $f(\bm x) \geq f_*, \forall \bm x.$
\end{assumption}
\begin{assumption}\label{assump_smoothness}
	 The client loss $F_i$ is differentiable and $L$-smooth, i.e., 
  for any $\bm x$ and $\bm y$, 
	\equa{
            & \|\nabla F_i (\bm x) - \nabla F_i (\bm y)\|^2 \leq  L\| \bm y - \bm x\|, \forall i, \\
            & F_i(\bm y) \leq  F_i(\bm x) + \<\nabla F_i(\bm x), \bm y- \bm x> + \frac{L}{2} \|\bm y - \bm x\|^2, \forall i.
	}
\end{assumption}
With Assumption \ref{assump_smoothness}, one can also claim that the cell losses $f_j$, $\forall j$ and global loss $f$ are $L$-smooth.  
\begin{assumption}\label{unbiased_gradient}
The stochastic gradient $\nabla l(\bm x, \xi_i)$ is an unbiased estimator of the true gradient, i.e., $\mathbb{E}_{\xi_i \sim \mathcal{D}_i}[\nabla l(\bm x, \xi_i)] = \nabla F_i (\bm x), \forall \bm x, \forall i$.
\end{assumption}
\begin{assumption}\label{assump_randomness_sgd}
        The variance of the stochastic gradient $\nabla l(\bm x, \xi)$ is bounded as
	\begin{equation}
		\mathbb{E}_{\xi \sim \mathcal{D}_i}\| \nabla l(\bm x, \xi) -  \nabla F_i(\bm x) \|^2 \leq \sigma^2, \forall \bm x, \forall i.
	\end{equation} 
\end{assumption}
\begin{assumption}\label{assump_gradient_dissimilarity_edge} 
 There exists a constant $\delta_1^2$ that bounds the gradient dissimilarity between the global loss function and edge loss functions, i,e.,
		\begin{equation}
			\frac{1}{N}\sum_{j=1}^N\|  \nabla f_j({\bm x}) - \nabla f(\bm x)  \|^2 \leq \delta_1^2, \forall \bm x.
		\end{equation}
\end{assumption}
\begin{assumption}\label{assump_gradient_dissimilarity_client}
There exists a constant $\delta_2^2$ that bounds gradient dissimilarities between edge loss functions and client loss functions, i,e.,
		\begin{equation}
			\frac{1}{n_j}\sum_{i\in \mathcal{C}_j} \|  \nabla F_i({\bm x}) - \nabla f_j(\bm x)  \|^2 \leq \delta_2^2, \forall \bm x, \forall j.
		\end{equation}
\end{assumption}

Assumptions \ref{assump_lowerbound}-\ref{assump_randomness_sgd} have been widely adopted in the context of stochastic non-convex and smooth optimization \cite{fang2022communication,li2020federated}. Assumptions \ref{assump_gradient_dissimilarity_edge} and \ref{assump_gradient_dissimilarity_client} serve to characterize the degree of data heterogeneity across cells and across clients, which are commonly used within the HFL literature \cite{wang2022demystifying,feng2022mobility,pervej2023hierarchical}.



\vspace{-0.05in}
{\color{black}
\subsection{Full Client Participation}}\label{sec_analytical}

{\color{black}
The global synchronization \eqref{synchronization} occurs at intervals of every $E$ steps of edge model updates. If we were to directly consider $\{\nabla f(\bar{\boldsymbol x}^{t})\}$, it would be difficult to capture the effect of local aggregation on the global model iteration as there are $E$ local aggregations within each global iteration. 
Moreover, it is infeasible to establish a close connection between \(\nabla f(\bar{\boldsymbol{x}}^{t})\) and \(\boldsymbol{x}_{i,h}^{t,e}, \forall i, h,\) due to a large lag in updates, as \(\boldsymbol{x}_{i,h}^{t,e}\) is updated incrementally with \(h\), while \(\bar{\boldsymbol{x}}^{t}\) is updated only after \(h\) completes \(E\) cycles from \(0\) to \(H-1\). 
To tackle this, we introduce iterates
\begin{equation}\nonumber
    \{ \hat{\bm x}^{t,e} =  \sum_{j=1}^N \bar{\bm x}^{t,e}_j \!\mid \!t=0, 1,\ldots, T-1; e=0, 1,\ldots, E-1 \}
\end{equation}
serving as a virtual sequence of global models realized at each edge round $e$. We will capture the convergence of \texttt{HIST} by characterizing the bound of the gradient of the global loss function evaluated on the virtual global model iterate, i.e., $\|\nabla f(\hat{\boldsymbol x}^{t,e})\|^2$.
}

\subsubsection{Analytical Results}
We first investigate the convergence properties of {\color{black}Algorithm \ref{alg:HFIST} under full client participation} (i.e., when $\mathcal{C}_j^{t,e} = \mathcal{C}_j$) by characterizing the evolution of $\|\nabla \!f\!\left( \hat{\bm x}^{t,e}  \right) \! \|^2$, to capture how fast the global model converges to a stationary point. 
The following gives one of our main results in this paper:
\begin{theorem}\label{theorem_non_iid_hierarchical}
	Suppose that Assumptions \ref{assump_lowerbound}-\ref{assump_gradient_dissimilarity_client} hold, $N\geq 2$, and the step size satisfies 
\begin{equation}\label{condition_gamma}
\resizebox{0.87\linewidth}{!}{$
\gamma \!\leq\!  \min \!\left\{ \!
\frac{1}{45\sqrt{N}EHL}, \frac{\tilde{N}}{NHL},\frac{1}{NH^2L},\frac{1}{N(N+1)E^2H^2L}
\!\right\}.$}
\end{equation}
	Then, for an arbitrary mask partitioning satisfying~\eqref{mask_IST} in each iteration, the \texttt{HIST} algorithm under full client participation satisfies
\begin{align}\label{equation_theorem_non_iid_hierarchical}
	\frac{1}{TE} \sum_{t=0}^T \sum_{e=0}^{E-1} &\mathbb{E} \left\| \nabla f(\hat{\bm x}^{t,e} ) \right\|^2
 \leq 4 \frac{f(\bar{\bm x}^0) - f_*}{\gamma TEH}  + 100 \gamma \tilde{N} L \sigma^2 \nonumber \\
 &+ 1356 \gamma L \delta_1^2 + 60 \gamma L \delta_2^2 +  6 \frac{N}{d} \frac{1}{T}\sum_{t=0}^{T-1} d^t_{\textrm{max}} \delta_1^2  \nonumber \\
& +  48 (N-1) L^2 \frac{1}{T}\sum_{t=0}^{T-1} \mathbb{E} \left \|\bar{\bm x}^{t}  \right\|^2,
\end{align}
	where $\tilde{N} = \sum_{j=1}^N \frac{1}{n_j}$ and
\begin{equation}\label{eq:dmax}
d^t_{\textrm{max}} = \max \{\|\bm p_1^t\|_1, \|\bm p_2^t\|_1, \ldots, \|\bm p_N^t\|_1\}, \forall t.
\end{equation}
\end{theorem}

\begin{proof}
    Please refer to Appendix \ref{proof_theorems}.
\end{proof}

Theorem \ref{theorem_non_iid_hierarchical} presents an upper bound for the optimality gap of \texttt{HIST}, characterized by the time-averaged squared gradient norm of the global function at the global virtual model sequence. The first term in this upper bound gives the effect of the initial optimality gap $f(\bar{\bm x}^0) \!-\! f_*$ on the convergence behavior. The second term shows how the optimality gap is related to the variance of stochastic gradients $\sigma^2$. This variance can be mitigated by enlarging the mini-batch size during the computation of stochastic gradients. 
The third, fourth, and fifth terms reflect the influence of non-i.i.d. characteristics within the cell ($\delta_2^2$) and across cells ($\delta_1^2$) on convergence. The fifth term shows that the impact of cross-cell data dissimilarity becomes more pronounced as $d^t_{\textrm{max}}$ increases: certain cells must receive smaller model partitions $\|\bm p_i^t\|_1$ to accommodate an increasing $d^t_{\textrm{max}}$, resulting in their datasets becoming reflected in relatively fewer parameters in round $t$.
The last term demonstrates how the norm of the synchronized global model $\left \|\bar{\bm x}^{t}  \right\|^2$ also impacts the optimality gap. 

In addition, the step size $\gamma$ is a configurable parameter that affects the first four terms of the derived upper bound. Focusing on one particular step size and employing a uniform random partitioning gives rise to the following corollary.

\begin{corollary}\label{corollary_non_iid_hierarchical}
	Suppose that Assumptions \ref{assump_lowerbound}-\ref{assump_gradient_dissimilarity_client} hold, and the masks $\{\bm p_1^t, \bm p_2^t, \ldots, \bm p_N^t\}$ are uniformly and randomly generated based on \eqref{mask_IST}.  Let the step size $\gamma = (T\!E\!H)^{-\!\frac{1}{2}}$ in which $T$ is large enough to satisfy \eqref{condition_gamma}. Then for $N\geq 2$, the \texttt{HIST} algorithm satisfies
	\begin{align}\label{equa_corollary_1}
		& \frac{1}{TE} \sum_{t=0}^T \sum_{e=0}^{E-1} \mathbb{E} \left\| \nabla f(\hat{\bm x}^{t,e} ) \right\|^2 
		\leq \mathcal{O}\left( \tilde{N} (T\!E\!H)^{-\frac{1}{2}} \right) \nonumber \\
		& + \mathcal{O}\left((T\!E\!H)^{-\frac{1}{2}} \right) + \mathcal{O} \left( \delta_1^2 + (N\!-\!1) 
  \frac{1}{T} \sum_{t=0}^{T-1} \mathbb{E} \left\|\bar{\boldsymbol{x}}^{t}\right\|^2 \right),
	\end{align}
	where $\tilde{N}$ is described in Theorem \ref{theorem_non_iid_hierarchical}.
\end{corollary}
\begin{proof}
    Please refer to Appendix \ref{sec_appen_proof_theorem}.
\end{proof}
\begin{remark}\label{remark_n_j}
In Corollary \ref{corollary_non_iid_hierarchical}, the first term is a function of $\tilde{N}=\sum_{j=1}^N \frac{1}{n_j}$, which is affected by the relationship between the number of clients in each cell, denoted as $n_j$, and the total number of cells, denoted as $N$. 
When the number of clients in each cell tends to be larger than the total number of cells, the convergence rate of the diminishing terms in the derived upper bound is primarily determined by $\mathcal{O}\left( (T\!E\!H)^{-\frac{1}{2}} \right)$. On the other hand, if the number of clients in each cell is significantly smaller than the total number of cells, $\tilde{N}$ becomes influential, and the convergence rate is dominated by $\mathcal{O}\left( \tilde{N}^{\frac{1}{2}} (T\!E\!H)^{-\frac{1}{2}} \right)$.
 
\end{remark}

\subsubsection{Implications of Theorem \ref{theorem_non_iid_hierarchical} and Corollary \ref{corollary_non_iid_hierarchical}}\label{sec_theory_C}

The results from Section \ref{sec_analytical} enable us to explore the performance-efficiency tradeoff induced by different parameters in \texttt{HIST}. We discuss several aspects here.

\emph{Non-diminishing terms:} With the step size chosen in Corollary \ref{corollary_non_iid_hierarchical}, the first four terms in \eqref{equation_theorem_non_iid_hierarchical} will diminish to zero as the number of total iterations $T$ grows large. The remaining two terms are non-diminishing parts that arise due to submodel training. Therefore, \texttt{HIST} converges to the neighborhood of a stationary point of the loss function under the aforementioned conditions. A similar phenomenon has also been reported in the single-cell case \cite{yuan2022distributed,shulgin2023better}. 

{\color{black}
\emph{The impact of $N$:} As $N$ increases, i.e., as the clients in the system are divided into smaller cells, the size of the submodels gets smaller, which naturally provides computation, communication, and storage reductions for clients. However, as observed in Corollary \ref{corollary_non_iid_hierarchical}, with all else constant, a larger $N$ causes the sequence to deviate further from the stationary point, which is to be expected since each cell's data is only being used to update one model partition in each global round. Overall, this highlights the trade-off between convergence performance and resource utilization. Note also that in the extreme case when $N = 1$, $\delta_1^2 = 0$, since there is only one cell. The third, fifth, and sixth terms in \eqref{equation_theorem_non_iid_hierarchical} vanish in this case, and the result reduces to FedAvg with convergence to a stationary point, as expected. We will investigate the tradeoff associated with increasing $N$ numerically in Sec.~\ref{simu_pure_HIST}.
}

\emph{The effect of $n_j$:}  {\color{black}Under the bounded data heterogeneity assumptions among cells and clients, i.e., Assumptions \ref{assump_gradient_dissimilarity_edge} and \ref{assump_gradient_dissimilarity_client}, 
we see that increasing the number of clients $n_j$ in each cell $j$ has a positive effect on the convergence. This result aligns with the linear speedup characteristics observed in conventional FedAvg and hierarchical FedAvg algorithms \cite{jianyu, wang2022demystifying, fang2022communication}.} 
To be more specific, $\tilde{N}$ monotonically decreases as $n_j$ increases for all $j$, and in Corollary \ref{corollary_non_iid_hierarchical}, we see that the smaller the value of $\tilde{N}$, the faster the convergence. This confirms our expectation that, all else constant, a larger cell size should provide a more well-crafted edge model in each training round.

\emph{The choices of $H$ and $E$:}
The number of local updates, $H$, and the number of edge aggregations, $E$, are controllable parameters that impact the communication frequency. As $H$ increases, the aggregation frequency at edge servers will become smaller,  reducing the communication load between clients and the edge server. On the other hand, a large  $E$ induces fewer global synchronizations, which alleviates the communication burden between edge servers and the cloud server. Intuitively, however, these values must be upper bounded to guarantee a certain frequency of global aggregations. The maximum values of $E$ and $H$ can be derived from the condition on the step size $\gamma$ in Theorem \ref{theorem_non_iid_hierarchical}.  
 Specifically, to ensure that the step size $\gamma = (T\!E\!H)^{-\!\frac{1}{2}}$ adheres to the conditions specified in \eqref{condition_gamma} from Corollary \ref{corollary_non_iid_hierarchical} , $H$ and $E$ can be set as on the order of $\mathcal{O}\left(T^{\frac{1}{3}} N^{-\!\frac{4}{3}}  H^{-\!\frac{1}{3}}\right)$ and $\min \left \{ \mathcal{O}\left(\tilde{N}^{2}TE N^{-2}\right) , \mathcal{O}\left( T^{\frac{1}{3}} N^{-\!\frac{4}{3}}  E^{-\!\frac{1}{3}} \right) \right\}$ at most, respectively. 

\emph{The effect of submodel partitioning:}
According to equations (\ref{equation_theorem_non_iid_hierarchical}) and (\ref{eq:dmax}) in Theorem \ref{theorem_non_iid_hierarchical}, a uniform partition with $\|\bm p_1^t\|_1 = \ldots= \|\bm p_N^t\|_1 = d/N$ leads the fifth term of the upper bound i.e., $6 \frac{N}{d} \frac{1}{T}\sum_{t=0}^{T-1} d^t_{\textrm{max}} \delta_1^2$, to attain its minimum. All else constant, varying the mask sizes across cells may cause the larger submodel partitions to be less refined from training in each global round. Nevertheless, this strategy may be problematic from a resource utilization perspective, when the communication and computation capabilities are heterogeneous across the clients. In other words, there will be a trade-off between learning performance and training efficiency/latency depending on the choice of mask sizes. This motivates us to optimize the mask sizes to balance between these competing objectives in Section \ref{comparison_existing_work}.

{\color{black}
\subsection{Partial Client Participation}\label{sec_partial_participation}
We also analyze the convergence of the \texttt{HIST}
algorithm under partial client participation. We have the following:

\begin{theorem}\label{theorem_non_iid_hierarchical_partial}
	Suppose that Assumptions \ref{assump_lowerbound}-\ref{assump_gradient_dissimilarity_client} hold, $N\geq 2$, and the step size satisfies 
 \begin{equation}\label{condition_gamma_partial}
\gamma \!\leq\!  \min \!\left\{ \!
\frac{1}{64\sqrt{N}EHL}, \frac{\tilde{N}^{\prime}}{2NHL},\frac{1}{2NH^2L},\frac{1}{2N(N+1)E^2H^2L}
\!\right\}.
\end{equation}
	Then, for an arbitrary mask partitioning satisfying~\eqref{mask_IST} in each iteration, the \texttt{HIST} algorithm under partial client participation satisfies
\begin{align}\label{equation_theorem_non_iid_hierarchical_partial}
	& \frac{1}{TE} \sum_{t=0}^T \sum_{e=0}^{E-1} \mathbb{E} \left\| \nabla f(\hat{\bm x}^{t,e} ) \right\|^2
 \leq  \mathcal{O}\left( \tilde{N}^{\prime} (TEH)^{-\frac{1}{2}} \right) \nonumber \\
 & + \mathcal{O}\left((TEH)^{-\frac{1}{2}} \right) + \mathcal{O}\left( \tilde{N}^{\prime} H^{\frac{1}{2}} (TE)^{-\frac{1}{2}} \right) \nonumber \\
		&+   \mathcal{O} \left(\frac{N}{d} \frac{1}{T}\sum_{t=0}^{T-1} d^t_{\textrm{max}} \delta_1^2 + (N-1) L^2 \frac{1}{T}\sum_{t=0}^{T-1} \mathbb{E} \left \|\bar{\bm x}^{t}  \right\|^2 \right),
\end{align}
	where $\tilde{N}^{\prime} = \sum_{j=1}^N \frac{1}{n_j^{\prime}}$, $n_j^{\prime}$ denotes the cardinality of the participating client set at cell $j$, i.e., $|\mathcal{C}_j^{t,e}| = n_j^{\prime}$, and $d^t_{\textrm{max}}$ is defined in Theorem \ref{theorem_non_iid_hierarchical}.
\end{theorem}
\begin{proof}
    Please refer to Appendix \ref{proof_theorems}.
\end{proof}

The impact of the number of participating clients in Theorem~\ref{theorem_non_iid_hierarchical_partial} is captured by $\tilde{N}^{\prime}$. As the number of participating clients in any cell increases, $\tilde{N}^{\prime}$ decreases, and a faster convergence speed can be achieved, as we would expect. Comparing \eqref{equation_theorem_non_iid_hierarchical_partial} and \eqref{equa_corollary_1}, the key differences between the full and partial client participation bounds are as follows: (1) the noise variance term includes $\tilde{N}$ for full participation versus $\tilde{N}^{\prime}$ for partial participation, reflecting reduced client involvement; (2) an additional divergence term $\mathcal{O}\left( \tilde{N}^{\prime} H^{\frac{1}{2}} (TE)^{-\frac{1}{2}} \right)$ emerges in the partial participation case, accounting for the increased randomness introduced by limited client engagement, analogous to observations in conventional FL scenarios \cite{jianyu,fang2022communication}; and (3) the bound shown in \eqref{equation_theorem_non_iid_hierarchical_partial} is derived under a smaller learning rate (comparing \eqref{condition_gamma} and \eqref{condition_gamma_partial}), which is necessary to ensure convergence to a stationary point under the higher level of randomness associated with partial participation.
}

\subsection{Comparison with Convergence Analysis of Existing Works} \label{comparison_existing_work}
A key distinction between our convergence analysis and previous works \cite{yuan2022distributed,mohtashami2022masked,zhou2022convergence,shulgin2023better} on IST is that we consider the hierarchical network architecture in this paper. Apart from this, our focus and assumptions are also different from these works. To be specific, the analysis in \cite{yuan2022distributed,mohtashami2022masked} concentrates on the centralized setting. Consequently, the non-i.i.d. influence is not encompassed in their works, which is an indispensable factor in FL. Additionally, \cite{yuan2022distributed,mohtashami2022masked,zhou2022convergence} make use of stronger assumptions in their analysis. 
In particular, \cite{yuan2022distributed, zhou2022convergence} assume Lipschitz continuity for the loss functions. Moreover, Assumption $5$ adopted in \cite{zhou2022convergence} is difficult to ensure as it requires a constant to bound the client drift. On the other hand, \cite{mohtashami2022masked} imposes some extra conditions on the masks, which may not hold in practical settings as argued in \cite{shulgin2023better}. The authors in \cite{shulgin2023better} then provided a tighter convergence bound for distributed IST based on milder assumptions. While their focus is on quadratic loss functions, ours is on general non-convex and smooth functions, which are more common in FL settings.

\section{Optimization for OMA-based \texttt{HIST}}\label{sec_partion_opt}


In this section, we develop a methodology for optimizing submodel partitioning sizes to balance training efficiency and learning performance.
{\color{black}For simplicity, we consider full client participation in this section.} We construct a model for training latency with a standard orthogonal multiple access (OMA) transmission protocol in Section \ref{section_partition_opt}. 
The convergence performance for our mask size optimization in Section \ref{section_4_B_maskopt} comes from Section \ref{sec_theory}.

\subsection{Training Latency Analysis }\label{section_partition_opt}

Since the edge and cloud servers possess significantly larger communication and computation resources than clients, we focus on the delays incurred from local model updates and uplink communication at the clients, as done in \cite{liu2021reconfigurable, cao2022transmission,sery2021over}.
Additionally, we assume that the computational and communication capabilities of clients remain stable throughout a specific global training round \cite{wang2022interference,zhu2023airfl}. 
Formally, we let $\mathcal{R}_i^t$ represent the data rate (in bits/sec) for uplink communication of client $i$ in the $t$-th global training round, and we let $\mathcal{F}_i^t$ represent the CPU frequency (in Hz). 
Furthermore, the communication load (in bits) for uploading a complete model is denoted by $L_0$, while the required number of CPU cycles for a single mini-batch update of the full model is represented by $V_0$. 
Without loss of generality, we assume that the latency required for local computation is linearly proportional to the model size\footnote{The analysis can be easily extended to any other functions that describe the relationship between the model size and the computation time.} following prior works \cite{dongjun_inforcom,dongjun_TMC}.

{\color{black}Assuming a traditional frequency division multiple access (FDMA) scheme, which is one of the standard OMA protocols, the communication latency for each client in cell $j$ is 
\begin{equation}\label{latency_client_i}
\frac{\|\bm p^{t}_j\|_1 L_0}{\mathcal{R}^t_i d},
\end{equation}
where $\mathcal{R}_i^t =  \frac{B}{n_j} \log(1 + \text{SNR}_i \|\bm{h}_i^t\|^2)$ represent the data rate (in bits/sec) for uplink communication. $\frac{B}{n_j}$ denotes the bandwidth allocated to each client in cell $j$, $\text{SNR}_i$ denotes the signal noise ratio for client $i$, and $\bm{h}_i^t$ represents the channel. 
As a result, the overall latency of one round training of \texttt{HIST}, including computation and communication latency, for the $(t,e)$-th round   within cell $j$ can be expressed as
\equa{\label{latency_comm_comp}
\max_{i \in \mathcal{C}_j} \left \{ H \frac{\|\bm p^{t}_j\|_1 V_0}{\mathcal{F}^t_i d} + \frac{\|\bm p^{t}_j\|_1 L_0}{\mathcal{R}^t_i d} \right\}.
}
}
Given that the duration required for a single update of the global model depends on the speed of the slowest edge server, we can express the overall latency for each global model update as follows:
\equa{\label{eq:overall_latency}
\max_{j\in\{1,2,\dots,N\}} \left \{E \max_{i \in \mathcal{C}_j} \left \{ H \frac{\|\bm p^{t}_j\|_1 V_0}{\mathcal{F}^t_i d} + n_j \frac{\|\bm p^{t}_j\|_1 L_0}{\mathcal{R}^t_i d} \right\} \right\}.
}

\subsection{Mask Size Optimization}\label{section_4_B_maskopt}

Considering \eqref{equation_theorem_non_iid_hierarchical} and \eqref{eq:overall_latency}, we see that the mask sizes $\left\{\|\bm p^t_1\|_1, \|\bm p^t_2\|_1, \ldots, \|\bm p^t_N\|_1\right\}$ affect both the latency and the learning convergence bound.
In particular, while an imbalanced mask size distribution may help speed up each global update in \eqref{eq:overall_latency} -- by assigning smaller partitions to cells with smaller communication/computation resources -- it results in a larger deviation from the stationary point for the \texttt{HIST} algorithm in Theorem \ref{theorem_non_iid_hierarchical}. Moreover, the uniform mask partition leads to the minimum convergence bound, but will only lead to the minimum latency if the resources are homogeneous. Choosing the mask partition thus induces a compromise between training efficiency and convergence performance.

In \eqref{equation_theorem_non_iid_hierarchical}, recall that the impact of the mask size is contained within the fifth term.
To suppress the impact of an imbalanced mask size distribution, 
we enforce this term to remain below to a predefined threshold $\epsilon_{th}$.
The latency minimization problem of \texttt{HIST} is thus formulated as
\begin{subequations} \label{latency_minimization}
\begin{align}
\min_{\left\{\|\bm p_j^t\|_1\right\}_{j=1}^N} & \max_{j\in\{1,\dots,N\}}\left \{E \max_{i \in \mathcal{C}_j} \left \{ H \frac{\|\bm p^{t}_j\|_1 V_0}{\mathcal{F}^t_i d} + n_j \frac{\|\bm p^{t}_j\|_1 L_0}{\mathcal{R}^t_i d} \right\} \right\} \label{objective_fun} \\
 \text{s.t.} \qquad & \sum_{j=1}^N \|\bm p^{t}_j\|_1  = d \label{cons_1}\\
 & 6 \frac{N\|\bm p^{t}_j\|_1}{d} \delta_1^2 \leq \epsilon_{\textnormal{th}}, \ j=1,2,\ldots,N. \label{cons_2}
\end{align}
\end{subequations}
To solve this problem, note that, \eqref{latency_minimization} can be rewritten as
\equa{\label{latency_minimization_linear}
\min_{\left\{\|\bm p_j^t\|_1\right\}_{j=1}^N} & t \\
 \text{s.t.} \quad & \max_{i \in \mathcal{C}_j} \left \{ H \frac{\|\bm p^{t}_j\|_1 V_0}{\mathcal{F}^t_i d} + n_j \frac{\|\bm p^{t}_j\|_1 L_0}{\mathcal{R}^t_i d} \right\} \leq t,~ \forall j \\
 \quad & \sum_{j=1}^N \|\bm p^{t}_j\|_1  = d \\
 & 6 \frac{N\|\bm p^{t}_j\|_1}{d} \delta_1^2 \leq \epsilon_{\textnormal{th}}, \ j=1,2,\ldots,N,
}
which is in the form of a mixed integer linear programming (MILP) problem, given the objective and constraints are each linear terms in the variables, i.e., the mask sizes of each cell. This is solvable using standard MILP solvers, e.g., Gurobi. 
At the commencement of each global training round, the cloud server will solve this problem, then randomly generate masks based on these optimized mask sizes, partition the current global model $\bar{\bm{x}}^t$  accordingly, and send the obtained submodel partitions to the corresponding cells. We will see in Section \ref{simu_HIST_opt} how this optimization leads to improvements in testing accuracy and training latency compared with baselines.

{\color{black}
\begin{remark}Our focus here is on HFL, where the system topology is specified based on geographical factors. The mask optimization in~\eqref{latency_minimization} is thus conducted based on a given topology. In clustered FL \cite{ghosh2020efficient}, by contrast, client groups are optimized based on some designated criteria (e.g., data similarity, computational capability). Our approach can be applied downstream from clustered FL, by utilizing the partitioning optimization method once the clustering is completed.
\end{remark}}

\vspace{-0.05in}
\section{AirComp-Assisted \texttt{HIST} Algorithm}\label{sec_aircomp}

In this section, we propose an AirComp-assisted version of  \texttt{HIST} {\color{black}under full client participation}, which takes advantage of the superposition property of multiple access channels at the edge layer, i.e., between clients and edge servers. We develop an optimization for AirComp as well as the partitioning.

Note that through our submodel partitioning strategy in  \texttt{HIST}, the per-iteration communication complexity at the cloud server remains constant regardless of the number of edge servers. Conversely, the communication complexity experienced by edge servers increases in proportion to their client counts. This is one of the motivations for incorporating AirComp into the \texttt{HIST} methodology, i.e., to facilitate scalable model aggregations at each edge server. {\color{black}
With AirComp-assisted HIST, the communication delay at each edge server is expected to be independent of the number of clients within its coverage, thereby accelerating the model aggregation in each cell.
}  
We assume that the downlink model dissemination in each cell is error-free, as edge servers possess sufficiently large transmit power compared to resource-constrained clients  \cite{liu2021reconfigurable, cao2022transmission,sery2021over}. Our primary focus will be on the uplink model uploading within each cell.
 Different from the existing works on AirComp-assisted FL, in \texttt{HIST}, each edge server is in charge of aggregating a different part of the model compared to other edge servers. Thus, there is no superposition effect in terms of model parameter errors across different edge servers induced by AirComp. Additionally, the communication load of model aggregations in each edge server is configurable.

\vspace{-0.05in}
\subsection{Signal Transmission Model}

In Step $7$ of Algorithm \ref{alg:HFIST}, each edge server aims to acquire an average of local models of clients within its coverage area, as represented by \eqref{fedavg}. We adopt AirComp to support this aggregation, which allows each edge server to directly estimate an average of signals transmitted from their clients, bypassing the decoding of individual signals. 
To mitigate the effect of wireless noise, instead of transmitting models, we let each client $i$ in cell $j$ upload its accumulated gradient $\bm \delta_i^{t,e}$ from $H$ steps of local updates in the $(t,e)$-th round of edge training. Specifically, $\bm \delta_i^{t,e}$ can be written as
\begin{align}
    \bm \delta_i^{t,e} = \frac{1}{\gamma} \left(\bm x_{i,H}^{t,e} - \bm x_{i,0}^{t,e} \right) = \sum_{h=0}^{H-1} {\bm p}^t_j \odot \nabla l({\bm x}^{t,e}_{i,h}, \xi^{t,e}_{i,h}).
\end{align} 
Under this scheme, the impact of wireless noise on the aggregated model will be mitigated by the ratio $\gamma$ compared to directly uploading the models \cite{zou2022knowledge,zhu2023airfl}, as we will see later.
Consider a single-input-multiple-output (SIMO) AirComp system, where clients within each cell are deployed with a single antenna,  and edge servers have $M$ antennas. 
In each timeslot of AirComp, where each client $i$ in cell $j$ concurrently uploads a particular element of $\bm \delta_i^{t,e}$, i.e., $(\delta_i^{t,e})_k$, for the $k$-th element. The goal of the edge server in this timeslot is to estimate the average $({\bar \delta}_j^{t,e})_k \coloneqq \frac{1}{n_j} \sum_{i  \in \mathcal{C}_j} (\delta_i^{t,e})_k$. The received signal becomes 
\equa{\label{aggregation}
	(\hat{\delta}_j^{t,e})_k = (\bm m_j^{t,e})^{\textrm{H}} \left( \sum_{i \in \mathcal{C}_j} \bm h_i^{t,e} \alpha_i^{t,e} ({\delta}_i^{t,e})_k + \bm z_{j,k}^{t,e} \right),
} 
where $\alpha_i^{t,e}$ denotes the precoding factor at client $i$, and $\bm h_i^{t,e} \in \mathbb{C}^{M}$ represents the SIMO channel between client $i$ and edge server $j$, which is assumed to be invariant during $(t,e)$-th communication round.  We assume that $\bm h_i^{t,e}$ is known to clients and edge servers  as in \cite{zou2022knowledge,liu2021reconfigurable}. 
$\bm m_j^{t,e}\in \mathbb{C}^{M}$ denotes the receive beamforming vector at edge server $j$, and $ \bm z_{j,k}^{t,e} \sim\mathcal{CN}(0, \sigma_0^2 \bm I_M) $ represents additive white Gaussian noise (AWGN). The precoding factor of clients in cell $j$ is subject to a maximum power constraint, i.e., $\|\alpha_i^{t,e}\| \leq P_j$.

\subsection{AirComp Aggregations}
It can be observed from \eqref{aggregation} that the distortion between the received signal $(\hat{\delta}_j^{t,e})_k$ and the target model average $\sum_{i \in \mathcal{C}_j} ({\delta}_i^{t,e})_k$ comes from the misalignment between channel conditions and noise across devices. To mitigate the effect of non-uniform channels, we set the precoding factor $\alpha_i^{t,e}$ for client $i$ as 
\begin{align}\label{alpha}
(\alpha_i^{t,e})^{\star}=\frac{1}{n_j}{{ \left((\bm m_j^{t,e})^{\textnormal{H}} \bm h_i^{t,e} \right)^{\dagger}}\over{\|(\bm m_j^{t,e})^{\textnormal{H}} \bm h_i^{t,e} \|^2}}, ~ \forall i\in \mathcal{C}_j,
\end{align}
where $()^{\dagger}$ represents a conjugate operation.
To meet the energy constraint $\|\alpha_i^{t,e}\|^2\leq P_j$, let $\bm m_j^{t,e} = \frac{1}{\nu_j^{t,e}}\bm a_j^{t,e}$ where $\nu_j^{t,e} = \sqrt{P_j} \min_{i\in\mathcal{C}_j} \|(\bm a_j^{t,e})^{\textnormal{H}} \bm h_i^{t,e}\|$ and $\|\bm a_j^{t,e}\|=1$, in which $\nu_j^{t,e}$ denotes the power normalization factor and $\bm a_j^{t,e}$ represents the normalized receive beamformer.

Under the precoding factor \eqref{alpha}, the signal shown in \eqref{aggregation} is simplified to $(\hat{\delta}_j^{t,e})_k = \frac{1}{n_j}\sum_{i \in \mathcal{C}_j} ({\delta}_i^{t,e})_k + \frac{1}{\nu_j^{t,e}}(\bm a_j^{t,e})^{\textnormal{H}} \bm z_{j,k}^{t,e}$,
which is an unbiased estimator of $({\bar \delta}_j^{t,e})_k$. The distortion of AirComp measured by mean-squared error (MSE), also known as the variance of $(\hat{\delta}_j^{t,e})_k$, is thus given by
\begin{align}\label{simplified_MSE}
    \textnormal{MSE}_{j,k}^{t,e} \!=\! \frac{\sigma_0^2 }{P_j \min_{i\in\mathcal{C}_j} \|(\bm a_j^{t,e})^{\textnormal{H}} \bm h_i^{t,e}\|^2 }.
\end{align}
This transmission repeats for all elements $k$ in the mask for cell $j$, i.e., $k \in \mathcal{S}_j^t = \{k | (p_j^t)_k \neq 0\}$. Consequently, the aggregated gradient at $(t,e)$-th round in cell $j$ admits the following expression:
\begin{align}
    \hat{\bm \delta}_j^{t,e}= \frac{1}{n_j} \sum_{i\in \mathcal{C}_j} { \bm \delta}_i^{t,e} + \bm Z_j^{t,e},
\end{align}
where $\bm Z_{j}^{t,e}$ denotes a noise vector with its $k$-th element being $1/\nu_j^{t,e}(\bm a_j^{t,e})^{\textnormal{H}} \bm z_{j,k}^{t,e}$ if $(p_j^t)_k \neq 0$, and $0$, otherwise.
Each edge server $j$ completes this process in parallel, and updates its edge model as  
\begin{equation}\label{fedavg_AirComp}
 \bar{\bm x}^{t,e+1}_j = \bar{\bm x}^{t,e}_j - \gamma \hat{\bm \delta}_j^{t,e}, \forall j, 
 \end{equation} 
 which is an unbiased estimator of the average of local models in $\mathcal{C}_j$, exhibiting a variance of
 \begin{equation}
     \gamma^2 \mathbb{E}\|\bm Z_j^{t,e}\|^2 = \gamma^2 \sum_{k\in \mathcal{S}^{t}_j}\textnormal{MSE}_{j,k}^{t,e}.
 \end{equation}
 With the combination of \texttt{HIST} and AirComp, the communication complexity at each edge server is independent of the number of clients within the coverage.

\begin{remark}
    By aggregating the accumulated gradient and then updating the edge model via an SGD step as shown in \eqref{fedavg_AirComp}, the impact of channel noise on the convergence of the AirComp-Assisted \texttt{HIST} algorithm can be analogized to the noise induced by the inherent randomness of stochastic gradients. Such a strategy facilitates the convergence analysis in Section \ref{sec_aircomp_analysis_} and allows us to select an appropriate step size $\gamma$ to mitigate the effect of channel noise. 
\end{remark}

\subsection{Convergence Analysis of AirComp-assisted \texttt{HIST}}
\label{sec_aircomp_analysis_}
In AirComp-assisted \texttt{HIST}, step 7 of Algorithm \ref{alg:HFIST} is replaced by \eqref{fedavg_AirComp}. Next, we analyze the convergence behavior of the AirComp-assisted \texttt{HIST} algorithm. 
The result is provided in Theorem \ref{theorem_non_iid_hierarchical_AirComp} below. 

\begin{theorem}\label{theorem_non_iid_hierarchical_AirComp}
	Suppose that Assumptions \ref{assump_lowerbound}-\ref{assump_gradient_dissimilarity_client} hold, $N\geq 2$, and the step size satisfies condition \eqref{condition_gamma}. Then the AirComp-assisted \texttt{HIST} algorithm satisfies 
\begin{align}\label{}
	& \frac{1}{TE} \sum_{t=0}^T \sum_{e=0}^{E-1} \mathbb{E} \left\| \nabla f(\hat{\bm x}^{t,e} ) \right\|^2
 \leq 4 \frac{f(\bar{\bm x}^0) - f_*}{\gamma TEH}  + 100 \gamma \tilde{N} L \sigma^2 \nonumber \\
 & + 1356 \gamma L \delta_1^2 + 60 \gamma L \delta_2^2 +  48 (N-1) L^2 \frac{1}{T}\sum_{t=0}^{T-1} \mathbb{E} \left \|\bar{\bm x}^{t}  \right\|^2   \nonumber \\
& + \gamma \frac{1}{TE} \sum_{t=0}^T \sum_{e=0}^{E-1} \sum_{j=1}^N \textnormal{MSE}_j^{t,e} +  6 \frac{N}{d} \frac{1}{T}\sum_{t=0}^{T-1} d^t_{\textrm{max}} \delta_1^2,
\end{align}
	where $\textnormal{MSE}_{j}^{t,e} = \sum_{k\in \mathcal{S}^{t}_j}\textnormal{MSE}_{j,k}^{t,e}$, $\mathcal{S}^{t}_j = \{k| p_j^t \neq 0\}$, and
 $\tilde{N}$ and $d^t_{\textrm{max}}$ are defined in Theorem \ref{theorem_non_iid_hierarchical}.
\end{theorem}
\begin{proof}
    Please refer to Appendix \ref{proof_theorems_aircomp}.
\end{proof}
Compared with Theorem \ref{theorem_non_iid_hierarchical}, the difference introduced by AirComp is quantified by the sixth term, i.e., the MSE from the aggregation in each cell. Plugging an appropriate step size into Theorem \ref{theorem_non_iid_hierarchical_AirComp} gives rise to the following corollary.

\begin{corollary}\label{corollary_non_iid_hierarchical_AirComp}
	Suppose that Assumptions \ref{assump_lowerbound}-\ref{assump_gradient_dissimilarity_client} hold, and let the step size be $\gamma = (T\!E\!H)^{-\!\frac{1}{2}}$ for $T$ is large enough to satisfy \eqref{condition_gamma}. Then the \texttt{HIST} algorithm satisfies
	\begin{align}
		& \frac{1}{TE} \sum_{t=0}^T \sum_{e=0}^{E-1} \mathbb{E} \left\| \nabla f(\hat{\bm x}^{t,e} ) \right\|^2
		\leq \mathcal{O}\left( \tilde{N}^{\frac{1}{2}} (TEH)^{-\frac{1}{2}} \right) \nonumber \\
  & + \mathcal{O}\left((TEH)^{-\frac{1}{2}} \right) + \mathcal{O}\left((N-1)L^2 
  \frac{1}{T} \sum_{t=0}^{T-1} \mathbb{E} \left\|\bar{\boldsymbol{x}}^{t}\right\|^2 \right)\nonumber \\
  & + \gamma \underbrace{\frac{1}{TE} \sum_{t=0}^T \sum_{e=0}^{E-1} \sum_{j=1}^N \textnormal{MSE}_j^{t,e} +  6 \frac{N}{d} \frac{1}{T}\sum_{t=0}^{T-1} d^t_{\textrm{max}} \delta_1^2}_{\textnormal{controllable terms}}.
	\end{align}
\end{corollary}
As shown in Theorem \ref{theorem_non_iid_hierarchical_AirComp}, the error term induced by AirComp, i.e., the sixth term, is of the same order as the variance of the stochastic gradient, i.e., the second term (growing proportionally to $N$ or $\tilde{N}$).  This similarity is attributed to the fact that, under the precoding design as specified in \eqref{alpha}, the accumulated gradient estimated by AirComp remains unbiased. Consequently, the variance affects the algorithm in a manner akin to the inherent randomness of the stochastic gradient, and just adds more uncertainty. 

In Theorem \ref{theorem_non_iid_hierarchical_AirComp} and Corollary \ref{corollary_non_iid_hierarchical_AirComp}, we see there are two error terms that are controllable based on AirComp-assisted HIST variables. We address the minimization of these terms through beamforming design (Section V-D) and submodel partitioning (Section V-E) next.

\vspace{-1mm}

\subsection{Receive Beamforming Design}


We aim to design the normalized receive beamformers to mitigate the impact of AirComp distortion, by minimizing the MSE term in Theorem \ref{theorem_non_iid_hierarchical_AirComp}.
Note that the expression of MSE in \eqref{simplified_MSE} is invariant to $k$, i.e., $\textnormal{MSE}_{j,k}^{t,e} = \textnormal{MSE}_{j,k^{\prime}}^{t,e}$, $\forall k, k^{\prime} \in \mathcal{S}^{t}_j$. Therefore, the beamformer design problem at the $(t,e)$-th round can be formulated as 
\vspace{-0.5mm}
\equa{\label{ori_pro_MSE}
\min_{{\bm a}_1^{t,e},{\bm a}_2^{t,e},\ldots,{\bm a}_N^{t,e}} &\sum_{i=1}^{N} \frac{\|\bm p_j^t\|_1 \sigma_0^2 }{P_j \min_{i\in\mathcal{C}_j} \|(\bm a_j^{t,e})^{\textnormal{H}} \bm h_i^{t,e}\|^2 } \\
\text{s.t.}\qquad &  \|\bm a_j^{t,e}\|^2 = 1, ~ j=1,2,\ldots,N.
}

Due to the independence between ${\bm a}_j^{t,e}$ and ${\bm a}_{j^{\prime}}^{t,e}$, $\forall j \neq j^{\prime}$, problem \eqref{ori_pro_MSE} can be transformed into $N$ independent subproblems, one for each cell $j$, as
\vspace{-0.5mm}
\equa{\label{re1_pro_MSE}
\max_{{\bm a}_j^{t,e}} & \min_{i\in\mathcal{C}_j} \|(\bm a_j^{t,e})^{\textnormal{H}} \bm h_i^{t,e}\|^2 \\
\text{s.t.} ~&  \|\bm a_j^{t,e}\|^2 = 1, ~ j=1,2,\ldots,N.
}
Since the objective function is increasing in $\|\bm a_j^{t,e}\|^2$, Problem \eqref{re1_pro_MSE} can be directly relaxed to 
\equa{\label{1_pro_MSE}
\max_{{\bm a}_j^{t,e}} & \min_{i\in\mathcal{C}_j} \|(\bm a_j^{t,e})^{\textnormal{H}} \bm h_i^{t,e}\|^2 \\
\text{s.t.} ~&  \|\bm a_j^{t,e}\|^2 \leq 1, ~ j=1,2,\ldots,N,
}
without compromising optimality. 
 Problem \eqref{1_pro_MSE} is a non-convex quadratically constrained quadratic programming (QCQP) problem. First-order methods have been developed to solve problems in this class efficiently, e.g., the Mirror-Prox-based Successive Convex Approximation (SCA) algorithm detailed in \cite{fang2021over}. 

\vspace{-3mm}
\subsection{Submodel Partitioning Optimization}

The submodel partitioning term in Theorem \ref{theorem_non_iid_hierarchical_AirComp} is the same as in Theorem \ref{theorem_non_iid_hierarchical}. Different from Sec. IV that studies the OMA-based \texttt{HIST}, we now need a latency model that incorporates AirComp. 
{\color{black}
When using AirComp, multiple devices transmit their submodel parameters concurrently in the same time slot and frequency band, as in the signal transmission model \eqref{aggregation}. Within this transmission mechanism, each parameter is amplitude-modulated
to a single analog symbol and each sub-channel is dedicated
to a single parameter transmission. Thus, uploading a model
update of dimension $\|\bm p^{t}_j\|_1$, the total number of analog symbols to be transmitted is $\|\bm p^{t}_j\|_1$. Additionally, since AirComp employs analog aggregations, while the channel conditions and SNR will impact the aggregation error, they will in theory not impact the communication latency.  Concretely, according to \cite{zhu2019broadband,cao2022transmission}, the uplink communication latency at cell $j$ can be written as
\setcounter{equation}{36}
\begin{align}
    \frac{\|\bm p^{t}_j\|_1}{B/\Delta f  } t_s ,
\end{align}
where $B$ denotes the system bandwidth, $\Delta f$ denotes the bandwidth of a sub-channel, and $t_{s}$ denotes a symbol duration. For example, in LTE systems, each resource block with the duration of $t_s = \frac{1}{14}$ ms and sub-channel bandwidth $\Delta f = 15$ kHz \cite{cao2022transmission}.
}
Thus the total latency for each edge training round within cell $j$ can be represented as
\equa{
\max_{i \in \mathcal{C}_j}\left\{H \frac{\|\bm p^{t}_j\|_1 V_0}{\mathcal{F}^t_i d} +  \frac{\|\bm p^{t}_j\|_1}{B/\Delta f  } t_s \right\},
} 
using the same computational latency model as in \eqref{latency_comm_comp}.
As a result, the latency minimization problem of this AirComp-assisted HFL system can be formulated as
\equa{\label{latency_minimization_AirComp}
\min_{\left\{\|\bm p_j^t\|_1\right\}_{j=1}^N} & \max_{j\in\{1,\dots,N\}} \left \{E \max_{i \in \mathcal{C}_j}\left\{H \frac{\|\bm p^{t}_j\|_1 V_0}{\mathcal{F}^t_i d} +  \frac{\|\bm p^{t}_j\|_1}{B/\Delta f  } t_s \right\} \right\} \\
 \text{s.t.} \quad & \eqref{cons_1}, \eqref{cons_2},
}
which can be resolved using the same method for \eqref{latency_minimization}. 

\begin{figure}[t]
	\centering	\includegraphics[scale=0.53]{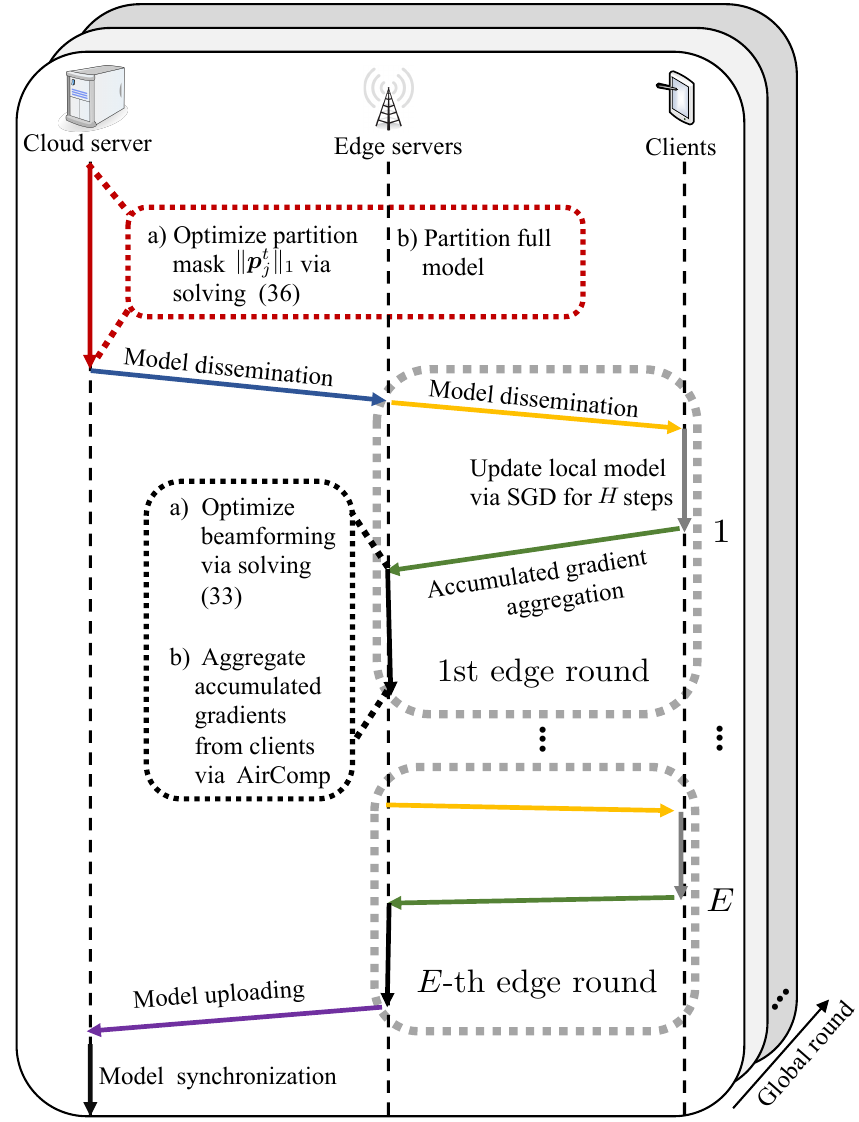}
	\caption{Visualization of the AirComp-assisted \texttt{HIST} algorithm. The mask partitioning is solved once per global round, while the beamforming optimization is solved once per edge round.} 
\label{fig:procedures}
\vspace{-0.05in}
\end{figure}

The full implementation procedure of the AirComp-assisted \texttt{HIST} algorithm is illustrated in Fig. \ref{fig:procedures}. 

 \begin{figure*}[t!]
	\centering
	\begin{minipage}{.24\textwidth}
		\begin{subfigure}{\textwidth}
			\centering
			\includegraphics[width=2.0in]{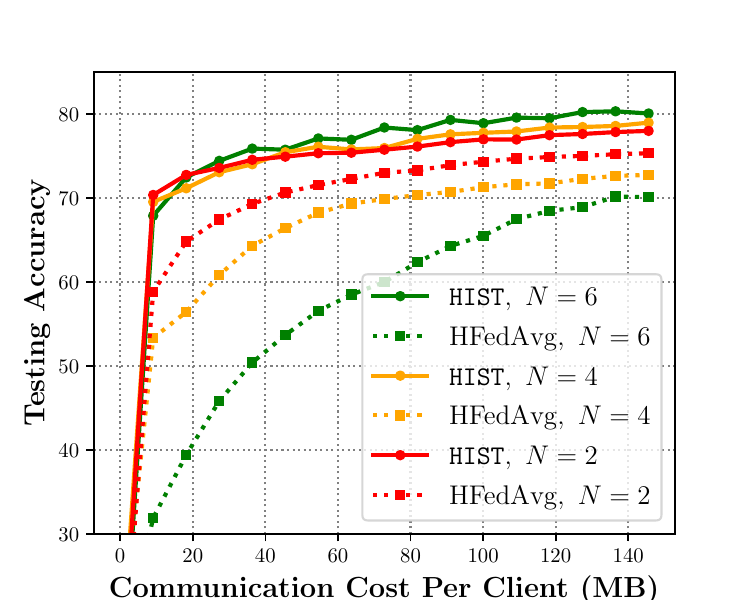}
			\caption{Fully non-i.i.d.}\label{non_iid_load}
		\end{subfigure}\\
		\begin{subfigure}{\textwidth}
			\centering
			\includegraphics[width=2.0in]{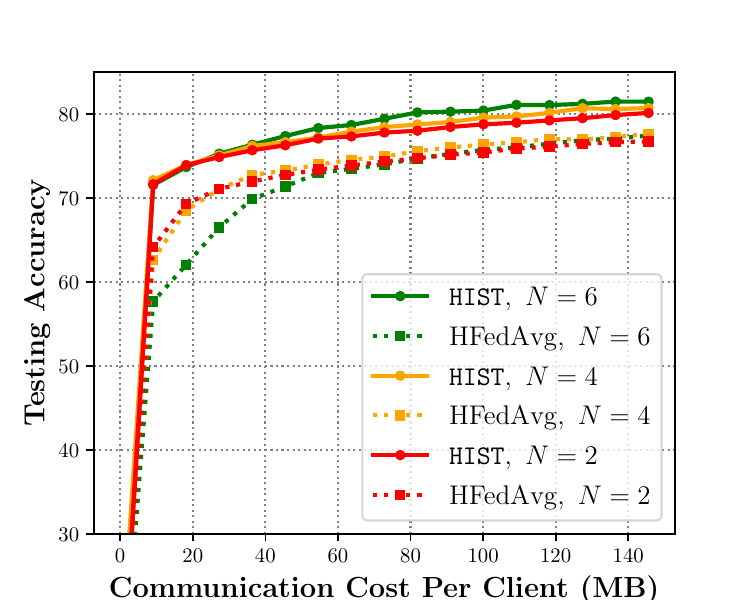}
			\caption{Cell i.i.d., client non-i.i.d.}\label{mixed_load}
		\end{subfigure}%
		 \caption{The impact of the number of cells $N$ on the convergence performance of \texttt{HIST}.} \label{load}
	\end{minipage}
	\hfill
	\begin{minipage}{.24\textwidth}
		\begin{subfigure}{\textwidth}
			\centering
			\includegraphics[width=2.0in]{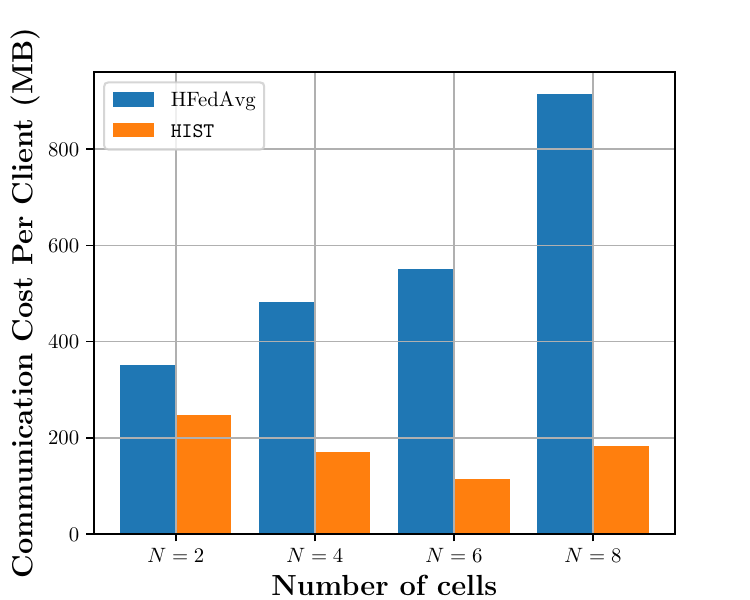}
			\caption{Fully non-i.i.d.}\label{desired_acc_noniid}
		\end{subfigure}\\
		\begin{subfigure}{\textwidth}
			\centering
			\includegraphics[width=2.0in]{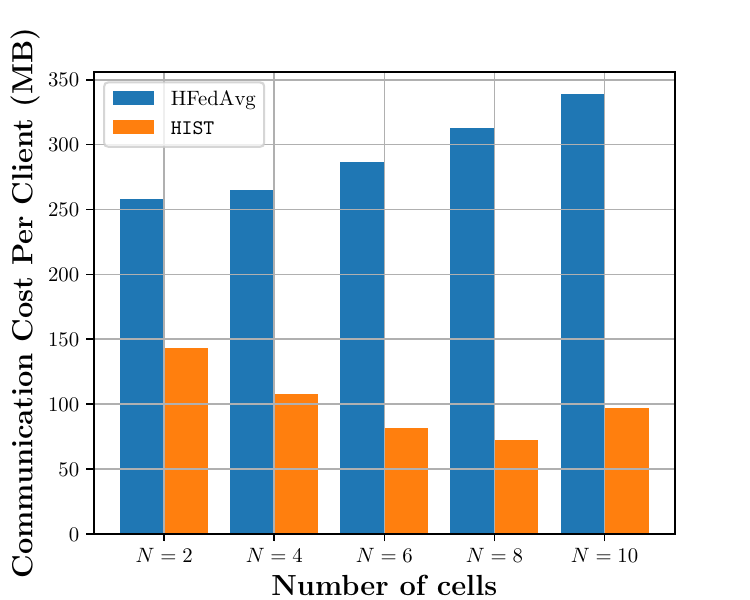}
			\caption{Cell i.i.d., client non-i.i.d.}\label{desired_acc_mixed}
		\end{subfigure}%
		\caption{{\color{black}Communication cost for achieving the testing accuracy of $80\%$ in each scheme.}}\label{desired_accuracy}
	\end{minipage}
	\hfill
	\begin{minipage}{.24\textwidth}
		\begin{subfigure}{\textwidth}
			\centering
			\includegraphics[width=2.0in]{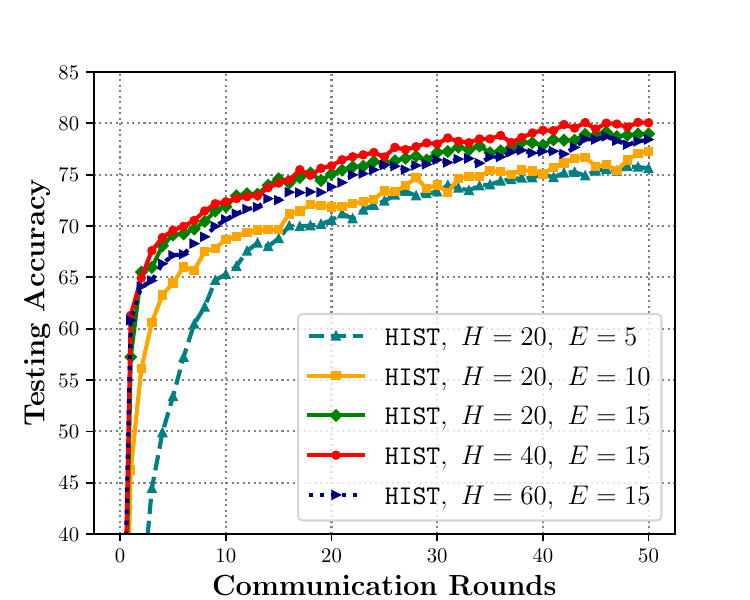}
    \caption{Fully non-i.i.d.}\label{non_iid_impact_H}
		\end{subfigure}\\
		\begin{subfigure}{\textwidth}
			\centering
			\includegraphics[width=2.0in]{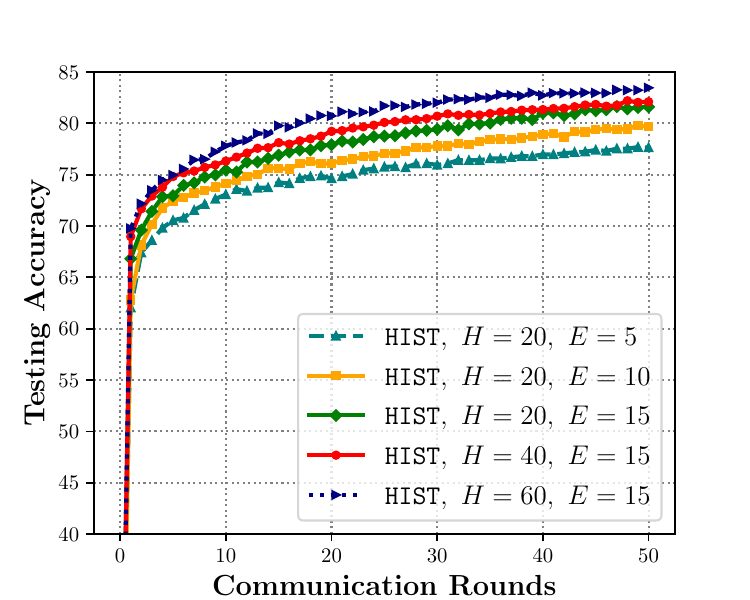}
    \caption{Cell i.i.d., client non-i.i.d.}\label{mixed_impact_HE}
		\end{subfigure}%
		\caption{Impacts of aggregation periods $H$ and $E$ on the performance of \texttt{HIST}.}\label{local_updates}
	\end{minipage} 
 \hfill
	\begin{minipage}{.24\textwidth}
		\begin{subfigure}{\textwidth}
			\centering
			\includegraphics[width=2.0in]{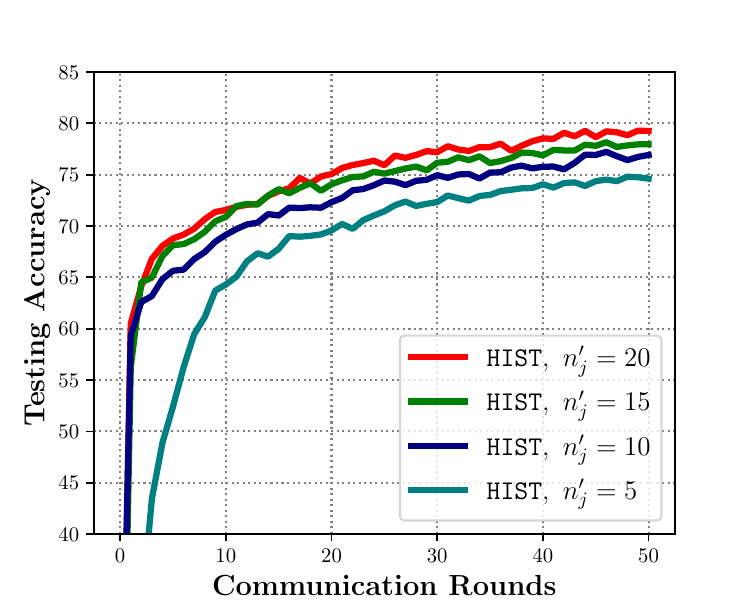}
    \caption{Fully non-i.i.d.}\label{impact_partial_participation_noniid}
		\end{subfigure}\\
		\begin{subfigure}{\textwidth}
			\centering
			\includegraphics[width=2.0in]{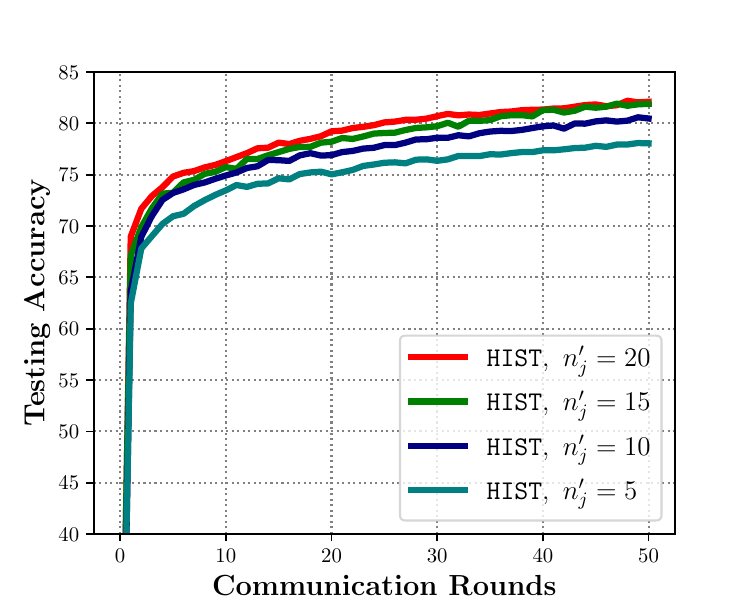}
    \caption{Cell i.i.d., client non-i.i.d.}\label{impact_partial_participation_mix}
		\end{subfigure}%
		\caption{{\color{black}Impact of the number of participating clients $ n_j^{\prime}$ on the performance of \texttt{HIST}.}}\label{partial_participation}
	\end{minipage}
    \vspace{-0.1in}
\end{figure*}

\section{Simulation Results}\label{sec_simulation}

\subsection{Simulation Settings}\label{sim_setting}

We consider a setup in which clients are evenly distributed across $N$ cells, i.e., $n_j=n_{j^\prime}, \forall j,j^\prime \in \{1,2,\ldots,N\}$, for various choices of $N$. 
Unless otherwise specified, the total number of clients in the system, i.e., $N n_j$, is set to $60$. 
We focus on two practical data distribution settings: (i) fully non-i.i.d. and (ii) i.i.d. data across cells but non-i.i.d data across the clients within the same cell.  
For case i), the client's dataset construction follows the approach outlined in \cite{fang2022communication}. 
For case (ii), we first uniformly and randomly divide the entire training set into $N$ partitions, corresponding to $N$ cells. We then distribute each partition to the clients within the respective cell in a non-i.i.d. manner following case (i). Case (ii) allows us to consider settings where clients that are closer in geographical proximity (i.e., within the same cell) have similar local datasets. 

{\color{black} We consider training fully connected neural networks and convolutional neural networks using Fashion-MNIST and CIFAR-10 datasets, with the details given in the following subsections. In Appendix \ref{appendix:transformer}, we show how our methodology can be applied to fine-tune transformer models with low rank adaptation (LoRA) as well.
Unless stated otherwise, the experiments are conducted under full client participation.}

\vspace{-0.1in}
\subsection{Performance Evaluation of \texttt{HIST}}\label{simu_pure_HIST}

\subsubsection{Fully connected neural networks}\label{simu_pure_HIST:fcn}

We first consider an image classification task on Fashion-MNIST using a two-layer fully connected neural network. In this model, we configure the input layer to have $784$ neurons, corresponding to the size of the input image, and the output layer to have $10$ neurons, which matches the number of classes.
Additionally, we employ a hidden layer with $300$ neurons. 
The model size is $0.91$ MB.

The cloud server disjointly partitions the hidden neurons to construct different submodels.
In this subsection, we first consider the case of the uniform partition, i.e., all submodels have the same sizes, before evaluating our partitioning optimization strategy in Section \ref{simu_HIST_opt}. This can be achieved by uniformly and randomly partitioning the hidden neurons. 
More details on our partitioning strategy and its difference from \cite{yuan2022distributed,zhou2022convergence} are discussed in Appendix \ref{appendix:model_partition}.

 \textbf{Comparison with Baseline:} In Fig. \ref{load}, we compare our proposed \texttt{HIST} algorithm with the traditional hierarchical FedAvg algorithm (denoted as HFedAvg in our figures), where the full model is communicated over the hierarchical network. We evaluate the performance by comparing testing accuracy across different numbers of cells, $N\in \{2, 4, 6\}$.
The X-axis here represents the communication load, quantified as the volume of parameters transmitted per client.  
We set  $H$ and $E$ to $20$ and $5$, respectively.

As shown in Figs. \ref{non_iid_load} and \ref{mixed_load}, the proposed \texttt{HIST}   outperforms HFedAvg in terms of testing accuracy achieved for the same level of communication cost for both data distribution settings. In particular, under a fully non-i.i.d. data setup, as $N$ increases, the extent of data dissimilarity across cells becomes more pronounced (since each cell has fewer class labels as $N$ increases), leading to performance degradation for HFedAvg. In contrast, \texttt{HIST} achieves a higher testing accuracy for a given communication load when $N$ increases from $2$ to $6$.  
This is due to the fact that, for \texttt{HIST}, the per-round communication cost per client decreases as the number of cells increases, partly compensating for the degradation of convergence induced by data heterogeneity.
 
Fig. \ref{desired_accuracy} compares the communication cost of \texttt{HIST} and HFedAvg for achieving a target accuracy of $80 \%$. The Y-axis measures the volume of parameters transmitted by each client during the training process, i.e., the X-axis of Fig. \ref{load}. It is observed that  \texttt{HIST} takes less communication to achieve the preset accuracy, which demonstrates the efficiency of the proposed algorithm over 
 HFedAvg. In addition, as the number of cells increases from $2$ to $6$, the communication cost shows a decreasing trend, which is a stark contrast to HFedAvg in the fully non-i.i.d. case. This further demonstrates the advantage of submodel partitioning in \texttt{HIST}, even without optimization applied.
 {\color{black}
However, once the number of cells is increased to $8$ (or $10$ in the cell i.i.d., client non-i.i.d. case), \texttt{HIST} experiences performance degradation, demonstrating the efficiency-accuracy trade-off associated with the value of $N$. This performance drop is primarily due to the submodels becoming too small to leverage the local datasets on each device, necessitating more communication rounds to reach the same accuracy.\footnote{If the number of cells gets too large in specific scenarios, the \texttt{HIST} algorithm can be adapted by modifying the partition strategy (e.g., by including partial overlaps of model parameters among the partitions) to prevent such performance degradation.} This impact of $N$ is consistent with our theoretical result in \eqref{equa_corollary_1}: the last term increases proportionally to a larger number of cells, which eventually outpaces the reduction in per-round communication cost from smaller submodels.
}


\textbf{Effects of System Parameters:} The impacts of   period of   edge aggregation    $H$ and period of   global synchronization  $E$ on the convergence behavior are demonstrated in Fig. \ref{local_updates}. The X-axis represents the number of global synchronizations at the cloud server.
We consider a scenario with $N=3$. As $E$ increases within $E \in \{5, 10, 15\}$, \texttt{HIST} attains a better convergence performance for both Figs. \ref{non_iid_impact_H} and \ref{mixed_impact_HE}. This is because a large  $E$ gives rise to more rounds of edge training within each global round.
{\color{black}When $H$ increases within $H \in \{20, 40, 60\}$, Fig. \ref{non_iid_impact_H} shows that the convergence speed of \texttt{HIST} first increases, and then experiences a degradation. This trend occurs due to data heterogeneity, since edge devices become more prone to overfitting on their local datasets as the aggregations become less frequent. This is also reflected in our theoretical analysis: we showed in Section \ref{sec_theory_C} that $H$ has an upper bound to ensure the step size condition for the convergence result. On the other hand,
Fig. \ref{mixed_impact_HE} shows that \texttt{HIST} continues to improve in performance as $H$ increases from $40$ to $60$: this setting exhibits lower data heterogeneity, which allows for a local aggregation period before overfitting.}

{\color{black}
\textbf{Effects of Client Participation Ratio:} 
In Fig. \ref{partial_participation}, we investigate the performance of the \texttt{HIST} algorithm under partial client participation. Specifically, in this experiment, the number of cells is set to $N=3$, and the aggregation periods are configured as $H=20$ and $E=15$. The number of clients in each cell is fixed at $n_j = 20$, but only a subset participates in each round. The results in Figs. \ref{impact_partial_participation_noniid} and \ref{impact_partial_participation_mix} demonstrate the impact of varying the number of participating clients in each cell, denoted as $n_j^{\prime}$, taking values $\{5, 10, 15, 20\}$. 
We see that, as expected, increasing the number of participating clients results in improved training performance for \texttt{HIST}. A higher participation ratio leads to less error in the aggregated model updates during each training round, especially when clients have diverse data distributions, which also explains why the impact of $n_j^{\prime}$ is larger in Fig. \ref{impact_partial_participation_noniid}. The observed speedups in performance align well with the theoretical observations discussed in Section \ref{sec_partial_participation}.}

\begin{table*}[t!]
 \centering
\begin{tabular}{lclclclclclclclclcl}
\toprule
\multicolumn{17}{c}{LeNet-5 on Fashion-MNIST} \\
\midrule
& & \multicolumn{7}{c}{Fully non-i.i.d.} & & \multicolumn{7}{c}{Cell i.i.d., client non-i.i.d.} \\ \cmidrule{3-9} \cmidrule{11-17}
& & \multicolumn{3}{c}{$2$ Cells} & & \multicolumn{3}{c}{$4$ Cells} & & \multicolumn{3}{c}{$2$ Cells}  & &  \multicolumn{3}{c}{$4$ Cells} \\
\cmidrule{3-5} \cmidrule{7-9} \cmidrule{11-13} \cmidrule{15-17}
Accuracy & & \texttt{HIST} & &   \texttt{HFedAvg} & & \texttt{HIST} & & \texttt{HFedAvg} & & \texttt{HIST} & & \texttt{HFedAvg} & & \texttt{HIST} & & \texttt{HFedAvg} \\
\cmidrule{1-1} \cmidrule{3-3} \cmidrule{5-5} \cmidrule{7-7} \cmidrule{9-9} \cmidrule{11-11} \cmidrule{13-13} \cmidrule{15-15} \cmidrule{17-17} 
70\% && 13.20 && 23.66 && 8.08 && 20.28 && 4.40 && 6.76 && 3.32 && 11.84 \\
80\% && 115.30 && 206.18 && 76.06 && 246.74 && 38.72 && 86.20 && 24.72 && 89.58 \\
\bottomrule
\end{tabular}
\caption{Communication cost (MB) per client for achieving testing accuracy of $70\%$ and $80\%$ on Fashion-MNIST. The results show that \texttt{HIST} can provide significant communication savings compared with the HFedAvg baseline during training.}
\label{tab:converge:fmnist}
\end{table*}
\begin{table*}[t!]
\centering
\begin{tabular}{lclclclclclclclclcl}
\toprule
\multicolumn{17}{c}{{\color{black}ResNet-18 on CIFAR-10.}} \\
\midrule
& & \multicolumn{7}{c}{Fully non-i.i.d.} & & \multicolumn{7}{c}{Cell i.i.d., client non-i.i.d.} \\ \cmidrule{3-9} \cmidrule{11-17}
& & \multicolumn{3}{c}{$2$ Cells} & & \multicolumn{3}{c}{$4$ Cells} & & \multicolumn{3}{c}{$2$ Cells}  & &  \multicolumn{3}{c}{$4$ Cells} \\
\cmidrule{3-5} \cmidrule{7-9} \cmidrule{11-13} \cmidrule{15-17}
Accuracy & & \texttt{HIST} & &   \texttt{HFedAvg} & & \texttt{HIST} & & \texttt{HFedAvg} & & \texttt{HIST} & & \texttt{HFedAvg} & & \texttt{HIST} & & \texttt{HFedAvg} \\
\cmidrule{1-1} \cmidrule{3-3} \cmidrule{5-5} \cmidrule{7-7} \cmidrule{9-9} \cmidrule{11-11} \cmidrule{13-13} \cmidrule{15-15} \cmidrule{17-17} 
75\% && 7.41 && 11.34 && 5.02 && 13.08 && 5.67 && 9.59 && 3.93 && 10.46 \\
80\% && 37.06 && 61.04 && 23.11 && 65.40 && 24.42 && 40.98 && 14.60 && 43.60 \\
\bottomrule
\end{tabular}
\caption{{\color{black}Communication cost (GB) per client for achieving testing accuracy of $75\%$ and $80\%$ on CIFAR-10. The results are consistent with the Fashion-MNIST results, further underscoring the effectiveness of the proposed HIST algorithm.}
}
\label{tab:converge:cifar}
\vspace{-0.2in}
\end{table*}

\begin{figure*}[t!]
	\centering
	\begin{minipage}{.3\textwidth}
		\begin{subfigure}{\textwidth}
			\centering
			\includegraphics[width=2.4in]{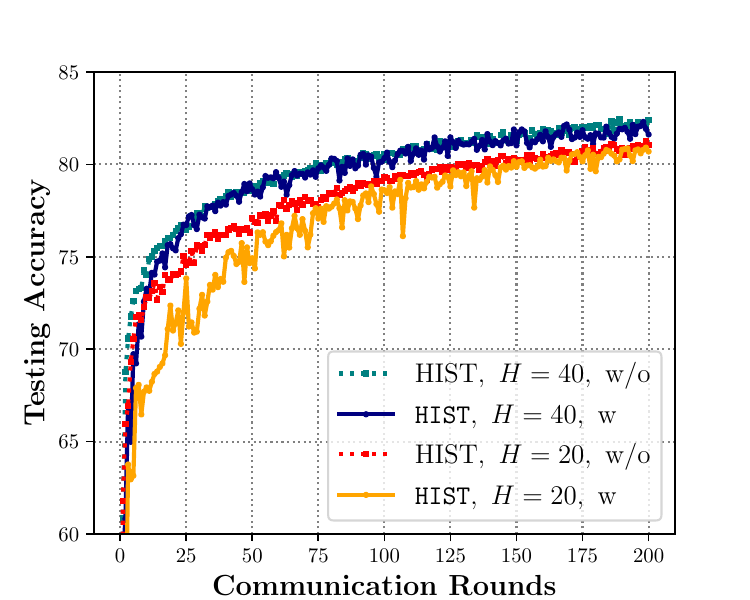}
      \caption{The impact of mask optimization on the convergence speed.}
      \label{mask_opt_rounds}
		\end{subfigure}
  \end{minipage}
  \hspace*{\fill} 
  \begin{minipage}{.3\textwidth}
  \begin{subfigure}{\textwidth}
			\centering
			\includegraphics[width=2.4in]{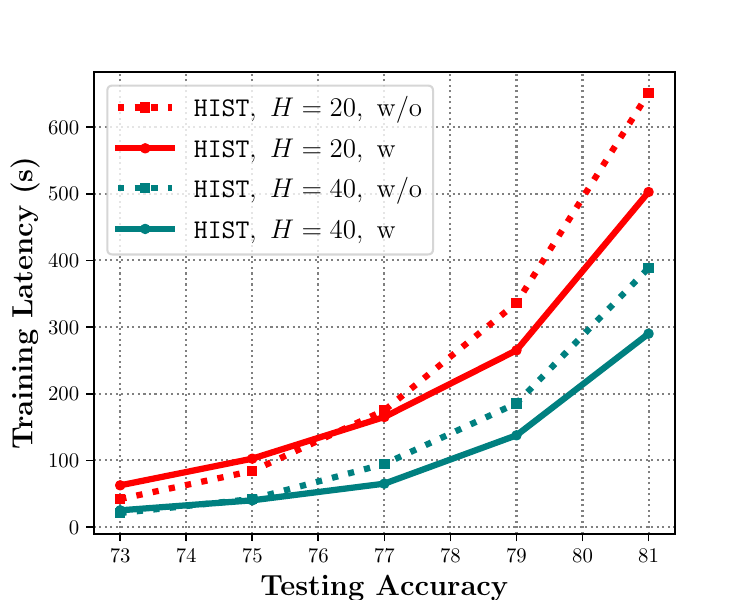}
   \caption{The impact of mask optimization on the training latency.}
   \label{mask_opt_latency_b}
		\end{subfigure} 
  \end{minipage}
  \hspace*{\fill} 
  \begin{minipage}{.3\textwidth}
  \begin{subfigure}{\textwidth}
			\centering
			\includegraphics[width=2.4in]{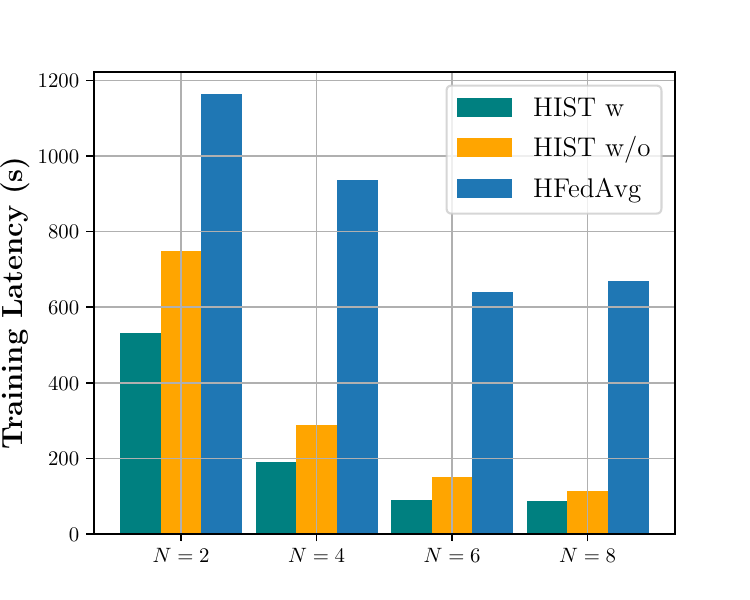}
   \caption{Training latency comparison for attaining accuracy of $80 \%$ under different $N$.}
   \label{mask_opt_latency_c}
		\end{subfigure} 
  \end{minipage}
   \caption{The impact of mask optimization in OMA-based \texttt{HIST}. We compare the cases  with mask optimization (w) and without mask optimization (w/o) under different settings. The results show that the proposed \texttt{HIST} with mask optimization can achieve the target accuracy much faster with reduced training time.}
   \label{mask_opt_OMA}
   \vspace{-0.15in}
\end{figure*}		

\begin{figure*}
    \begin{minipage}{.3\textwidth}
		\begin{subfigure}{\textwidth}
			\centering
			\includegraphics[width=2.4in]{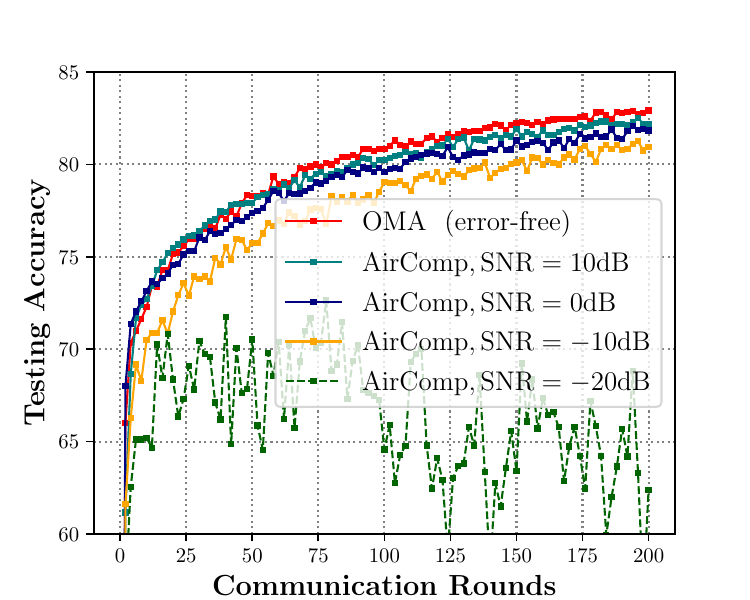}
   \caption{The impact of SNR on the convergence speed over training rounds.}
   \label{AirComp_FMNIST}
   \label{HIST_AirComp_SNR_fmnist}
		\end{subfigure}
    \end{minipage}
    \hspace*{\fill} 
    \begin{minipage}{.3\textwidth}
            \begin{subfigure}{\textwidth}
			\centering
			\includegraphics[width=2.4in]{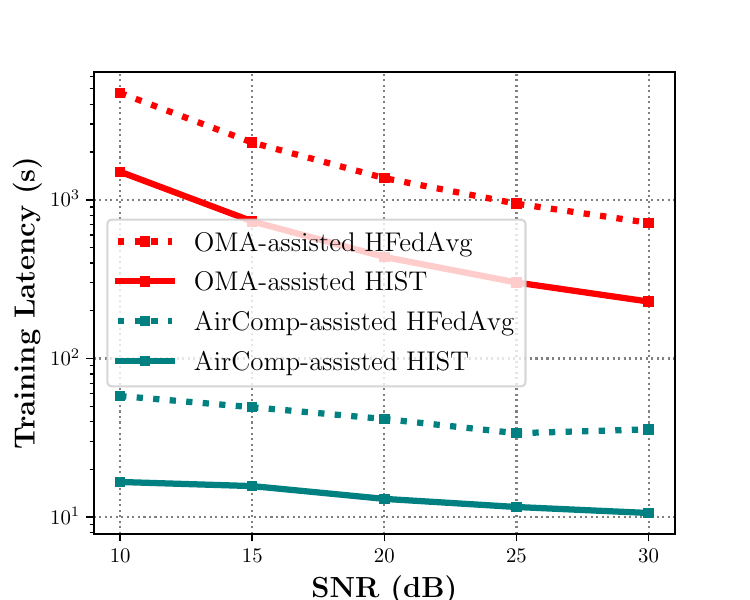}
   \caption{The impact of SNR on the training latency to achieve $80\%$ testing accuracy.}
   \label{HIST_AirComp_SNR_latency}
		\end{subfigure}
    \end{minipage}
    \hspace*{\fill} 
    \begin{minipage}{.3\textwidth}
        \begin{subfigure}{\textwidth}
			\centering
			\includegraphics[width=2.4in]{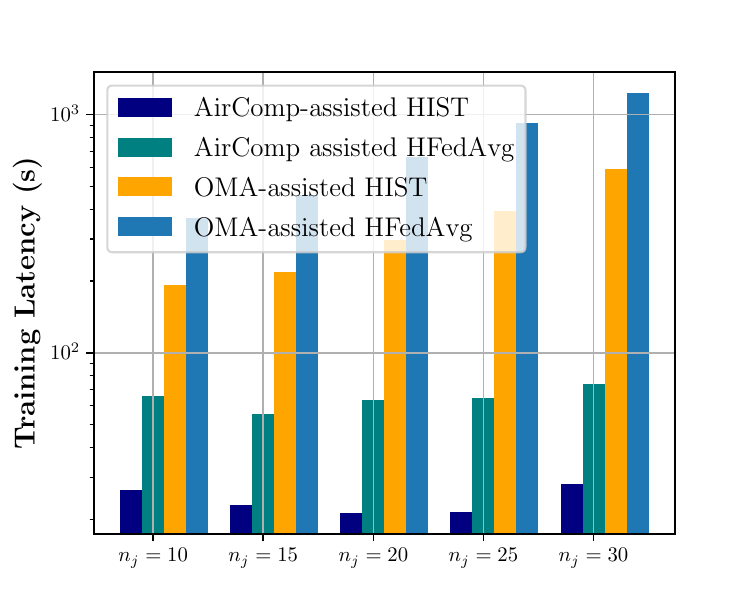}
           \caption{Training latency comparison for attaining accuracy of $80 \%$ under different $n_j$.}
           \label{HIST_AirComp_K}
        \end{subfigure} 
    \end{minipage}
    \caption{The performance of AirComp-assisted \texttt{HIST}. The proposed mask optimization strategy combined with AirComp-assisted \texttt{HIST} provides significant performance advantages compared with baselines under different settings.}
    \label{fig:HIST_AirComp}
       \vspace{-3mm}
\end{figure*}	

{\color{black}
\subsubsection{Convolutional neural networks} \label{simu_pure_HIST:cnn}
We numerically investigate the performance of  \texttt{HIST} on CNNs by training LeNet-5 and ResNet-18 on Fashion MNIST and CIFAR-$10$, respectively. Note that the total number of parameters in the convolutional layers of LeNet-5 is $2572$, while the fully connected layers consist of $59134$ parameters, which is $95.8 \%$ of the entire model.  Consequently, we only partition the fully connected layers in LeNet-5. In contrast, for ResNet-18, where the majority of parameters reside in the convolutional layers, we focus solely on partitioning those layers.
 
Similar to our previous experiments, we consider the communication cost incurred by \texttt{HIST} and HFedAvg for attaining some preset accuracies.
For the Fashion-MNIST task using LeNet-5, we set the target accuracies to $70\%$ and $80\%$, while for the CIFAR-$10$ with ResNet-18, the target accuracies are set to $75\%$ and $80\%$.
The results are shown in Tables \ref{tab:converge:fmnist} and \ref{tab:converge:cifar}.
It can be observed from both Tables \ref{tab:converge:fmnist} and \ref{tab:converge:cifar} that \texttt{HIST} incurs significantly less communication cost than HFedAvg. 
Moreover, as we observed in Figs.~\ref{desired_acc_noniid}\&\ref{desired_acc_mixed}, the advantage of the \texttt{HIST} algorithm gets more significant when the number of cells increases from $N=2$ to $N=4$.
For example, referring to Table \ref{tab:converge:cifar}, for the $80\%$ accuracy level in the fully non-i.i.d. case, the communication cost of the \texttt{HIST} algorithm is $\frac{1}{1.65}$ of what is required by HFedAvg when the number of cells is $N=2$. Notably, this ratio improves to $\frac{1}{2.83}$ as the number of cells increases to $N=4$, indicating that the \texttt{HIST} algorithm becomes more advantageous (with a smaller ratio signifying better performance) in scenarios with more cells.  
This shows that the proposed hierarchical submodel training methodology significantly enhances efficiency in the training of convolutional layers over hierarchical networks as well.
}

\vspace{-0.05in}
\subsection{Effect of Partitioning Optimization in OMA-based HIST}\label{simu_HIST_opt}
We next evaluate the effectiveness of our proposed partition optimization strategy from Section \ref{sec_partion_opt} in reducing the training latency. We consider training the fully connected neural network on Fashion-MNIST under the non-i.i.d. data setup, using the same model architecture from Section \ref{simu_pure_HIST:fcn}.
In the following experiments, we determine the CPU frequency for each client based on their cell location. Specifically, for client $i \in \mathcal{C}_j,~j\in [1,\frac{N}{2}]$, the CPU frequency is chosen randomly within the range $ (1,2)$ GHz. For clients in cells numbered in the latter half, the frequency range from which we draw is $ (2,4)$ GHz.
Meanwhile, we set $\text{SNR}$ to $30$ dB for clients $i\in \mathcal{C}_j$, $j\in [1,\frac{N}{2}]$ and $40$ dB for for clients $i\in \mathcal{C}_j$, $j\in (\frac{N}{2}, N]$.  
The channel is modeled as $\bm{h}_i^{t} \sim \mathcal{CN}(0,\mathbf{I}_M), ~\forall i,t$. In addition, $V_0$ is set to $10^6$ cycles per update and $\epsilon_{th}$ is set to $9 \delta_1^2$. The number of antennas at each edge server is set to $M=10$.

In Fig. \ref{mask_opt_OMA}, we compare the performance of the \texttt{HIST} algorithm with and without integration of mask optimization. In this context, `with mask optimization' is denoted as `w', while `uniform partition' is denoted as `w/o'. This experiment was conducted with $N=4$ and $E=5$. 
As shown in Fig. \ref{mask_opt_rounds}, when using the optimized mask tailored to training latency minimization,  \texttt{HIST}   exhibits a slightly slower convergence speed in terms of communication rounds, compared to the case with a uniform partition. Note that uniform partition is considered optimal in terms of the convergence bound but it does not consider the network resource availability of the system. On the other hand, as depicted in Fig. \ref{mask_opt_latency_b}, the optimized mask effectively reduces the training latency of the \texttt{HIST} algorithm across various settings. This is because an optimized partition will not compromise  the convergence speed largely based on   condition \eqref{cons_2} (as shown in Fig. \ref{mask_opt_rounds}) but notably reduce per-round training latency. 

Fig. \ref{mask_opt_latency_c} compares the training latency of   \texttt{HIST} with mask optimization, the \texttt{HIST} with a uniform partitioning, and HFedAvg. The latency of all these algorithms decreases as the number of cells increases from $2$ to $6$, which can be attributed to the fact that the communication complexity at each edge server linear decreases as the number of clients in the cell decreases. 
Additionally, when the number of cells increases from $6$ to $8$, HFedAvg experiences a significant increase in training latency while \texttt{HIST} shows comparatively less degradation.
This can be attributed to the fact that the per-round reduction in training latency partially compensates for the degradation  in convergence speed. Additionally, it is clear that \texttt{HIST} with partitioning optimization can always beat these baselines for all $N \in \{2,4,6,8\}$. These observations confirm that the \texttt{HIST} combined with an optimized mask could improve the training efficacy by a wide margin. 


 \vspace{-0.05in}
\subsection{Performance Evaluation of the AirComp-assisted \texttt{HIST}}\label{simu_HIST_AirComp}
Finally, we evaluate the performance of AirComp-assisted HIST from Section \ref{sec_aircomp}.
 Fig. \ref{HIST_AirComp_SNR_fmnist} plots the convergence behavior of the AirComp-assisted \texttt{HIST} algorithm in training the fully connected neural network on Fashion-MNIST under different SNR values (defined as $P_j / \sigma_0^2$, as in Section \ref{simu_HIST_opt}). This experiment was carried out with $N=4$, $H=40$, and $E=5$.  We take OMA-based \texttt{HIST}, studied in the last subsection, as a benchmark whose convergence behavior (with respect to communication rounds) is independent of the SNR value.
The results show that AirComp-assisted \texttt{HIST} performs well under a wide SNR region (as long as it is larger than $-10$ dB), which reveals that it is robust to channel conditions.
However, when the SNR becomes extremely low (e.g., $-20$ dB), the convergence behavior degrades and leads to large oscillations.
A low SNR leads to a large variance in the model aggregated by AirComp, preventing \texttt{HIST} from converging.
Such a phenomenon fits well with our theoretical analysis provided in Section \ref{sec_aircomp}, since the SNR will directly impact the MSE, which makes the convergence performance poorer.

Fig. \ref{HIST_AirComp_SNR_latency} shows the required training latency for attaining a target testing accuracy of $80 \%$. We compare the training latency of the OMA-based and AirComp-assisted HFedAvg and \texttt{HIST} under different SNR values. The CPU frequencies of clients are drawn from $ (2,4)$ GHz. Other settings are the same as that of Fig. \ref{HIST_AirComp_SNR_fmnist}. 
As shown in Fig. \ref{HIST_AirComp_SNR_latency}, the training latency decreases as the SNR increases for all algorithms. 
For the OMA scheme, an error-free aggregation, the SNR affects the upload rate, which in turn impacts the latency experienced in each round of communication. In the AirComp scheme, a larger SNR results in a smaller error during aggregation, and thus reduces the training latency for achieving the preset testing accuracy. Importantly, we see that AirComp-assisted \texttt{HIST} achieves substantially improved training latency for all choices of SNR, showing the joint advantage of AirComp and submodel partitioning design in improving efficiency.

Fig. \ref{HIST_AirComp_K} shows the impact of the number of clients in each cell on the training latency. 
We consider $N=4$, $H=40$, and $E=5$.
The number of clients in each cell is varied across $n_j \in  \{10, 15, 20, 25, 30\}$, with the total number of clients $N n_j$ varying accordingly.    
Other parameters including CPU frequencies and SNRs are the same as in Section \ref{simu_HIST_opt}.
As depicted in Fig. \ref{HIST_AirComp_K},  the AirComp-assisted \text{HIST} significantly reduces latency compared to both the OMA-based \texttt{HIST} and the AirComp-assisted HFedAvg.
Additionally, the latency of the OMA-based \texttt{HIST} has a noticeable upward trend as the number of clients increases, while the fluctuation of the AirComp-assisted algorithms is slighter. This can be attributed to the fact that the communication latency of the OMA schemes linearly increases with the number of clients, whereas the per-round latency of the AirComp-assisted algorithms is independent of the number of clients.
For AirComp-assisted \text{HIST}, the downward trend in training latency from $n_j = 10$ to $20$ can be attributed to a speedup of convergence coming from having more clients in each cell, consistent with the effect of a decreasing $\tilde{N}$ in Theorem \ref{theorem_non_iid_hierarchical_AirComp}.
On the other hand, as the number of clients grows large, this will induce a higher diversity across the data distributions in each cell. The impact of this is an increasing $\delta_2$ (within-cell gradient dissimilarity), in Theorem \ref{theorem_non_iid_hierarchical_AirComp} makes the convergence bound worse. Accordingly, in this experiment, we see that as the number of clients exceeds a specific value, i.e., $25$, the impact of the induced data heterogeneity begins to dominate any speedup from having more devices, thus slowing down the overall convergence.

\vspace{-0.05in}
\section{Conclusion {\color{black}and Future Work}}\label{sec_conclusion}
In this paper, we proposed hierarchical federated submodel training (\texttt{HIST}), which integrates independent submodel partitioning into hierarchical FL to obtain improvements in communication, computation, and storage efficiencies for training neural networks. We characterized the convergence behavior of \texttt{HIST} with arbitrary submodel partitioning under non-convex loss functions and non-i.i.d. data settings. This revealed the impacts of submodel partitioning sizes, the degree of non-i.i.d. data, and other factors on the convergence performance. 
Based on the derived convergence bound, we proposed an algorithm for optimizing the model partition strategy across cells, minimizing the training latency subject to maintaining a desired loss. 
Subsequently, we adopted AirComp-assisted local aggregations within each cell to further enhance the efficiency of \texttt{HIST} over hierarchical wireless networks.
Numerical evaluations showed that  \texttt{HIST} is able to achieve target accuracies significantly faster and with less resource costs compared to a baseline. Improvements in training latency from our submodel partitioning optimization and AirComp were also demonstrated empirically. 

In this work, we verified that {\tt HIST} is applicable to both fully connected layers and CNN layers without strict constraints. Additionally, when considering transformer-based models, we showed how the LoRA concept can be employed to apply our method for fine-tuning. Future work can investigate the {\tt HIST} strategy in the context of other deep learning architectures.

\balance
\bibliographystyle{ieeetr}
\bibliography{confer_ref.bib}


\onecolumn
\appendix

{\color{black}
\subsection{Extension to LoRA Fine-tuning of Transformers}\label{appendix:transformer}
Transformer-based models have received a lot of recent attention, particularly for large language model (LLM) tasks. Training them over edge networks is generally considered to be impractical due to their large sizes. Here, we examine the extension of {\tt HIST} to fine-tuning such models for downstream tasks.

Consider applying the popular Low-Rank Adaptation (LoRA) \cite{hu2021lora} technique to fine-tune a pre-trained transformer model on data spread across a set of edge devices, e.g., for LLM personalization. In the forward pass of a LoRA module, the computation follows the transformation $\bm y = \bm B \bm A \bm z$, where $\bm A$ and $\bm B$ are the down-projection and up-projection matrices, respectively, and $\pm z$ is the input. This can be viewed as a fully connected network with a single hidden layer, where $\bm A$ and $\bm B$ represent the weight matrices of the first and second layers, respectively.
To extend model partitioning to LoRA, we can introduce a sparse diagonal matrix $\bm{S}$ with diagonal elements restricted to $0$ or $1$, between $\bm A$ and $\bm B$, which functions similarly to the mask in partitioning hidden neurons, i.e., $\bm{X}_l = \bm{B} \bm{S}_l \bm{A}$ assigns the submatrix of parameters for partition $l$. For example, if the 2nd and 4th elements of the diagonal matrix $\bm{S}_l$ are non-zero, this is equivalent to $\bm X_l$ possessing the 2nd and 4th rows of $\bm A$ and the 2nd and 4th columns of $\bm B$, i.e., the 2nd and 4th neurons in the hidden layer.

\textbf{Experiments.}\label{simu_pure_HIST:lora} 
We consider fine-tuning GPT-2 small, which is a version of the GPT-2 architecture \cite{radford2019language}, a popular decoder-only transformer model available in 5 sizes. GPT-2 is designed for tasks involving text generation and understanding, and the ``small" version has around 124 million parameters, making it more computationally manageable while still powerful for various LLM tasks. In our experiments, we applied this model to the 20 Newsgroups dataset (\href{http://qwone.com/~jason/20Newsgroups}{http://qwone.com/~jason/20Newsgroups}), a widely used benchmark for text classification. The dataset consists of documents from 20 different categories, making it an appropriate test case for evaluating how well the model can classify complex text data. By partitioning the LoRA module across the transformer layers of GPT-2, we aim to enhance communication efficiency during fine-tuning, optimizing the model’s performance for the Newsgroups classification task. 

Following the setting in \cite{kuo2024federated}, we set the rank of the LoRA modules to $16$. We then consider two partitioning scenarios: $N = 2$ and $N = 4$. 
In the first case, where $N = 2$, the LoRA module is split into two smaller modules, each with a rank of 8. Each client is responsible for training one of these smaller modules. In the second case, where $N = 4$, the LoRA module is divided into four smaller modules, each with a rank of 4, and each client trains one of these modules.
In both cases, we set the number of local training rounds $H = 40$ and the number of communication rounds $E = 5$. We compared the communication costs required by the proposed \texttt{HIST} and HFedAvg algorithms to achieve target accuracies of $70\%$ and $76\%$. The communication cost per client for reaching these accuracies is presented in Table \ref{tab:converge:gpt}. Across all the data and cell configurations considered, we draw the same conclusions as with the fully connected and convolutional neural network models from Sec.~\ref{simu_pure_HIST}: \texttt{HIST} consistently requires less communication, and the improvement is more pronounced as $N$ is increased.
}

\begin{table*}[h!]
\centering
\begin{tabular}{lclclclclclclclclcl}
\toprule
\multicolumn{17}{c}{\color{black}  Fine-tuning GPT-2 Model on 20 Newsgroups Dataset.} \\
\midrule
& & \multicolumn{7}{c}{Fully non-i.i.d.} & & \multicolumn{7}{c}{Cell i.i.d., client non-i.i.d.} \\ \cmidrule{3-9} \cmidrule{11-17}
& & \multicolumn{3}{c}{$2$ Cells} & & \multicolumn{3}{c}{$4$ Cells} & & \multicolumn{3}{c}{$2$ Cells}  & &  \multicolumn{3}{c}{$4$ Cells} \\
\cmidrule{3-5} \cmidrule{7-9} \cmidrule{11-13} \cmidrule{15-17}
Accuracy & & \texttt{HIST} & &   \texttt{HFedAvg} & & \texttt{HIST} & & \texttt{HFedAvg} & & \texttt{HIST} & & \texttt{HFedAvg} & & \texttt{HIST} & & \texttt{HFedAvg} \\
\cmidrule{1-1} \cmidrule{3-3} \cmidrule{5-5} \cmidrule{7-7} \cmidrule{9-9} \cmidrule{11-11} \cmidrule{13-13} \cmidrule{15-15} \cmidrule{17-17} 
70\% && 50.63 && 78.75 && 33.75 && 83.25 && 32.63 && 56.25 && 21.94 && 60.75 \\
76\% && 168.75 && 258.75 && 100.69 && 267.75 && 111.38 && 166.50 && 72.56 && 177.45 \\
\bottomrule
\end{tabular}
\caption{\color{black} Communication cost (MB) per client for achieving testing accuracy of $70\%$ and $76\%$ for fine-tuning GPT-2 small model on 20 Newsgroups dataset. The results show that the \texttt{HIST} algorithm obtains improvements for tasks beyond training fully connected and convolutional neural networks, such as LoRA fine-tuning.}
\label{tab:converge:gpt}
\vspace{-2mm}
\end{table*}

\subsection{Proof of Theorems \ref{theorem_non_iid_hierarchical} - \ref{theorem_non_iid_hierarchical_AirComp}}\label{sec_appen_proof_theorem}


For analysis, we introduce virtual iterate $
\hat{\bm x}^{t,e} = \sum_{j=1}^N \bar{\bm x}^{t,e}_j,
$ which denotes the virtual synchronized global model of the proposed \texttt{HIST} algorithm. {\color{black} We denote $\hat{\bm x}^{t,e}_{i,h}, \forall i \in \mathcal{C}_j$ as virtual local model for client $i$, defined as
\begin{equation}\label{}
 \hat{\bm x}^{t,e}_{i,h+1} = \hat{\bm x}^{t,e}_{i,h} - \gamma {\bm p}^t_j \odot \nabla l(\hat{\bm x}^{t,e}_{i,h}, \xi^{t,e}_{i,h}) , h = 0,1,\ldots, H\!-\!1, \nonumber
 \end{equation}
 and $\hat{\bm x}^{t,e}_{i,0}={\bm x}^{t,e}_{i,0}$. Due to the uniform sampling, we have $$\mathbb{E}[\frac{1}{|\mathcal{C}_j^{t,e}|}\sum_{i\in \mathcal{C}_j^{t,e}} \hat{\bm x}^{t,e}_{i,h}] = \frac{1}{|\mathcal{C}_j|} \sum_{i\in \mathcal{C}_j} {\bm x}^{t,e}_{i,h}.$$}
For notational ease, we define
\equa{\label{D_Q_t}
D_t &= \sum_{j=1}^N\sum_{e=0}^{E-1} \mathbb{E} \left \|\hat{\bm x}^{t,e}  - \bar{\bm x}^{t,e}_j \right \|^2~\text{and}  \\ Q_t &= \sum_{j=1}^N \frac{1}{n_j}\sum_{i \in \mathcal{C}_j} \frac{1}{H} \! \sum_{e=0}^{E-1} \sum_{h=0}^{H-1} \mathbb{E} \left\|\bar{\bm x}^{t,e}_j -{\color{black} \hat{\bm x}^{t,e}_{i,h}}  \right\|^2.
}

The proofs of Theorems \ref{theorem_non_iid_hierarchical}, \ref{theorem_non_iid_hierarchical_partial}, and \ref{theorem_non_iid_hierarchical_AirComp} rely on the following four lemmas, which are in turn proven in the supplementary material.

\begin{lemma}\label{lemma_function_value}
Suppose that Assumptions \ref{assump_smoothness}-\ref{assump_gradient_dissimilarity_edge} hold, $N \geq 2$, and $\gamma \leq \frac{1}{2HL}$. Then, 
\vspace{2mm}

{\color{black} 1) \textbf{With full participation},} the iterates generated by the \texttt{HIST} algorithm satisfy
\begin{small}
\begin{align}\label{equa_lemma1_full}
\mathbb{E}[f(\bar{\bm x}^{t+1})] 
		\leq & \mathbb{E}[f(\bar{\bm x}^t)] - \frac{\gamma H }{2}  \sum_{e=0}^{E-1} \mathbb{E} \left\| \nabla f(\hat{\bm x}^{t,e} ) \right\|^2 + \gamma^2 E H \tilde{N} L \sigma^2 \nonumber \\
  &+  \frac{3\gamma}{2} \frac{EHNd_{\textrm{max}}}{d} \delta_1^2 + \frac{3H\gamma L^2}{2} \left( D_t +Q_t \right).
\end{align}
\end{small}

{\color{black}
2) \textbf{With partial participation}, the iterates generated by the \texttt{HIST} algorithm satisfy
\begin{small}
\begin{align}\label{equa_lemma1_partial}
\mathbb{E}[f(\bar{\bm{x}}^{t+1})] 
\leq  & \mathbb{E}[f(\bar{\bm{x}}^t)] - \frac{\gamma H }{2}  \sum_{e=0}^{E-1} \mathbb{E} \left\| \nabla f(\hat{\bm{x}}^{t,e} ) \right\|^2 + \gamma^2 E H \tilde{N}^{\prime} L \sigma^2 \nonumber \\
&+\!  \frac{3\gamma}{2} \frac{EHNd_{\textrm{max}}}{d} \delta_1^2 
\!+\! 3H\gamma L^2 \left( D_t \!+\! Q_t \right) \nonumber \\
&+\! 3\gamma^2 E H^2 \tilde{N}^{\prime} L \delta_2^2.
\end{align}
\end{small}
}
\end{lemma}

\begin{lemma}\label{lemma_function_value_AirComp}
Suppose that Assumptions \ref{assump_smoothness}-\ref{assump_gradient_dissimilarity_edge} hold, $N \geq 2$, and $\gamma \leq \frac{1}{2HL}$. Then {\color{black}with full participation}, the iterates generated by the AirComp-assisted \texttt{HIST} algorithm satisfy
\begin{small}
\begin{align}\label{}
\mathbb{E}[f(\bar{\bm x}^{t+1})] 
		\leq & \mathbb{E}[f(\bar{\bm x}^t)] - \frac{\gamma H }{2}  \sum_{e=0}^{E-1} \mathbb{E} \left\| \nabla f(\hat{\bm x}^{t,e} ) \right\|^2 + \gamma^2 E H \tilde{N} L \sigma^2 \nonumber \\
  &+  \frac{3\gamma}{2} \frac{EHNd_{\textrm{max}}}{d} \delta_1^2 + \frac{1}{2} \sum_{e=0}^{E-1} \sum_{j=1}^N \textnormal{MSE}_j^{t,e} \nonumber \\
  & + \frac{3H\gamma L^2}{2} \left( D_t +Q_t \right).
\end{align}
\end{small}
\end{lemma}

\begin{lemma}\label{lemma_edge_divergence}
Suppose that Assumptions \ref{assump_smoothness}-\ref{assump_gradient_dissimilarity_edge} hold and $\gamma \leq \frac{1}{EL\sqrt{54(N+1)}}$. Then the difference between the edge models and the global model can be bounded as
\begin{small}
        \begin{align}
            D_t \leq & 162 \gamma^2 (N-1)E^2H^2\sum_{e = 0}^{E-1}\mathbb{E}\|\nabla f(\hat{\bm x}^{t,e})\|^2 +  4 E (N-1) \mathbb{E} \left \|\bar{\bm x}^{t}  \right\|^2  \nonumber \\
& + 18\gamma^2(N+1)E^2H^2L^2 Q_t   + 6 \gamma^2 (N+1)E^2H {\color{black}\tilde{N}^{\prime}} \sigma^2  \nonumber\\
&+ 108 \gamma^2 (N+1)E^3H^2 \frac{Nd_{\textrm{max}}}{d} \delta_1^2.
        \end{align}
\end{small}
\end{lemma}
\begin{lemma}\label{lemma_client_divergence}
Suppose that Assumptions \ref{assump_smoothness}-\ref{assump_randomness_sgd} and \ref{assump_gradient_dissimilarity_client} hold and $\gamma \leq \frac{1}{\sqrt{15}HL}$. Then the difference between the local models and the edge models can be bounded as
\begin{small}
\begin{align}
Q_t \leq & 5 \gamma^2 H^2 L^2 D_t + 5 \gamma^2 N H^2 \sum_{e=0}^{E-1}  \mathbb{E} \left\| \nabla f( \hat{\bm x}^{t,e} )  \right\|^2 \nonumber \\
&+ 2 \gamma^2 N H E \sigma^2 + 5 \gamma^2 N H^2 E \delta_2^2+ 5 \gamma^2 N H^2 E \delta_1^2.
\end{align}
\end{small}
\end{lemma}

Lemmas \ref{lemma_function_value} and \ref{lemma_function_value_AirComp} characterize the dynamics of the global loss function. Lemmas \ref{lemma_edge_divergence} and \ref{lemma_client_divergence}
characterize the upper bound of the diversity between the virtual global model and edge models and between the edge model and local models, respectively. {\color{black}It is worth noting that Lemmas \ref{lemma_edge_divergence} and \ref{lemma_client_divergence} are consistent for both full participation and partial participation.} 

\subsubsection{Proof of Theorem \ref{theorem_non_iid_hierarchical}}
\label{proof_theorems}

{\color{black}With the first part of Lemma \ref{lemma_function_value}, Lemmas \ref{lemma_edge_divergence} and \ref{lemma_client_divergence},  we can prove Theorem \ref{theorem_non_iid_hierarchical} as follows.}

Based on Lemmas \ref{lemma_edge_divergence} and \ref{lemma_client_divergence}, we have
\begin{small}
\begin{align}\label{sum_D_t_Q_t}
D_t + Q_t \leq & \tilde{\alpha} \left\{ 162 \gamma^2 (N-1)E^2H^2\sum_{e = 0}^{E-1}\mathbb{E}\|\nabla f(\hat{\bm x}^{t,e})\|^2 
+  4 E (N-1) \mathbb{E} \left \|\bar{\bm x}^{t}  \right\|^2  + 6 \gamma^2 (N+1)E^2H \tilde{N} \sigma^2 \right. \nonumber \\
&\left. + 108 \gamma^2 (N+1)E^3H^2 \frac{Nd_{\textrm{max}}}{d} \delta_1^2 \right\} \!+\! \tilde{\beta} \Big \{ 2 \gamma^2 N H E \sigma^2 \!+\! 5 \gamma^2 N H^2 E \delta_2^2 \!+\! 5 \gamma^2 N H^2 E \delta_1^2 \!+\! 5 \gamma^2 N H^2 \sum_{e=0}^{E-1}  \mathbb{E} \left\| \nabla f( \hat{\bm x}^{t,e} )  \right\|^2 \Big \},
\end{align}
\end{small}
\noindent where $\tilde{\alpha} = \frac{1+5 \gamma^2 H^2 L^2 }{1-18\gamma^2(N+1)E^2H^2L^2*5 \gamma^2 H^2 L^2}$ and $\tilde{\beta} = \frac{1+18\gamma^2(N+1)E^2H^2L^2 }{1-18\gamma^2(N+1)E^2H^2L^2*5 \gamma^2 H^2 L^2}$. According to the setting of $\gamma$, one can claim  $\tilde{\alpha} \leq 2, ~\tilde{\beta} \leq 2$. 
We thus have
\begin{small}
\begin{align}\label{sum_D_t_Q_t}
D_t + Q_t \leq & \gamma^2 \left( 324 (N-1)E^2H^2 + 10 N H^2 \right) \sum_{e = 0}^{E-1}\|\nabla f(\hat{\bm x}^{t,e})\|^2 +  8 E (N-1) \mathbb{E} \left \|\bar{\bm x}^{t}  \right\|^2 + \gamma^2 \left( (12 (N+1)E^2H \tilde{N} + 4 N H E \right) \sigma^2 \nonumber \\
& + \gamma^2 \left(216 (N+1)\frac{Nd_{\textrm{max}}}{d}E^3H^2 + 10  N H^2 E \right) \delta_1^2 + 10 \gamma^2 N H^2 E \delta_2^2. 
\end{align}
\end{small}

Additionally, {\color{black}according to \eqref{equa_lemma1_full} of Lemma \ref{lemma_function_value},} we obtain
\begin{small}
\begin{align}
\label{Lemma1_called_theorem}
\sum_{e=0}^{E-1} \mathbb{E} \left\| \nabla f(\hat{\bm x}^{t,e} ) \right\|^2 \leq & \frac{2}{\gamma H}\mathbb{E}[f(\bar{\bm x}^t)] - \mathbb{E}[f(\bar{\bm x}^{t+1})] + 2 \gamma E \tilde{N} L \sigma^2 +  3 \frac{ENd_{\textrm{max}}}{d} \delta_1^2 \!+\! 3 L^2 (D_t\!+\!Q_t).
\end{align}
\end{small}
Plugging \eqref{sum_D_t_Q_t} into \eqref{Lemma1_called_theorem}, we can write
\begin{small}
\begin{align}\label{}
	& \left(1 -  \gamma^2 L^2 \left( 972 (N-1)E^2H^2 + 30 N H^2 \right) \right) \sum_{e=0}^{E-1} \mathbb{E} \left\| \nabla f(\hat{\bm x}^{t,e} ) \right\|^2	\nonumber \\
 \leq & 2\frac{\mathbb{E}[f(\bar{\bm x}^t)] - \mathbb{E}[f(\bar{\bm x}^{t+1})]}{\gamma H}  + 2 \gamma E \tilde{N} L \sigma^2 +  3 \frac{ENd_{\textrm{max}}}{d} \delta_1^2  + \gamma^2 L^2 \left( (36 (N+1)E^2H \tilde{N} + 12 N H E \right) \sigma^2 \nonumber \\
& + \gamma^2 L^2 \left(648 (N+1)\frac{Nd_{\textrm{max}}}{d}E^3H^2 + 30  N H^2 E \right) \delta_1^2 + 30 \gamma^2 L^2 N H^2 E \delta_2^2 +  24 E (N-1) L^2 \mathbb{E} \left \|\bar{\bm x}^{t}  \right\|^2  .
\end{align}
\end{small}
Based on the condition of $\gamma$ (i.e., \eqref{condition_gamma}), we have $1 -  \gamma^2 L^2 \left( 972 (N-1)E^2H^2 + 30 N H^2 \right) \leq \frac{1}{2}$. As a result, we obtain
\begin{small}
\begin{align}\label{}
	\frac{1}{E}\sum_{e=0}^{E-1} \mathbb{E} \left\| \nabla f(\hat{\bm x}^{t,e} ) \right\|^2
 \leq &  4 \frac{\mathbb{E}[f(\bar{\bm x}^t)] - \mathbb{E}[f(\bar{\bm x}^{t+1})]}{\gamma EH} + \left(4 \gamma \tilde{N} L +  \gamma^2 L^2 \left( (72 (N+1)EH \tilde{N} + 24 N H \right) \right) \sigma^2 \nonumber \\
 & + \gamma^2 L^2 \left(1296 (N+1)\frac{Nd_{\textrm{max}}}{d}E^2H^2 + 60  N H^2 \right)\delta_1^2 \nonumber \\
&+ 60 \gamma^2 L^2 N H^2 \delta_2^2  +  6 \frac{Nd_{\textrm{max}}}{d} \delta_1^2 +  48 (N-1) L^2 \mathbb{E} \left \|\bar{\bm x}^{t}  \right\|^2.
\end{align}
\end{small}

Utilizing the condition of $\gamma$ again, we obtain Theorem \ref{theorem_non_iid_hierarchical}. 

{\color{black}
\subsubsection{Proof of Theorem \ref{theorem_non_iid_hierarchical_partial}}
With the second part of Lemma \ref{lemma_function_value}, Lemmas \ref{lemma_edge_divergence} and \ref{lemma_client_divergence},  we can prove Theorem \ref{theorem_non_iid_hierarchical} as follows. 

First, according to \eqref{equa_lemma1_partial} of Lemma \ref{lemma_function_value}, we obtain 
\begin{small}
\begin{align}
\label{Lemma1_called_theorem_partial}
\sum_{e=0}^{E-1} \mathbb{E} \left\| \nabla f(\hat{\bm x}^{t,e} ) \right\|^2 \leq & \frac{2}{\gamma H}\mathbb{E}[f(\bar{\bm x}^t)] - \mathbb{E}[f(\bar{\bm x}^{t+1})] + 2 \gamma E \tilde{N}^{\prime} L \sigma^2 
+  6\gamma E H \tilde{N}^{\prime} L \delta_2^2 +  3 \frac{ENd_{\textrm{max}}}{d} \delta_1^2 \!+\! 6 L^2 (D_t\!+\!Q_t).
\end{align}
\end{small}

The bound for $D_t\!+\!Q_t$ derived in \eqref{sum_D_t_Q_t} is based on Lemmas \ref{lemma_edge_divergence} and \ref{lemma_client_divergence}, which are consistent for both full and partial client participation. Therefore, we can directly apply this bound here by substituting $\tilde{N}$ with $\tilde{N}^{\prime}$.

Plugging \eqref{sum_D_t_Q_t} into \eqref{Lemma1_called_theorem} gives rise to
\begin{small}
\begin{align}\label{}
	&\left(1 -  2\gamma^2 L^2 \left( 972 (N-1)E^2H^2 + 30 N H^2 \right) \right) \sum_{e=0}^{E-1} \mathbb{E} \left\| \nabla f(\hat{\bm x}^{t,e} ) \right\|^2	\nonumber \\
 \leq & 2\frac{\mathbb{E}[f(\bar{\bm x}^t)] \!-\! \mathbb{E}[f(\bar{\bm x}^{t+1})]}{\gamma H}  \!+\! 2 \gamma E \tilde{N}^{\prime} L \sigma^2 \!+\!  6\gamma E H \tilde{N}^{\prime} L \delta_2^2 \!+\!  3 \frac{ENd_{\textrm{max}}}{d} \delta_1^2 + 2 \gamma^2 L^2 \left( (36 (N+1)E^2H \tilde{N} + 12 N H E \right) \sigma^2 \nonumber \\
& + 2 \gamma^2 L^2 \left(648 (N+1)\frac{Nd_{\textrm{max}}}{d}E^3H^2 + 30  N H^2 E \right) \delta_1^2 + 60 \gamma^2 L^2 N H^2 E \delta_2^2 +  48 E (N-1) L^2 \mathbb{E} \left \|\bar{\bm x}^{t}  \right\|^2.
\end{align}
\end{small}

Based on the condition of $\gamma$, we have $1 -  2 \gamma^2 L^2 \left( 972 (N-1)E^2H^2 + 30 N H^2 \right) \leq \frac{1}{2}$. As a result, we obtain
\begin{small}
\begin{align}\label{}
	\frac{1}{E}\sum_{e=0}^{E-1} \mathbb{E} \left\| \nabla f(\hat{\bm x}^{t,e} ) \right\|^2
 \leq & 4 \frac{\mathbb{E}[f(\bar{\bm x}^t)] - \mathbb{E}[f(\bar{\bm x}^{t+1})]}{\gamma EH}  + \left(4 \gamma \tilde{N}^{\prime} L +  2 \gamma^2 L^2 \left( (72 (N+1)EH \tilde{N} + 24 N H \right) \right) \sigma^2 \nonumber \\
 & + 2 \gamma^2 L^2 \left(1296 (N+1)\frac{Nd_{\textrm{max}}}{d}E^2H^2 + 60  N H^2 \right)\delta_1^2 + 120 \gamma^2 L^2 N H^2 \delta_2^2   +  12 \gamma H \tilde{N}^{\prime} L \delta_2^2 \nonumber \\ 
&+  6 \frac{Nd_{\textrm{max}}}{d} \delta_1^2 +  96 (N-1) L^2 \mathbb{E} \left \|\bar{\bm x}^{t}  \right\|^2.
\end{align}
\end{small}

Based on the setting of $\gamma$ in \eqref{condition_gamma_partial}, it follows that
\begin{align}\label{}
	\frac{1}{TE} \sum_{t=0}^T \sum_{e=0}^{E-1} \mathbb{E} \left\| \nabla f(\hat{\bm x}^{t,e} ) \right\|^2
 \leq & 4 \frac{f(\bar{\bm x}^0) - f_*}{\gamma TEH}  + 100 \gamma \tilde{N}^{\prime} L \sigma^2 + 1356 \gamma L \delta_1^2 + 60 \gamma L \delta_2^2 +  12 \gamma H \tilde{N}^{\prime} L \delta_2^2  \nonumber \\
& +   6 \frac{N}{d} \frac{1}{T}\sum_{t=0}^{T-1} d^t_{\textrm{max}} \delta_1^2 +  96 (N-1) L^2 \frac{1}{T}\sum_{t=0}^{T-1} \mathbb{E} \left \|\bar{\bm x}^{t}  \right\|^2. \nonumber
\end{align}

Setting $\gamma = (T\!E\!H)^{-\!\frac{1}{2}}$ and ignoring constant multiplicative factors (including $L$), we have
\begin{small}
\begin{align}\label{}
	\frac{1}{TE} \sum_{t=0}^T \sum_{e=0}^{E-1} \mathbb{E} \left\| \nabla f(\hat{\bm x}^{t,e} ) \right\|^2
 \leq & 4 \frac{f(\bar{\bm x}^0) - f_*}{\gamma TEH}  + 100 \gamma \tilde{N}^{\prime} L \sigma^2 + 1356 \gamma L \delta_1^2 + 60 \gamma L \delta_2^2 +  12 \gamma H \tilde{N}^{\prime} L \delta_2^2  \nonumber \\
& +   6 \frac{N}{d} \frac{1}{T}\sum_{t=0}^{T-1} d^t_{\textrm{max}} \delta_1^2 +  96 (N-1) L^2 \frac{1}{T}\sum_{t=0}^{T-1} \mathbb{E} \left \|\bar{\bm x}^{t}  \right\|^2 \nonumber \\
\sim &  \mathcal{O}\left( \tilde{N}^{\prime} (TEH)^{-\frac{1}{2}} \right) + \mathcal{O}\left((TEH)^{-\frac{1}{2}} \right) + \mathcal{O}\left( \tilde{N}^{\prime} H^{\frac{1}{2}} (TE)^{-\frac{1}{2}} \right) \nonumber \\
		&+   \mathcal{O} \left(\frac{N}{d} \frac{1}{T}\sum_{t=0}^{T-1} d^t_{\textrm{max}} \delta_1^2 + (N\!-\!1) \frac{1}{T}\sum_{t=0}^{T-1} \mathbb{E} \left \|\bar{\bm x}^{t}  \right\|^2 \right). \nonumber
\end{align}
\end{small}

This completes the proof of Theorem \ref{theorem_non_iid_hierarchical_partial}.
}

\subsubsection{Proof of Theorem \ref{theorem_non_iid_hierarchical_AirComp}}
\label{proof_theorems_aircomp}

With Lemmas \ref{lemma_function_value_AirComp}, \ref{lemma_edge_divergence}, and \ref{lemma_client_divergence}, we can prove Theorem \ref{theorem_non_iid_hierarchical_AirComp} following the same steps as in the proof of Theorem \ref{theorem_non_iid_hierarchical}. Thus, the detailed proof is omitted here for brevity.

\subsection{Proof of Lemmas}\label{proof_lemmas}
Before providing the proofs of lemmas, we introduce some notations. Denote $\mathbb{E}_t^e$ as an expectation conditioned on the historical information up to the start of the $(t,e)$-th round. Let $\bm g^{t,e}_{i,h}$ denote the stochastic gradient $\nabla l({\color{black}\hat{\bm x}^{t,e}_{i,h}}, \xi^{t,e}_{i,h})$, $\xi^{t,e}_{i,h} \sim \mathcal{D}_i$. {\color{black}As $\hat{\bm x}^{t,e}_{i,h} = {\bm x}^{t,e}_{i,h}$, $i \in \mathcal{C}_j$, $\bm g^{t,e}_{i,h}$ is equivalent to the stochastic gradient $\nabla l({\color{black}{\bm x}^{t,e}_{i,h}}, \xi^{t,e}_{i,h})$ of active client $i \in \mathcal{C}_j^{t,e}$.}
In addition, we present three propositions that will be used in this subsection.

\begin{proposition}\label{fact_bounded_masked_norm}
Suppose that mask $\bm p_j$ is randomly generated via rule \eqref{mask_IST} and $\|\bm p_j\|_1 = d_j$, then
\[
\mathbb{E} [\bm p_j \odot \bm z] = \frac{d_j}{d}\|\bm z\|^2,~ \forall j ~\text{and}~\sum_{j=1}^N \left \| {\bm p}^t_j \odot \left( \nabla f(\bm x)
- \nabla f_j(\bm x) \right)  \right \|^2 \leq \frac{Nd_{\textrm{max}}}{d} \delta_1^2, \forall \bm x,
\]
where $d_{\textrm{max}} = \max \{d_1,d_2,\ldots,d_N\}$.
\end{proposition}
\begin{proof}
The former can be derived by the following series of transformations
\[ 
\mathbb{E} [ \|\bm p_j \odot \bm z \|^2 ] = \mathbb{E} [\sum_{k=1}^d ((p_j)_k  z_k)^2] 
= \mathbb{E} [\sum_{k=1}^d ((p_j)_k  z_k)^2]
= \sum_{k=1}^d \mathbb{E} [ ((p_j)_k  z_k)^2] = \sum_{k=1}^d  \frac{d_j}{d} z_k^2 = \frac{d_j}{d} \| \bm z\|^2,
\] 
where $(p_j)_k$ and $z_k$ represent the $k$-th elements of $\bm p_j$ and $\bm z$, respectively. Furthermore, combining it with Assumption \ref{assump_gradient_dissimilarity_edge} gives rise to the latter.
\end{proof}

\begin{proposition}\label{fact_mask_N_1}
    Any masks $\{\bm p_j\}_{j=1}^N$ generated by rule \eqref{mask_IST} satisfy
    \[
    \sum_{j=1}^N\|\bm p_j \odot \bm z - \bm z\|^2 = (N-1) \|\bm z\|^2.
    \]
\end{proposition}
\begin{proof}
    \[
    \sum_{j=1}^N\|\bm p_j \odot \bm z - \bm z\|^2 = \sum_{j=1}^N \{ \|\bm p_j \odot \bm z\| - 2\langle \bm p_j \odot \bm z, \bm z \rangle  + \|\bm z\|^2 \}= (N-1) \|\bm z\|^2.
    \]
\end{proof}
\begin{proposition}\label{fact_matrix_iterations}
    For the following inequalities,
\equa{
x \leq \alpha y + a\\
y \leq \beta x + b,
}
where $\alpha \beta < 1$, we have 
    $x+y \leq \frac{1+ \beta }{1-\alpha \beta } a +  \frac{1+ \alpha }{1-\alpha \beta } b$.
    \end{proposition}

Based on the above propositions, we proof the lemmas.

\subsubsection{Proof of Lemma \ref{lemma_function_value}}


Based on the virtual iteration $\hat{\bm x}^{t,e+1} = \hat{\bm x}^{t,e} - \gamma \sum_{j=1}^N \bm p_j^t \odot {\color{black}\frac{1}{|\mathcal{C}_j^{t,e}|}\sum_{i\in \mathcal{C}_j^{t,e}}} \sum_{h=0}^{H-1} {\bm g}^{t,e}_{i,h}$ and Assumption \ref{assump_smoothness}, we have
\begin{equation}\label{L_smoothness_expand}
	\begin{aligned}
		\mathbb{E}_t^e[f(\hat{\bm x}^{t,e+1})] \leq & f(\hat{\bm x}^{t,e}) \underbrace{- \gamma \mathbb{E}_t^e \left \langle \nabla f(\hat{\bm x}^{t,e}), \sum_{j=1}^N \bm p_j^t \odot {\color{black}\frac{1}{|\mathcal{C}_j^{t,e}|}\sum_{i\in \mathcal{C}_j^{t,e}}} \sum_{h=0}^{H-1} {\bm g}^{t,e}_{i,h} \right \rangle}_{T_1}  +  \gamma^2L \underbrace{\frac{1}{2}\mathbb{E}_t^e \left \| \sum_{j=1}^N \bm p_j^t \odot {\color{black}\frac{1}{|\mathcal{C}_j^{t,e}|}\sum_{i\in \mathcal{C}_j^{t,e}}} \sum_{h=0}^{H-1} {\bm g}^{t,e}_{i,h} \right \|^2}_{T_2}.
	\end{aligned}
\end{equation} 

{\color{black}
Utilizing 
$
\mathbb{E}_t^e \left[
\frac{1}{|\mathcal{C}_j^{t,e}|}\sum_{i\in \mathcal{C}_j^{t,e}} \sum_{h=0}^{H-1} {\bm g}^{t,e}_{i,h} -  \frac{1}{|\mathcal{C}_j^{t,e}|} \sum_{i\in \mathcal{C}_j^{t,e}} \sum_{h=0}^{H-1} \nabla F_i ({\bm x}^{t,e}_{i,h})
\right] = 0
$
and 
$\mathbb{E}_t^e \left[  \frac{1}{|\mathcal{C}_j^{t,e}|} \sum_{i\in \mathcal{C}_j^{t,e}} \sum_{h=0}^{H-1} \nabla F_i ({\bm x}^{t,e}_{i,h}) - \frac{1}{n_j}\sum_{i\in \mathcal{C}_j} \sum_{h=0}^{H-1} \nabla F_i (\hat{\bm x}^{t,e}_{i,h})
\right] = 0$,
}
we rewrite $T_1$ as follows
\equa{\label{T_1_with_T_3}
		T_1 =& - \gamma H \mathbb{E}_t^e \left \langle \nabla f(\hat{\bm x}^{t,e}), \sum_{j=1}^N \bm p_j^t \odot \frac{1}{n_j}\sum_{i\in \mathcal{C}_j} \frac{1}{H}\sum_{h=0}^{H-1} \nabla F_i ({\color{black} \hat{\bm x}^{t,e}_{i,h}}) \right \rangle \\
		= & \frac{\gamma H}{2} \! \left\{ \!\underbrace{\mathbb{E}_t^e \! \left\| \! \nabla f(\hat{\bm x}^{t,e}) \!-\! \sum_{j=1}^N \bm p_j^t \! \odot \! \frac{1}{n_j} \! \sum_{i\in \mathcal{C}_j} \! \frac{1}{H} \! \sum_{h=0}^{H-1} \! \nabla F_i ({\color{black} \hat{\bm x}^{t,e}_{i,h}}) \right\|^2}_{T_3} \!-\! \left\| \nabla f(\hat{\bm x}^{t,e}) \right\|^2 \!-\! \mathbb{E}_t^e \! \left\|\!\sum_{j=1}^N \bm p_j^t \!\odot \! \frac{1}{n_j} \! \sum_{i\in \mathcal{C}_j}  \! \frac{1}{H} \! \sum_{h=0}^{H-1}\! \nabla F_i ({\color{black} \hat{\bm x}^{t,e}_{i,h}}) \right\|^2 \! \right\}.
}
Based on the facts that $\nabla f(\hat{\bm x}^{t,e})  = \sum_{j=1}^N {\bm p}^t_j \odot \nabla f(\hat{\bm x}^{t,e}) $ and $ \|\sum_{j=1}^N {\bm p}^t_j \odot \bm z_j  \|^2 
=\sum_{j=1}^N\| {\bm p}^t_j \odot \bm z_j  \|^2$, we can bound $T_3$ as 
\equa{\label{T_3_T_3_prime}
T_3 
=& \sum_{j=1}^N \mathbb{E}_t^e \! \left\| \bm p_j^t \! \odot \left( \nabla f(\hat{\bm x}^{t,e}) \!-\! \! \frac{1}{n_j} \! \sum_{i\in \mathcal{C}_j} \! \frac{1}{H} \! \sum_{h=0}^{H-1} \! \nabla F_i ({\color{black} \hat{\bm x}^{t,e}_{i,h}}) \right) \right\|^2 \\
\leq & \frac{1}{H} \! \sum_{h=0}^{H-1} \underbrace{\sum_{j=1}^N \mathbb{E}_t^e \left\|{\bm p}^t_j \odot \left( \nabla f(\hat{\bm x}^{t,e}) - \frac{1}{n_j}\sum_{i \in \mathcal{C}_j} \nabla F_i ({\color{black} \hat{\bm x}^{t,e}_{i,h}}) \right) \right\|^2}_{T_3^{\prime}},
}
where the inequality follows Jensen's inequality.
Inserting $\mp \nabla f_j(\hat{\bm x}^{t,e}) \mp \nabla f_j(\bar{\bm x}^{t,e}_j)$ into $T_3^{\prime}$ and calling Cauchy-Schwartz inequality, we have
\equa{\label{T_3_prime}
T_3^{\prime}
\leq & 3 \sum_{j=1}^N\mathbb{E}_t^e\left \| {\bm p}^t_j \odot \left( \nabla f(\hat{\bm x}^{t,e})  
- \nabla f_j(\hat{\bm x}^{t,e}) \right)  \right \|^2 + 3\sum_{j=1}^N\mathbb{E}_t^e \left \|{\bm p}^t_j \odot \left(\nabla f_j(\hat{\bm x}^{t,e}) - \nabla f_j (\bar{\bm x}^{t,e}_j) \right)\right \|^2 \\
&+ 3 \sum_{j=1}^N \mathbb{E}_t^e \left \|{\bm p}^t_j \odot \left( \nabla f_j(\bar{\bm x}^{t,e}_j) - \frac{1}{n_j}\sum_{i \in \mathcal{C}_j} \nabla F_i ({\color{black} \hat{\bm x}^{t,e}_{i,h}}) \right) \right \|^2.
}
Given Proposition \ref{fact_bounded_masked_norm}, it follows that
\equa{\label{T_3_prime_1}
3\sum_{j=1}^N \left \| {\bm p}^t_j \odot \left( \nabla f(\hat{\bm x}^{t,e})  
- \nabla f_j(\hat{\bm x}^{t,e}) \right)  \right \|^2 \leq 3\frac{Nd_{\textrm{max}}}{d} \delta_1^2.
}
In addition, recalling Jensen's inequality, we have
\equa{\label{T_3_prime_3}
\sum_{j=1}^N \mathbb{E}_t^e \left \|{\bm p}^t_j \odot \left( \nabla f_j(\bar{\bm x}^{t,e}_j) - \frac{1}{n_j}\sum_{i \in \mathcal{C}_j} \nabla F_i ({\color{black} \hat{\bm x}^{t,e}_{i,h}}) \right) \right \|^2 =& \sum_{j=1}^N \mathbb{E}_t^e \left \|{\bm p}^t_j \odot \left( \frac{1}{n_j}\sum_{i \in \mathcal{C}_j} \nabla F_i (\bar{\bm x}^{t,e}_j) - \frac{1}{n_j}\sum_{i \in \mathcal{C}_j} \nabla F_i ({\color{black} \hat{\bm x}^{t,e}_{i,h}}) \right) \right \|^2 \\
\leq & \sum_{j=1}^N \frac{1}{n_j}\sum_{i \in \mathcal{C}_j} \mathbb{E}_t^e \left \| \nabla F_i (\bar{\bm x}^{t,e}_j) -  \nabla F_i ({\color{black} \hat{\bm x}^{t,e}_{i,h}}) \right \|^2.
}
Combining \eqref{T_3_T_3_prime}, \eqref{T_3_prime}, \eqref{T_3_prime_1}, and \eqref{T_3_prime_3}, and then utilizing the smoothness of $F_i$ and $f_j$ defined in Assumption \ref{assump_smoothness}, we have
\equa{\label{T_3_final}
T_3 = 3 \frac{Nd_{\textrm{max}}}{d} \delta_1^2 + 3L^2\sum_{j=1}^N \mathbb{E}_t^e\left \|\hat{\bm x}^{t,e}  - \bar{\bm x}^{t,e}_j \right \|^2 + 3L^2\sum_{j=1}^N \frac{1}{n_j}\sum_{i \in \mathcal{C}_j} \frac{1}{H} \! \sum_{h=0}^{H-1}  \mathbb{E}_t^e  \left\| \bar{\bm x}^{t,e}_j - {\color{black} \hat{\bm x}^{t,e}_{i,h}} \right\|^2.
}

Resorting to Cauchy-Schwartz inequality, we bound $T_2$ as
\equa{
	T_2 
	\leq & \underbrace{\mathbb{E}_t^e \left \|  \sum_{j=1}^N \bm p_j^t \odot {\color{black}\frac{1}{|\mathcal{C}_j^{t,e}|}\sum_{i\in \mathcal{C}_j^{t,e}}} \sum_{h=0}^{H-1} {\bm g}^{t,e}_{i,h} - \sum_{j=1}^N \bm p_j^t \odot {\color{black}\frac{1}{|\mathcal{C}_j^{t,e}|}\sum_{i\in \mathcal{C}_j^{t,e}}} \sum_{h=0}^{H-1} \nabla F_i ({\bm x}^{t,e}_{i,h}) \right \|^2}_{T_2^l} + \mathbb{E}_t^e \left \| \sum_{j=1}^N \bm p_j^t \odot {\color{black}\frac{1}{|\mathcal{C}_j^{t,e}|}\sum_{i\in \mathcal{C}_j^{t,e}}} \sum_{h=0}^{H-1} \nabla F_i ({\bm x}^{t,e}_{i,h}) \right \|^2.
}
Denoting $\bm a_i = \sum_{h=0}^{H-1} \left({\bm g}^{t,e}_{i,h} - \nabla F_i ({\bm x}^{t,e}_{i,h}) \right)$, we have 
\equa{
T_2^l &= \sum_{j=1}^N \mathbb{E}_t^e \left \|  \bm p_j^t \odot {\color{black}\frac{1}{|\mathcal{C}_j^{t,e}|}\sum_{i\in \mathcal{C}_j^{t,e}}} \bm a_i \right \|^2
=\sum_{j=1}^N \mathbb{E}_t^e \left[ {\color{black}\frac{1}{|\mathcal{C}_j^{t,e}|^2}\sum_{i\in \mathcal{C}_j^{t,e}}} \left \| \bm p_j^t \odot \bm a_i \right \|^2 \right]
{\color{black}\leq \sum_{j=1}^N \frac{1}{n_j |\mathcal{C}_j^{t,e}|}\sum_{i\in \mathcal{C}_j} \mathbb{E}_t^e \|\bm a_i \|^2} \\
&={\color{black}\sum_{j=1}^N \frac{1}{n_j |\mathcal{C}_j^{t,e}|}\sum_{i\in \mathcal{C}_j} \sum_{h=0}^{H-1} \mathbb{E}_t^e \|{\bm g}^{t,e}_{i,h} - \nabla F_i ({\bm x}^{t,e}_{i,h}) \|^2 \leq \sum_{j=1}^N \frac{1}{|\mathcal{C}_j^{t,e}|} H\sigma^2 = H \tilde{N}^{\prime} \sigma^2,
}
}
where the first equality holds because there is no overlapping between any two different masks in the same round, the second equality follows the fact that $\mathbb{E}_t^e[\bm a_i] = 0$ and $\bm a_i$ is independent of $\bm a_j$ for any $j\neq i$, the third equality follows \cite[Lemma 2]{jianyu}, {\color{black}the first inequality is due to the uniform client participation and $\|\bm p \odot \bm z\|^2 \leq \|\bm z\|^2$,} and the second 
inequality comes from Assumption \ref{assump_randomness_sgd}. 

We thus obtain
\equa{\label{T_2_squared}
	T_2 
	\leq & H \tilde{N}^{\prime} \sigma^2 +\mathbb{E}_t^e \left \| \sum_{j=1}^N \bm p_j^t \odot {\color{black}\frac{1}{|\mathcal{C}_j^{t,e}|}\sum_{i\in \mathcal{C}_j^{t,e}}} \sum_{h=0}^{H-1} \nabla F_i ({\bm x}^{t,e}_{i,h}) \right \|^2.
}

{\color{black}1) For \textbf{full participation case}, $\mathcal{C}_j^{t,e} = \mathcal{C}_j$, ${\bm x}^{t,e}_{i,h} = \hat{\bm x}^{t,e}_{i,h}$, and $ \tilde{N}^{\prime} = \tilde{N}$,}
by combining \eqref{L_smoothness_expand}, \eqref{T_1_with_T_3}, \eqref{T_3_final}, and \eqref{T_2_squared}, we obtain
\equa{\label{}
\mathbb{E}_t^e[f(\hat{\bm x}^{t,e+1})] 
		\leq & f(\hat{\bm x}^{t,e}) - \frac{\gamma H }{2}  \mathbb{E}_t^e \left\| \nabla f(\hat{\bm x}^{t,e} ) \right\|^2 + \gamma^2 H \tilde{N} L \sigma^2 +  \frac{3\gamma H}{2} \frac{Nd_{\textrm{max}}}{d} \delta_1^2\\
  &+ \frac{3 \gamma HL^2}{2} \left\{ \sum_{j=1}^N \mathbb{E}_t^e \left \|\hat{\bm x}^{t,e}  - \bar{\bm x}^{t,e}_j \right \|^2  +  \sum_{j=1}^N \frac{1}{n_j}\sum_{i \in \mathcal{C}_j} \frac{1}{H} \! \sum_{h=0}^{H-1} \mathbb{E}_t^e \left\|\bar{\bm x}^{t,e}_j -{\bm x}^{t,e}_{i,h}  \right\|^2 \right\}.
}
Taking an expectation over all the randomness for the above inequality and telescoping it from $e=0$ to $e=H-1$, we will obtain \eqref{equa_lemma1_full} in Lemma \ref{lemma_function_value}.  

{\color{black}
2) For \textbf{partial participation case}, we need to further bound $T_2$:
\equa{\label{T_2_squared_partial}
	T_2 
	\leq & H \tilde{N}^{\prime} \sigma^2 +\underbrace{\mathbb{E}_t^e \left \| \sum_{j=1}^N \bm p_j^t \odot \frac{1}{|\mathcal{C}_j^{t,e}|}\sum_{i\in \mathcal{C}_j^{t,e}} \sum_{h=0}^{H-1} \nabla F_i ({\bm x}^{t,e}_{i,h}) \right \|^2}_{T_2^r}.
}
For $T_2^r$, we have
\begin{small}
\equa{\label{T_2_squared_partial_1}
T_2^r &= \mathbb{E}_t^e \left \| \sum_{j=1}^N \bm p_j^t \odot \frac{1}{|\mathcal{C}_j^{t,e}|}\sum_{i\in \mathcal{C}_j^{t,e}} \sum_{h=0}^{H-1} \nabla F_i ({\bm x}^{t,e}_{i,h}) \mp \sum_{j=1}^N \bm p_j^t \odot \frac{1}{n_j}\sum_{i\in \mathcal{C}_j} \sum_{h=0}^{H-1} \nabla F_i (\hat{\bm x}^{t,e}_{i,h}) \right \|^2 \\
&=\mathbb{E}_t^e \left \| \sum_{j=1}^N \bm p_j^t \odot \frac{1}{|\mathcal{C}_j^{t,e}|}\sum_{i\in \mathcal{C}_j^{t,e}} \sum_{h=0}^{H-1} \nabla F_i ({\bm x}^{t,e}_{i,h}) - \sum_{j=1}^N \bm p_j^t \odot \frac{1}{n_j}\sum_{i\in \mathcal{C}_j} \sum_{h=0}^{H-1} \nabla F_i (\hat{\bm x}^{t,e}_{i,h}) \right \|^2  + \mathbb{E}_t^e \left \| \sum_{j=1}^N \bm p_j^t \odot \frac{1}{n_j}\sum_{i\in \mathcal{C}_j} \sum_{h=0}^{H-1} \nabla F_i (\hat{\bm x}^{t,e}_{i,h})  \right \|^2 \\
&= \sum_{j=1}^N \mathbb{E}_t^e \left \| \bm p_j^t \odot \frac{1}{|\mathcal{C}_j^{t,e}|}\sum_{i\in \mathcal{C}_j^{t,e}} \sum_{h=0}^{H-1} \nabla F_i ({\bm x}^{t,e}_{i,h}) -  \bm p_j^t \odot \frac{1}{n_j}\sum_{i\in \mathcal{C}_j} \sum_{h=0}^{H-1} \nabla F_i (\hat{\bm x}^{t,e}_{i,h}) \right \|^2  + \mathbb{E}_t^e \left \| \sum_{j=1}^N \bm p_j^t \odot \frac{1}{n_j}\sum_{i\in \mathcal{C}_j} \sum_{h=0}^{H-1} \nabla F_i (\hat{\bm x}^{t,e}_{i,h})  \right \|^2 \\
& \leq \sum_{j=1}^N \mathbb{E}_t^e \left \| \frac{1}{|\mathcal{C}_j^{t,e}|}\sum_{i\in \mathcal{C}_j^{t,e}} \sum_{h=0}^{H-1} \nabla F_i ({\bm x}^{t,e}_{i,h}) - \frac{1}{n_j}\sum_{i\in \mathcal{C}_j} \sum_{h=0}^{H-1} \nabla F_i (\hat{\bm x}^{t,e}_{i,h}) \right \|^2  + \mathbb{E}_t^e \left \| \sum_{j=1}^N \bm p_j^t \odot \frac{1}{n_j}\sum_{i\in \mathcal{C}_j} \sum_{h=0}^{H-1} \nabla F_i (\hat{\bm x}^{t,e}_{i,h})  \right \|^2.
}
\end{small}

Furthermore, by substituting $\mp \frac{1}{|\mathcal{C}_j^{t,e}|}\sum_{i\in \mathcal{C}_j^{t,e}} \sum_{h=0}^{H-1} \nabla F_i ({\bar{\bm x}}_j^{t,e})$ and $\mp \frac{1}{n_j}\sum_{i\in \mathcal{C}_j} \sum_{h=0}^{H-1} \nabla F_i ({\bar{\bm x}}_j^{t,e})$ into the first term of the above equality and utilizing Jensen's inequality and Assumption \ref{assump_smoothness}, we have
\equa{\label{T_2_squared_partial_2}
& \sum_{j=1}^N \mathbb{E}_t^e \left \| \frac{1}{|\mathcal{C}_j^{t,e}|}\sum_{i\in \mathcal{C}_j^{t,e}} \sum_{h=0}^{H-1} \nabla F_i ({\bm x}^{t,e}_{i,h}) \!-\! \frac{1}{n_j}\sum_{i\in \mathcal{C}_j} \sum_{h=0}^{H-1} \nabla F_i (\hat{\bm x}^{t,e}_{i,h}) \right \|^2 \\
\leq & 3\sum_{j=1}^N \mathbb{E}_t^e \left \| \frac{1}{|\mathcal{C}_j^{t,e}|}\sum_{i\in \mathcal{C}_j^{t,e}} \sum_{h=0}^{H-1} \nabla F_i ({\bar{\bm x}}_j^{t,e}) \!-\!  \frac{1}{n_j}\sum_{i\in \mathcal{C}_j} \sum_{h=0}^{H-1} \nabla F_i ({\bar{\bm x}}_j^{t,e}) \right \|^2 \!+\! 6H^2L^2  \sum_{j=1}^N \frac{1}{n_j}\sum_{i \in \mathcal{C}_j} \frac{1}{H} \! \sum_{h=0}^{H-1} \mathbb{E}_t^e \left\|\bar{\bm x}^{t,e}_j -\hat{\bm x}^{t,e}_{i,h}  \right\|^2 \\
\leq & 3H \sum_{j=1}^N \sum_{h=0}^{H-1} \mathbb{E}_t^e  \left \| \frac{1}{|\mathcal{C}_j^{t,e}|}\sum_{i\in \mathcal{C}_j^{t,e}} \nabla F_i ({\bar{\bm x}}_j^{t,e}) \!-\! \nabla f_j ({\bar{\bm x}}_j^{t,e}) \right \|^2 \!+\! 6H^2L^2  \sum_{j=1}^N \frac{1}{n_j}\sum_{i \in \mathcal{C}_j} \frac{1}{H} \! \sum_{h=0}^{H-1} \mathbb{E}_t^e \left\|\bar{\bm x}^{t,e}_j -\hat{\bm x}^{t,e}_{i,h}  \right\|^2 \\
\leq & 3H^2 \sum_{j=1}^N \frac{1}{|\mathcal{C}_j^{t,e}|^2} \sum_{i\in \mathcal{C}_j^{t,e}}  \mathbb{E}_t^e  \left \|\nabla F_i ({\bar{\bm x}}_j^{t,e}) \!-\!  \nabla f_j ({\bar{\bm x}}_j^{t,e}) \right \|^2 \!+\! 6H^2L^2  \sum_{j=1}^N \frac{1}{n_j}\sum_{i \in \mathcal{C}_j} \frac{1}{H} \! \sum_{h=0}^{H-1} \mathbb{E}_t^e \left\|\bar{\bm x}^{t,e}_j -\hat{\bm x}^{t,e}_{i,h}  \right\|^2 \\
=&3H^2 \tilde{N}^{\prime} \delta_2^2 \!+\! 6H^2L^2  \sum_{j=1}^N \frac{1}{n_j}\sum_{i \in \mathcal{C}_j} \frac{1}{H} \! \sum_{h=0}^{H-1} \mathbb{E}_t^e \left\|\bar{\bm x}^{t,e}_j -\hat{\bm x}^{t,e}_{i,h}  \right\|^2,
}
where we use $\frac{1}{n_j}\sum_{i \in \mathcal{C}_j}\mathbb{E}_t^e \left\|\bar{\bm x}^{t,e}_j -\hat{\bm x}^{t,e}_{i,h}  \right\|^2 = \mathbb{E}_t^e \left[ \frac{1}{|\mathcal{C}_j^{t,e}|} \sum_{i\in \mathcal{C}_j^{t,e}} \left\|\bar{\bm x}^{t,e}_j -{\bm x}^{t,e}_{i,h}  \right\|^2 \right]$ in the first inequality, and
the third inequality follows by the fact that clients in $\mathcal{C}_j^{t,e}$ are independently sampled from $\mathcal{C}_j$ and $\mathbb{E}[\nabla F_i ({\bar{\bm x}}_j^{t,e})]= \nabla f_j ({\bar{\bm x}}_j^{t,e})$.

Based on \eqref{T_2_squared_partial}, \eqref{T_2_squared_partial_1}, and \eqref{T_2_squared_partial_2}, we can derive an upper bound for $T_2$ as
\equa{\label{T_2_squared_partial_final}
T_2 \!\leq\! H \tilde{N}^{\prime} \sigma^2  \!+\! \mathbb{E}_t^e \left \| \sum_{j=1}^N \bm p_j^t \odot \frac{1}{n_j}\sum_{i\in \mathcal{C}_j} \sum_{h=0}^{H-1} \nabla F_i (\hat{\bm x}^{t,e}_{i,h})  \right \|^2 \!+\!  3H^2 \tilde{N}^{\prime} \delta_2^2 \!+\! 6H^2L^2  \sum_{j=1}^N \frac{1}{n_j}\sum_{i \in \mathcal{C}_j} \frac{1}{H} \! \sum_{h=0}^{H-1} \mathbb{E}_t^e \left\|\bar{\bm x}^{t,e}_j \!-\!\hat{\bm x}^{t,e}_{i,h}  \right\|^2.
}
Combining \eqref{L_smoothness_expand}, \eqref{T_1_with_T_3}, \eqref{T_3_final}, and \eqref{T_2_squared_partial_final}, and utilizing $\gamma \leq \frac{1}{4HL},$ we obtain
\equa{\label{}
\mathbb{E}_t^e[f(\hat{\bm x}^{t,e+1})] 
		\leq & f(\hat{\bm x}^{t,e}) - \frac{\gamma H }{2}  \mathbb{E}_t^e \left\| \nabla f(\hat{\bm x}^{t,e} ) \right\|^2 + \gamma^2 H \tilde{N}^{\prime} L \sigma^2 +  3\gamma^2 H^2 L \tilde{N}^{\prime} \delta_2^2 +  \frac{3\gamma H}{2} \frac{Nd_{\textrm{max}}}{d} \delta_1^2\\
  &+ \frac{3 \gamma HL^2}{2} \left\{ \sum_{j=1}^N \mathbb{E}_t^e \left \|\hat{\bm x}^{t,e}  - \bar{\bm x}^{t,e}_j \right \|^2  +  \sum_{j=1}^N \frac{1}{n_j}\sum_{i \in \mathcal{C}_j} \frac{1}{H} \! \sum_{h=0}^{H-1} \mathbb{E}_t^e \left\|\bar{\bm x}^{t,e}_j -\hat{\bm x}^{t,e}_{i,h}  \right\|^2 \right\} \\
  &+ 6\gamma^2 H^2L^3  \sum_{j=1}^N \frac{1}{n_j}\sum_{i \in \mathcal{C}_j} \frac{1}{H} \! \sum_{h=0}^{H-1} \mathbb{E}_t^e \left\|\bar{\bm x}^{t,e}_j \!-\!\hat{\bm x}^{t,e}_{i,h}  \right\|^2 \\
  \leq & f(\hat{\bm x}^{t,e}) - \frac{\gamma H }{2}  \mathbb{E}_t^e \left\| \nabla f(\hat{\bm x}^{t,e} ) \right\|^2 + \gamma^2 H \tilde{N}^{\prime} L \sigma^2 +  3\gamma^2 H^2 L \tilde{N}^{\prime} \delta_2^2 +  \frac{3\gamma H}{2} \frac{Nd_{\textrm{max}}}{d} \delta_1^2\\
  &+ 3 \gamma HL^2 \left\{ \sum_{j=1}^N \mathbb{E}_t^e \left \|\hat{\bm x}^{t,e}  - \bar{\bm x}^{t,e}_j \right \|^2  +  \sum_{j=1}^N \frac{1}{n_j}\sum_{i \in \mathcal{C}_j} \frac{1}{H} \! \sum_{h=0}^{H-1} \mathbb{E}_t^e \left\|\bar{\bm x}^{t,e}_j -\hat{\bm x}^{t,e}_{i,h}  \right\|^2 \right\}.
}
Taking an expectation over all the randomness for the above inequality and telescoping it from $e=0$ to $e=H-1$, we obtain \eqref{equa_lemma1_partial} in Lemma \ref{lemma_function_value}:
\begin{small}
\begin{align}\label{}
&\mathbb{E}[f(\bar{\bm x}^{t+1})] 
		\leq  \mathbb{E}[f(\bar{\bm x}^t)] - \frac{\gamma H }{2}  \sum_{e=0}^{E-1} \mathbb{E} \left\| \nabla f(\hat{\bm x}^{t,e} ) \right\|^2 + \gamma^2 E H \tilde{N}^{\prime} L \sigma^2  \nonumber \\
  &~~~~~~~~+ 3\gamma^2 E H^2 L \tilde{N}^{\prime} \delta_2^2 +  \frac{3\gamma}{2} \frac{EHNd_{\textrm{max}}}{d} \delta_1^2 + 3H\gamma L^2 \left( D_t +Q_t \right).
\end{align}
\end{small}

}

\subsubsection{Proof of Lemma \ref{lemma_function_value_AirComp}} 
With AirComp, the iteration becomes 
\begin{align}
    \hat{\bm x}^{t,e+1} = \hat{\bm x}^{t,e} - \gamma \sum_{j=1}^N \bm p_j^t \odot \frac{1}{n_j}\sum_{i\in \mathcal{C}_j} \sum_{h=0}^{H-1} {\bm g}^{t,e}_{i,h} + \sum_{j=1}^N \bm p_j^t \odot \bm n_j^{t,e}.
\end{align} 
Note that $\bm p_j^t \odot \bm n_j^{t,e} =  \bm n_j^{t,e}$ holds.
According to Assumption \ref{assump_smoothness}, we have
\begin{equation}\label{}
	\begin{aligned}
		\mathbb{E}_t^e[f(\hat{\bm x}^{t,e+1})] \leq & f(\hat{\bm x}^{t,e}) \underbrace{- \gamma \mathbb{E}_t^e \left \langle \nabla f(\hat{\bm x}^{t,e}), \sum_{j=1}^N \bm p_j^t \odot \frac{1}{n_j}\sum_{i\in \mathcal{C}_j} \sum_{h=0}^{H-1} {\bm g}^{t,e}_{i,h} - \sum_{j=1}^N \bm p_j^t \odot \bm n_j^{t,e} \right \rangle}_{G_1} \\
  &+  L \underbrace{\frac{1}{2}\mathbb{E}_t^e \left \| \gamma \sum_{j=1}^N \bm p_j^t \odot \frac{1}{n_j}\sum_{i\in \mathcal{C}_j} \sum_{h=0}^{H-1} {\bm g}^{t,e}_{i,h}  - \sum_{j=1}^N \bm p_j^t \odot \bm n_j^{t,e} \right \|^2}_{G_2}.
	\end{aligned}
\end{equation} 
Since $\bm n_j^{t,e}$ is independent of ${\bm g}^{t,e}_{i,h}$ and $\mathbb{E}[\bm n_j^{t,e}] = 0$, the upper bound of $G_1$ is the same as $T_1$ in \eqref{L_smoothness_expand}. Additionally, we can rewrite $G_2$ as
\equa{
G_2 =& \frac{1}{2}\gamma^2\mathbb{E}_t^e \left \| \sum_{j=1}^N \bm p_j^t \odot \frac{1}{n_j}\sum_{i\in \mathcal{C}_j} \sum_{h=0}^{H-1} {\bm g}^{t,e}_{i,h}  \right \|^2 + \frac{1}{2} \sum_{j=1}^N \mathbb{E}\|\bm p_j^t \odot \bm n_j^{t,e}\|^2.
}
Next, following the same step in the proof of Lemma \ref{lemma_edge_divergence}, we can obtain
\equa{\label{}
\mathbb{E}_t^e[f(\hat{\bm x}^{t,e+1})] 
		\leq & f(\hat{\bm x}^{t,e}) - \frac{\gamma H }{2}  \mathbb{E}_t^e \left\| \nabla f(\hat{\bm x}^{t,e} ) \right\|^2 + \gamma^2 H \tilde{N} L \sigma^2 +  \frac{3\gamma H}{2} \frac{Nd_{\textrm{max}}}{d} \delta_1^2  + \frac{1}{2} \sum_{j=1}^N \text{MSE}_j^{t,e} \\
  &+ \frac{3 \gamma HL^2}{2} \left\{ \sum_{j=1}^N \mathbb{E}_t^e \left \|\hat{\bm x}^{t,e}  - \bar{\bm x}^{t,e}_j \right \|^2  +  \sum_{j=1}^N \frac{1}{n_j}\sum_{i \in \mathcal{C}_j} \frac{1}{H} \! \sum_{h=0}^{H-1} \mathbb{E}_t^e \left\|\bar{\bm x}^{t,e}_j -{\bm x}^{t,e}_{i,h}  \right\|^2 \right\},
}
where $\text{MSE}_j^{t,e} = \mathbb{E}\|\bm p_j^t \odot \bm n_j^{t,e}\|^2$.
Taking an expectation over all the randomness for the above inequality and telescoping it from $e=0$ to $e=H-1$, we will obtain Lemma \ref{lemma_function_value_AirComp}. 

\subsubsection{Proof of Lemma \ref{lemma_edge_divergence}}
According to the iteration in Algorithm \ref{alg:HFIST}, we have
\equa{\label{D_t_without_sum}
\mathbb{E}\left \|\hat{\bm x}^{t,e}  - \bar{\bm x}^{t,e}_j \right \|^2 =& \mathbb{E} \left \|\bar{\bm x}^{t} - \gamma \sum_{\tau_1 = 0}^{e-1} \sum_{j=1}^N \bm p_j^t \odot {\color{black}\frac{1}{|\mathcal{C}_j^{t,e}|}\sum_{i\in \mathcal{C}_j^{t,e}}} \sum_{h=0}^{H-1} {\bm g}^{t,\tau_1}_{i,h}   - \bar{\bm x}^{t,0}_j + \gamma \sum_{\tau_1 = 0}^{e-1} \bm p_j^t \odot {\color{black}\frac{1}{|\mathcal{C}_j^{t,e}|}\sum_{i\in \mathcal{C}_j^{t,e}}} \sum_{h=0}^{H-1} {\bm g}^{t,\tau_1}_{i,h} \right \|^2 \\
\leq & 2\mathbb{E} \left \|\bar{\bm x}^{t}   - \bar{\bm x}^{t,0}_j  \right\|^2  + 2 \gamma^2 \underbrace{ \mathbb{E} \left\|\sum_{j=1}^N \bm p_j^t \odot\sum_{\tau_1 = 0}^{e-1} {\color{black}\frac{1}{|\mathcal{C}_j^{t,e}|}\sum_{i\in \mathcal{C}_j^{t,e}}} \sum_{h=0}^{H-1} {\bm g}^{t,\tau_1}_{i,h} - \bm p_j^t \odot \sum_{\tau_1 = 0}^{e-1} {\color{black}\frac{1}{|\mathcal{C}_j^{t,e}|}\sum_{i\in \mathcal{C}_j^{t,e}}} \sum_{h=0}^{H-1} {\bm g}^{t,\tau_1}_{i,h} \right \|^2}_{T_{4,j}^{t,e}}
}
For $T_{4,j}^{t,e}$, we can obtain
\begin{small}
\equa{
T_{4,j}^{t,e} = & \mathbb{E} \left\|\sum_{j=1}^N \bm p_j^t \odot\sum_{\tau_1 = 0}^{e-1}{\color{black}\frac{1}{|\mathcal{C}_j^{t,e}|}\sum_{i\in \mathcal{C}_j^{t,e}}} \sum_{h=0}^{H-1} \left({\bm g}^{t,\tau_1}_{i,h} \mp  \nabla F_i({\bm x}_{i,h}^{t,\tau_1}) \right) - \bm p_j^t \odot \sum_{\tau_1 = 0}^{e-1} {\color{black}\frac{1}{|\mathcal{C}_j^{t,e}|}\sum_{i\in \mathcal{C}_j^{t,e}}} \sum_{h=0}^{H-1} \left({\bm g}^{t,\tau_1}_{i,h} \mp  \nabla F_i({\bm x}_{i,h}^{t,\tau_1}) \right) \right \|^2 \\
\leq & \underbrace{3\mathbb{E}\left\|\sum_{j=1}^N \bm p_j^t \odot\sum_{\tau_1 = 0}^{e-1} {\color{black}\frac{1}{|\mathcal{C}_j^{t,e}|}\sum_{i\in \mathcal{C}_j^{t,e}}} \sum_{h=0}^{H-1} \left({\bm g}^{t,\tau_1}_{i,h} -  \nabla F_i({\bm x}_{i,h}^{t,\tau_1}) \right) \right\|^2  + 3\mathbb{E}\left\|\bm p_j^t \odot \sum_{\tau_1 = 0}^{e-1} {\color{black}\frac{1}{|\mathcal{C}_j^{t,e}|}\sum_{i\in \mathcal{C}_j^{t,e}}} \sum_{h=0}^{H-1} \left({\bm g}^{t,\tau_1}_{i,h} -  \nabla F_i({\bm x}_{i,h}^{t,\tau_1}) \right) \right \|^2}_{T_{5,j}^{t,e}} \\
&+\underbrace{3\mathbb{E}\left\|\sum_{j=1}^N \bm p_j^t \odot\sum_{\tau_1 = 0}^{e-1} {\color{black}\frac{1}{|\mathcal{C}_j^{t,e}|}\sum_{i\in \mathcal{C}_j^{t,e}}} \sum_{h=0}^{H-1} \nabla F_i({\bm x}_{i,h}^{t,\tau_1}) - \bm p_j^t \odot \sum_{\tau_1 = 0}^{e-1} {\color{black}\frac{1}{|\mathcal{C}_j^{t,e}|}\sum_{i\in \mathcal{C}_j^{t,e}}} \sum_{h=0}^{H-1} \nabla F_i({\bm x}_{i,h}^{t,\tau_1}) \right\|^2}_{T_{6,j}^{t,e}}.
}
\end{small}

For $T_{5,j}^{t,e}$, we have the following result:
\begin{small}
\equa{
T_{5,j}^{t,e}
= & 3\sum_{j=1}^N \mathbb{E}\left\| \bm p_j^t \odot\sum_{\tau_1 = 0}^{e-1} {\color{black}\frac{1}{|\mathcal{C}_j^{t,e}|}\sum_{i\in \mathcal{C}_j^{t,e}}} \sum_{h=0}^{H-1} \left({\bm g}^{t,\tau_1}_{i,h} -  \nabla F_i({\bm x}_{i,h}^{t,\tau_1}) \right) \right\|^2  + 3\mathbb{E}\left\|\bm p_j^t \odot \sum_{\tau_1 = 0}^{e-1} {\color{black}\frac{1}{|\mathcal{C}_j^{t,e}|}\sum_{i\in \mathcal{C}_j^{t,e}}} \sum_{h=0}^{H-1} \left({\bm g}^{t,\tau_1}_{i,h} -  \nabla F_i({\bm x}_{i,h}^{t,\tau_1}) \right) \right \|^2 \\
\leq & {\color{black}3\sum_{j=1}^N \sum_{\tau_1 = 0}^{e-1} \frac{1}{|\mathcal{C}_j^{t,e}|^2}\sum_{i\in \mathcal{C}_j^{t,e}} \sum_{h=0}^{H-1} 
 \mathbb{E}\left\|{\bm g}^{t,\tau_1}_{i,h} -  \nabla F_i({\bm x}_{i,h}^{t,\tau_1}) \right\|^2  + 3 \sum_{\tau_1 = 0}^{e-1} \frac{1}{|\mathcal{C}_j^{t,e}|^2}\sum_{i\in \mathcal{C}_j^{t,e}} \sum_{h=0}^{H-1} \mathbb{E}\left\|{\bm g}^{t,\tau_1}_{i,h} -  \nabla F_i({\bm x}_{i,h}^{t,\tau_1}) \right \|^2} \\
\leq &  {\color{black}3eH \sum_{j=1}^N \frac{1}{|\mathcal{C}_j^{t,e}|} \sigma^2 + 3eH \frac{1}{|\mathcal{C}_j^{t,e}|} \sigma^2 = \tilde{N}^{\prime} \sigma^2 + 3eH \frac{1}{n_j^{\prime}} \sigma^2},
}
\end{small}

where $\tilde{N}^{\prime} = \sum_{j=1}^N \frac{1}{n_j^{\prime}}$, the first inequality comes from \cite[Lemma 2]{jianyu} and the inequality $\| \bm p \odot \bm z\| \leq \| \bm z\|^2$ and the second inequality follows Assumption \ref{assump_randomness_sgd}.

For $T_{6,j}^{t,e}$, we have 
\begin{small}
\equa{
T_{6,j}^{t,e} = & 9eH \sum_{\tau_1 = 0}^{e-1} \sum_{h=0}^{H-1} \mathbb{E}\left\|\sum_{j=1}^N \bm p_j^t \odot {\color{black}\frac{1}{|\mathcal{C}_j^{t,e}|}\sum_{i\in \mathcal{C}_j^{t,e}}} \left ( \nabla F_i({\bm x}_{i,h}^{t,\tau_1}) \mp \nabla F_i(\bar{\bm x}_{j}^{t,\tau_1}) \right) - \bm p_j^t \odot {\color{black}\frac{1}{|\mathcal{C}_j^{t,e}|}\sum_{i\in \mathcal{C}_j^{t,e}}} \left ( \nabla F_i({\bm x}_{i,h}^{t,\tau_1}) \mp \nabla F_i(\bar{\bm x}_{j}^{t,\tau_1}) \right) \right\|^2 \\
\leq & 9eH \! \sum_{\tau_1 = 0}^{e-1} \! \sum_{h=0}^{H-1} \! \mathbb{E}\!\left\|\sum_{j=1}^N \bm p_j^t \! \odot \! {\color{black}\frac{1}{|\mathcal{C}_j^{t,e}|}\sum_{i\in \mathcal{C}_j^{t,e}}} \! \left ( \nabla F_i({\bm x}_{i,h}^{t,\tau_1}) \!- \! \nabla F_i(\bar{\bm x}_{j}^{t,\tau_1}) \right) \right\|^2 
\!+\! 9eH \! \sum_{\tau_1 = 0}^{e-1} \! \sum_{h=0}^{H-1} \! \mathbb{E}\!\left\| \bm p_j^t \!\odot\! {\color{black}\frac{1}{|\mathcal{C}_j^{t,e}|}\sum_{i\in \mathcal{C}_j^{t,e}}}\! \left ( \nabla F_i({\bm x}_{i,h}^{t,\tau_1}) \!-\! \nabla F_i(\bar{\bm x}_{j}^{t,\tau_1}) \right) \right\|^2 \\
&+9eH^2 \sum_{\tau_1 = 0}^{e-1} \mathbb{E}\left\| 
\sum_{j=1}^N \bm p_j^t \odot {\color{black}\frac{1}{|\mathcal{C}_j^{t,e}|}\sum_{i\in \mathcal{C}_j^{t,e}}}  \nabla F_i(\bar{\bm x}_{j}^{t,\tau_1}) - \bm p_j^t \odot {\color{black}\frac{1}{|\mathcal{C}_j^{t,e}|}\sum_{i\in \mathcal{C}_j^{t,e}}} \nabla F_i(\bar{\bm x}_{j}^{t,\tau_1})
\right\|^2 \\
\leq & 9eH \sum_{\tau_1 = 0}^{e-1} \sum_{h=0}^{H-1} \sum_{j=1}^N \frac{1}{n_j}\sum_{i\in \mathcal{C}_j} \mathbb{E}\left\| \nabla F_i({\color{black}\hat{\bm x}_{i,h}^{t,\tau_1}}) - \nabla F_i(\bar{\bm x}_{j}^{t,\tau_1})  \right\|^2 
+ 9eH \sum_{\tau_1 = 0}^{e-1} \sum_{h=0}^{H-1} \frac{1}{n_j}\sum_{i\in \mathcal{C}_j} \mathbb{E}\left\| \nabla F_i({\color{black}\hat{\bm x}_{i,h}^{t,\tau_1}}) - \nabla F_i(\bar{\bm x}_{j}^{t,\tau_1}) \right\|^2 \\
&+9eH^2 \sum_{\tau_1 = 0}^{e-1} \mathbb{E}\left\| 
\sum_{j=1}^N \bm p_j^t \odot  \nabla f_j(\bar{\bm x}_{j}^{t,\tau_1}) - \bm p_j^t \odot f_j(\bar{\bm x}_{j}^{t,\tau_1})
\right\|^2 \\
\leq & 9eHL^2 \sum_{\tau_1 = 0}^{e-1} \sum_{h=0}^{H-1} \sum_{j=1}^N \frac{1}{n_j}\sum_{i\in \mathcal{C}_j} \mathbb{E}\left\| {\color{black}\hat{\bm x}_{i,h}^{t,\tau_1}} - \bar{\bm x}_{j}^{t,\tau_1} \right\|^2 + 9eHL^2 \sum_{\tau_1 = 0}^{e-1} \sum_{h=0}^{H-1} \frac{1}{n_j}\sum_{i\in \mathcal{C}_j} \mathbb{E}\left\| {\color{black}\hat{\bm x}_{i,h}^{t,\tau_1}} - \bar{\bm x}_{j}^{t,\tau_1}\right\|^2 \\
&+9eH^2 \sum_{\tau_1 = 0}^{e-1} \underbrace{\mathbb{E}\left\| 
\sum_{j=1}^N \bm p_j^t \odot  \nabla f_j(\bar{\bm x}_{j}^{t,\tau_1}) - \bm p_j^t \odot f_j(\bar{\bm x}_{j}^{t,\tau_1})
\right\|^2}_{T_{7,j}^{t,e}}.
}
\end{small}

\equa{
T_{7,j}^{t,e} =& \mathbb{E}\left\| 
\sum_{j=1}^N \bm p_j^t \odot \left( \nabla f_j(\bar{\bm x}_{j}^{t,\tau_1})  \mp \nabla f_j(\hat{\bm x}^{t,\tau_1}) \right) - \bm p_j^t \odot \left(\nabla f_j(\bar{\bm x}_{j}^{t,\tau_1}) \mp f_j(\hat{\bm x}^{t,\tau_1}) \right)
\right\|^2 \\
\leq & 3\mathbb{E}\left\| 
\sum_{j=1}^N \bm p_j^t \odot \left( \nabla f_j(\bar{\bm x}_{j}^{t,\tau_1})  - \nabla f_j(\hat{\bm x}^{t,\tau_1}) \right) \right \|^2 + 3\left\| \bm p_j^t \odot \left(\nabla f_j(\bar{\bm x}_{j}^{t,\tau_1})  - \nabla f_j(\hat{\bm x}^{t,\tau_1}) \right)
\right\|^2 \\
&+3\mathbb{E}\left\|\sum_{j=1}^N \bm p_j^t \odot \nabla f_j(\hat{\bm x}^{t,\tau_1}) -  \bm p_j^t \odot \nabla f_j(\hat{\bm x}^{t,\tau_1}) \right\|^2 \\
\leq & 3L^2\sum_{j=1}^N \mathbb{E}\left\| 
\bar{\bm x}_{j}^{t,\tau_1}  - \hat{\bm x}^{t,\tau_1} \right \|^2 + 3L^2 \left\|\bar{\bm x}_{j}^{t,\tau_1} - \hat{\bm x}^{t,\tau_1}
\right\|^2 \\
&+\underbrace{3\mathbb{E}\left\|\sum_{j=1}^N \bm p_j^t \odot \left( \nabla f_j(\hat{\bm x}^{t,\tau_1}) \mp \nabla f(\hat{\bm x}^{t,\tau_1}) \right) -  \bm p_j^t \odot \left( \nabla f_j(\hat{\bm x}^{t,\tau_1}) \mp \nabla f(\hat{\bm x}^{t,\tau_1}) \right) \right\|^2}_{T_{8,j}^{t,e}}.
}
Similarly, we further bound $\sum_{j=1}^N T_{8,j}^{t,e}$ as follows 
\equa{
\sum_{j=1}^N T_{8,j}^{t,e} \leq& 9 (N+1) \sum_{j=1}^N \mathbb{E}\left\| \bm p_j^t \odot \left( \nabla f_j(\hat{\bm x}^{t,\tau_1}) - \nabla f(\hat{\bm x}^{t,\tau_1}) \right) \right\|^2 + 9 \sum_{j=1}^N \mathbb{E}\left\|\bm p_j^t \odot \nabla f(\hat{\bm x}^{t,\tau_1}) - \nabla f(\hat{\bm x}^{t,\tau_1})  \right\|^2\\
\leq& 9 (N+1) \frac{Nd_{\textrm{max}}}{d} \delta_1^2 + 9 (N-1) \mathbb{E}\left\|\nabla f(\hat{\bm x}^{t,\tau_1})  \right\|^2,
}
where the second inequality comes from Propositions \ref{fact_bounded_masked_norm}, \ref{fact_mask_N_1} and Assumption \ref{assump_gradient_dissimilarity_edge}. 
\equa{\label{T_4_final_sum}
\sum_{e=0}^{E-1}\sum_{j=1}^N T_{4,j}^{t,e} 
\leq& \frac{3}{2} (N+1) {\color{black}\tilde{N}^{\prime}} E^2 H \sigma^2 + \frac{9}{2}(N+1) E^2HL^2 \sum_{e=0}^{E-1} \sum_{h=0}^{H-1} \sum_{j=1}^N \frac{1}{n_j}\sum_{i\in \mathcal{C}_j} \mathbb{E}\left\| {\color{black}\hat{\bm x}_{i,h}^{t,e}} - \bar{\bm x}_{j}^{t,e} \right\|^2 \\
& \frac{27}{2}(N+1)E^2H^2 L^2 \sum_{e=0}^{E-1} \sum_{j=1}^N \mathbb{E}\left\| 
\bar{\bm x}_{j}^{t,e}  - \hat{\bm x}^{t,e} \right \|^2 \\
&27 (N+1)E^3H^2 \frac{Nd_{\textrm{max}}}{d} \delta_1^2 + \frac{81}{2} (N-1)E^2H^2 \sum_{e=0}^{E-1} \mathbb{E}\left\|\nabla f(\hat{\bm x}^{t,e})  \right\|^2.
}
Plugging \eqref{T_4_final_sum} into \eqref{D_t_without_sum} and utilizing Proposition \ref{fact_mask_N_1}, we obtain
\equa{
D_t \leq& 2 E (N-1) \mathbb{E} \left \|\bar{\bm x}^{t}  \right\|^2  + 3 \gamma^2 (N+1)E^2H {\color{black}\tilde{N}^{\prime}} \sigma^2 + 9\gamma^2(N+1)E^2H^2L^2 Q_t  \\
&+27 \gamma^2 (N+1)E^2H^2L^2 D_t + 54 \gamma^2 (N+1)E^3H^2 \frac{Nd_{\textrm{max}}}{d} \delta_1^2 \\
&+81 \gamma^2 (N-1)E^2H^2\sum_{e = 0}^{E-1}\mathbb{E}\|\nabla f(\hat{\bm x}^{t,e})\|^2.
}
As $1 -  27 \gamma^2 (N+1)E^2H^2L^2 \geq \frac{1}{2}$ when $\gamma \leq \frac{1}{\sqrt{54(N+1)}EHL}$, we thus obtain Lemma \ref{lemma_edge_divergence}.

\subsubsection{Proof of Lemma \ref{lemma_client_divergence}}
Based on the iteration $\bar{\bm x}^{t,e}_j -{\color{black}\hat{\bm x}^{t,e}_{i,h}} = \gamma \bm p_j^t \odot \sum_{\tau_2=0}^{h-1} {\bm g}^{t,e}_{i,\tau_2}$, we can write
\equa{
\mathbb{E} \left\|\bar{\bm x}^{t,e}_j - {\color{black}\hat{\bm x}^{t,e}_{i,h}} \right\|^2
\leq & 2 \gamma^2\mathbb{E} \left\| \sum_{\tau_2=0}^{h-1} {\bm g}^{t,e}_{i,\tau_2} -\sum_{\tau_2=0}^{h-1} \nabla F_i( {\color{black}\hat{\bm x}^{t,e}_{i,\tau_2}}) \right\|^2 + 2\gamma^2 \mathbb{E} \left\|\sum_{\tau_2=0}^{h-1} \nabla F_i( {\color{black}\hat{\bm x}^{t,e}_{i,\tau_2}}) \right\|^2 \\
= & 2 \gamma^2 \sum_{\tau_2=0}^{h-1} \mathbb{E} \left\| {\bm g}^{t,e}_{i,\tau_2} - \nabla F_i( {\color{black}\hat{\bm x}^{t,e}_{i,\tau_2}}) \right\|^2 + 2\gamma^2 \mathbb{E} \left\| \sum_{\tau_2=0}^{h-1} \nabla F_i( {\color{black}\hat{\bm x}^{t,e}_{i,\tau_2}}) \right\|^2 \\
\leq & 2 \gamma^2 h \sigma^2 + 2\gamma^2 h \sum_{\tau_2=0}^{h-1} \mathbb{E} \left\| \nabla F_i( {\color{black}\hat{\bm x}^{t,e}_{i,\tau_2}}) \right\|^2,
}
where the first inequality follows Cauchy-Schwartz inequality, the equality comes from \cite[Lemma 2]{jianyu}, and the second inequality follows Assumption \ref{assump_randomness_sgd}. Additionally, we can bound $\mathbb{E} \left\| \nabla F_i( {\color{black}\hat{\bm x}^{t,e}_{i,\tau_2}}) \right\|^2$ as
\equa{
\mathbb{E} \left\| \nabla F_i( {\color{black}\hat{\bm x}^{t,e}_{i,\tau_2}}) \right\|^2 = &  \mathbb{E} \left\| \nabla F_i( {\color{black}\hat{\bm x}^{t,e}_{i,\tau_2}}) \mp \nabla F_i( \bar{\bm x}^{t,e}_{j} ) \mp \nabla f_j( \bar{\bm x}^{t,e}_{j} ) \mp \nabla f( \bar{\bm x}^{t,e}_{j} ) \mp \nabla f( \hat{\bm x}^{t,e} )  \right\|^2 \\
\leq &  
5\mathbb{E} \left\| \nabla F_i( {\color{black}\hat{\bm x}^{t,e}_{i,\tau_2}}) - \nabla F_i( \bar{\bm x}^{t,e}_{j} ) \right \|^2  
+ 5\mathbb{E} \left\| \nabla F_i( \bar{\bm x}^{t,e}_{j} ) - \nabla f_j( \bar{\bm x}^{t,e}_{j} ) \right\|  
+ 5\mathbb{E} \left\| \nabla f_j( \bar{\bm x}^{t,e}_{j} ) -\nabla f( \bar{\bm x}^{t,e}_{j} )  \right\|^2  \\
&+ 5\mathbb{E} \left\| \nabla f( \bar{\bm x}^{t,e}_{j} ) - \nabla f( \hat{\bm x}^{t,e} )  \right\|^2
+ 5\mathbb{E} \left\| \nabla f( \hat{\bm x}^{t,e} )  \right\|^2\\
\leq &  
5L^2 \mathbb{E} \left\| {\color{black}\hat{\bm x}^{t,e}_{i,\tau_2}} - \bar{\bm x}^{t,e}_{j} \right \|^2  
+ 5\mathbb{E} \left\| \nabla F_i( \bar{\bm x}^{t,e}_{j} ) - \nabla f_j( \bar{\bm x}^{t,e}_{j} ) \right\|  
+ 5\mathbb{E} \left\| \nabla f_j( \bar{\bm x}^{t,e}_{j} ) -\nabla f( \bar{\bm x}^{t,e}_{j} )  \right\|^2  \\
&+ 5L^2\mathbb{E} \left\| \bar{\bm x}^{t,e}_{j} - \hat{\bm x}^{t,e}  \right\|^2
+ 5\mathbb{E} \left\| \nabla f( \hat{\bm x}^{t,e} )  \right\|^2,
}
where the first inequality comes from Cauchy-Schwartz inequality and the second one follows Assumption \ref{assump_smoothness}. Hence, we have
\equa{
\sum_{j=1}^N \frac{1}{n_j}\sum_{i \in \mathcal{C}_j} \mathbb{E} \left\|\bar{\bm x}^{t,e}_j -{\color{black}\hat{\bm x}^{t,e}_{i,\tau_2}}  \right\|^2 
\leq &  10 \gamma^2 h L^2 \sum_{\tau_2=0}^{h-1} \sum_{j=1}^N \frac{1}{n_j}\sum_{i \in \mathcal{C}_j} 
 \mathbb{E} \left\|{\color{black}\hat {\bm x}^{t,e}_{i,\tau_2}} - \bar{\bm x}^{t,e}_{j} \right \|^2 + 10 \gamma^2 h^2 L^2 \sum_{j=1}^N  \mathbb{E} \left\| \bar{\bm x}^{t,e}_{j} - \hat{\bm x}^{t,e}  \right\|^2  \\
& + 10 \gamma^2 N h^2 \mathbb{E} \left\| \nabla f( \hat{\bm x}^{t,e} )  \right\|^2 + 2 \gamma^2 N h \sigma^2 + 10 \gamma^2 N h^2 \delta_2^2 + 10 \gamma^2 N h^2 \delta_1^2.
}
Recalling the definitions of $Q_t$ and $Q_t$ in \eqref{D_Q_t}, we have
\equa{
  Q_t \!\leq &  5 \gamma^2 H^2 L^2   Q_t \!+\! \frac{10}{3} \gamma^2 H^2 L^2  D_t  \!+\! \frac{10}{3} \gamma^2 N H^2 \! \sum_{e=0}^{E-1} \!\mathbb{E} \!\left\| \nabla f( \hat{\bm x}^{t,e} )  \right\|^2 \!+\! \gamma^2 N H E \sigma^2 \!+\! \frac{10}{3} \gamma^2 N H^2 E \delta_2^2 \!+\! \frac{10}{3} \gamma^2 N H^2 E \delta_1^2.
}
As $3\left(1 - 5 \gamma^2 H^2 L^2 \right) \geq 2$ when $\gamma \leq \frac{1}{\sqrt{15}HL}$, we thus obtain Lemma \ref{lemma_client_divergence}.
\newpage

\end{document}